\documentclass[11pt]{article}
\usepackage{fullpage,graphicx,psfrag,amsmath,amsfonts,verbatim}
\usepackage[small,bf]{caption}
\usepackage[american]{babel}

\usepackage{natbib} 
\usepackage{mathtools} 
\usepackage{booktabs} 
\usepackage{tikz} 
\usepackage{caption}
\usepackage{xcolor}
\usepackage{subcaption}
\usepackage{amsthm}
\usepackage{amsmath, mathrsfs,amssymb}
\usepackage{algorithm}
\usepackage{algorithmic}
\usepackage{authblk}

\usepackage{amsmath,amsfonts,bm}

\newcommand{\ei}{\mathbf{e}_i}
\newcommand{\ia}{\mathbb{I}_A}
\newcommand{\ga}{\gamma} 
\newcommand{\alp}{\alpha}

\newcommand{\Expect}{\mathbb{E{}}}

\newcommand{\argmax}{\mathop{\rm argmax}}

\def\eps{{\epsilon}}
\newcommand{\p}{\prod} 
\newcommand{\diag}{\mathop{\bf Diag}}

\newtheorem{thm}{Theorem}[section]
\newtheorem{lem}{Lemma}[section]

\newtheorem{prop}{Proposition}[section]
\newtheorem{asmp}{Assumption}[section]
\newtheorem{defn}{Definition}[section]

\newcommand{\li}{\limits} 

\newcommand{\mdp}{\mathcal{M}}
\newcommand{\state}{\mathcal{S}}
\newcommand{\action}{\mathcal{A}}
\newcommand{\trans}{\mathcal{P}}
\newcommand{\reward}{R}

\def\eqref#1{equation~\ref{#1}}

\def\1{\bm{1}}

\def\eps{{\epsilon}}

\DeclareMathAlphabet{\mathsfit}{\encodingdefault}{\sfdefault}{m}{sl}
\SetMathAlphabet{\mathsfit}{bold}{\encodingdefault}{\sfdefault}{bx}{n}

\title{Provably Efficient Convergence of Primal-Dual Actor-Critic with Nonlinear Function Approximation}
\author[1]{Jing Dong}
\author[2]{Li Shen}
\author[1]{Yinggan Xu}
\author[1]{Baoxiang Wang}
\affil[1]{The Chinese University of Hong Kong, Shenzhen}
\affil[2]{JD Explore Academy}
\affil[ ]{\textit {\{jingdong, yingganxu\}@link.cuhk.edu.cn, mathshenli@gmail.com, bxiangwang@cuhk.edu.cn}}
\date{}
\allowdisplaybreaks

\begin{document}
\maketitle

\begin{abstract}
We study the convergence of the actor-critic algorithm with nonlinear function approximation under a nonconvex-nonconcave primal-dual formulation. Stochastic gradient descent ascent is applied with an adaptive proximal term for robust learning rates. We show the first efficient convergence result with primal-dual actor-critic with a convergence rate of $\mathcal{O}\left(\sqrt{\frac{\ln \left(N d G^2 \right)}{N}}\right)$ under Markovian sampling, where $G$ is the element-wise maximum of the gradient, $N$ is the number of iterations, and $d$ is the dimension of the gradient. Our result is presented with only the Polyak-\L{}ojasiewicz condition for the dual variables, which is easy to verify and applicable to a wide range of reinforcement learning (RL) scenarios. The algorithm and analysis are general enough to be applied to other RL settings, like multi-agent RL. Empirical results on OpenAI Gym continuous control tasks corroborate our theoretical findings.
\end{abstract}

\section{Introduction}
Actor-critic \citep{barto1983neuronlike,barto1989learning,konda1999actor} is one of the most successful algorithms in reinforcement learning. The algorithm features an actor, which learns the optimal policy that maximizes the long-term expected reward through sequential interactions with the environment, and a critic, which learns to approximate a value function that evaluates the performance of a policy.  
The actor-critic method effectively combines the benefits from policy-based algorithms \citep{williams1992simple,sutton2000policy, kakade2001natural,silver2014deterministic} and value-based algorithms \citep{barto1983neuronlike, watkins1992q, sutton1988learning, tesauro1992practical,hester2018deep}. 

Armed with recent developments in deep learning, the actor-critic algorithm gains empirical success in a variety of real applications \citep{haarnoja2018soft,fujimoto2018addressing, iqbal2019actor}. 
However, the underlying theory and limits have yet been fully understood. 
Most previous analyses have their limitations. \citet{castro2010convergent, maei2018convergent} establish asymptotic convergence in the original setting with an unknown sample complexity. Follow-up works that investigate finite-sample performance are conducted with two-timescale updates \citep{wu2020finite, hong2020two, doan2021finitetime} or linear function approximation \citep{xu2020non,xu2021doubly}, where the best known convergence rate is established to be $\mathcal{O}(\epsilon^{-2/3})$. 
It is left open to theoretically justify the actor-critic method's practical achievements in theory in its general setting.

We study a \textit{single-timescale} variant of the actor-critic algorithm with \textit{nonlinear function approximation} based on a minimax optimization formulation that combines the objectives for actor and critic \citep{dai2018boosting}. 
Under tabular or linear function approximation, the objective function could ideally serve as an essential indicator for the convergence and to select learning rates. Under nonlinear function approximation however, the value of the objective function is no longer a useful tool for convergence indication and learning rate selection due to the nonconvex-nonconcave structure. Previous empirical attempts on solving this formulation resort to local convexification techniques such as path regularization, which demands a high computational complexity.  Leveraging the use of proximal functions, which are shown to be effective for minimizing regrets in online learning, we derive an algorithm with an implicit proximal term for a better convergence rate. Our proximal function is chosen in a data-driven way, which is similar to adaptive gradient methods (e.g. AdaGrad \citep{duchi2011adaptive}).
Our adaptive method circumvents this computation cost while alleviating the need for manually tuning learning rates, with a guaranteed convergence rate.

In this paper, we show a convergence rate of $\mathcal{O}\left(\sqrt{\frac{\ln \left(N d G^2\right)}{N}}\right)$ under Markovian sampling and adaptive gradient, where $N$ is the number of total iterations, $d$ is the dimension of the gradient, and $G$ is the element-wise maximum value of the gradient. 
This implies a nearly optimal sample complexity of $\tilde{\mathcal{O}}(\epsilon^{-2})$ with a constant batch size that is independent of $N$ and $\epsilon$.
Our theorems are under Polyak-\L{}ojasiewicz (PL) condition with respect to the dual variable, which is a much weaker assumption than the Minty Variational Inequality (MVI) commonly seen in the nonconvex-nonconcave optimization literature \citep{lin2018solving,liu2019towards}. 
The PL condition can be further lifted at the cost of a polynomial convergence rate \citep{mangoubi2021greedy}. 
Our results show the effectiveness of adaptive gradient on single-timescale actor-critic with nonlinear function approximation, which has been practically deployed to reinforcement learning systems. We conduct extensive evaluations on OpenAI Gym continuous tasks to verify our theoretical findings.

Our analysis is flexible enough to be adapted to other reinforcement learning settings. As an illustration of that, we show that a similar theoretical guarantee of convergence holds for the multi-agent case. To our knowledge, this is the first finite-sample analysis of decentralized primal-dual multi-agent actor-critic reinforcement learning. Interestingly, the derivation points out the importance of communication between agents as more frequent communication accelerates the convergence, which agrees with the practical finding in the multi-agent reinforcement learning literature.

\section{Related Works}

\paragraph{Analysis of Actor-Critic Algorithms}
\begin{table}[th]
    \centering
    \resizebox{0.48\textwidth}{!}{\begin{tabular}{@{}llll@{}}
    \toprule
         & \begin{tabular}[c]{@{}l@{}}Sample\\ Complexity\end{tabular} & Sampling & \begin{tabular}[c]{@{}l@{}}Function \\ Approximation\end{tabular} \\ \midrule
    (1)  & $\mathcal{O}(\epsilon^{-2})$    &  i.i.d  & \begin{tabular}[c]{@{}l@{}}Overparametrized\\ 2-layer NN\end{tabular}                                                                    \\
    (2)  & $\tilde{\mathcal{O}}(\epsilon^{-2.5})$    &   Non-i.i.d.     &  Linear                                                                 \\
    (3)    &   $\mathcal{O}(\epsilon^{-2/3})$             &   Non-i.i.d.       &    Critic is linear \\
    This work     &  $\tilde{\mathcal{O}}(\epsilon^{-2})$  &  Markov         &  Nonlinear                                                                 \\ \bottomrule
    \end{tabular}}
    \caption{Summary of previous studies and this work. (1) \citep{wang2019neural}; (2) \citep{wu2020finite}; (3) \citep{hong2020two}. Note that $\tilde{\mathcal{O}}$ hides the logarithmic factor on $\epsilon$.}
\end{table}

The actor-critic algorithm is first proposed by Konda and Tsitsiklis, which is guaranteed to converge asymptotically \citep{konda1999actor}.
The natural actor-critic variant is later established by Bhatnagar, Sutton, Ghavamzadeh, and Lee with a similar guarantee \citep{bhatnagar2009natural}.

It is not until recent years that the finite-sample performance of the actor-critic algorithm is analyzed. 
\citet{yang2018finite} study the effect of the critic of batch actor-critic with nonlinear function approximation under i.i.d. sampling. They show that each limiting point of the actor updates is affected by the statistical error achieved by the critic with a constant factor.
This requires the critic to perform several rounds of empirical risk minimization under a two-timescale framework.
In the case of natural actor-critic, \citet{wang2019neural} prove a sublinear $\mathcal{O}(N^{-1/2})$ convergence, assuming that samples are independent and the function approximation is an overparametrized two-layer neural network.
In most practical reinforcement learning settings, obtaining independent samples is unrealistic and thus results under Markovian sampling are more desired. 
\citet{chen2021finitesample} extend the results to Markovian sampling and off-policy sampling with a sample complexity of $\mathcal{O}(\epsilon^{-3})$, but the analysis is restricted to linear function approximation.
\citet{wu2020finite} and \citet{xu2020non} both study actor-critic with non-i.i.d. sampling with the two-timescale structure. 
They also both assume that the critic takes a more restricted form of function approximation, such as linear approximation, and the convergence rates established under this assumption is $\tilde{O}(\epsilon^{-2.5})$.

Similar to our minimax optimization formulation, bilevel optimization is adapted to model actor-critic algorithms.
\citet{hong2020two} employ a two-timescale framework under non-i.i.d. sampling, and when the outer problem is strongly convex, the convergence rate is proved to be $\mathcal{O}(N^{-2/3})$.

\paragraph{Optimization Methods for Nonconvex Minimax Problems}
The nonconvex minimax problem serves as a fundamental framework for many machine learning applications, such as generative adversarial networks and actor-critic algorithms.
\citet{thekumparampil2019efficient} study the minimax problem when the objective function is nonconvex but concave. 
While they established a $\tilde{\mathcal{O}}(N^{-2})$ convergence rate, their algorithm admits to a double loop structure with an inner maximization.
\citet{pmlr-v132-abernethy21a} show a linear convergence with a second-order iterative algorithm, for a subclass of ``sufficiently bilinear'' functions.
In comparison, the PL condition we considered is much easier to verify compared to the sufficiently bilinear property and we only require first-order information.

\citet{nouiehed2019solving} consider the minimax problem with one-sided PL inequality and proposes a multi-step gradient descent ascent algorithm. 
Though the algorithm finds the global optimality within $\mathcal{O}(\epsilon^{-2})$ iterations, it requires multiple descents at one iteration. 
\citet{yang2020global} achieve a $\mathcal{O}(\epsilon^{-2})$ sample complexity with one-sided PL condition and alternative gradient descent ascent.
Both works hold only under deterministic gradients and are thus infeasible in practice for reinforcement learning. Without the MVI inequality and PL condition, local convergence results for nonconvex-nonconcave optimization are limited. \citet{mangoubi2021greedy} show a $\text{poly}(\epsilon^{-1})$ convergence rate with second-order information. 

\paragraph{Adaptive Gradient Methods for RL Algorithms}
The adaptive gradient methods largely ease the dull tuning process for primal-dual reinforcement learning and its adaptive nature often induces better empirical performance. 
Despite the common use of adaptive gradient methods in training reinforcement learning agents, limited results have been presented for their theoretical guarantees. 
Temporal difference learning is a popular value-based reinforcement learning method, and 
\citet{sun2020adaptive} show a $O(\epsilon^2 \log^4 (1/\epsilon))/\log (1/\rho)$ convergence rate with AdaGrad (a variant for adaptive gradient \citep{duchi2011adaptive}), where $\rho$ is a measure of the speed that the underlying Markov chain changes. 
The result is restricted to linear function approximation though.

We provide the theoretical guarantee of actor-critic through investigating adaptive gradient methods for its primal-dual formulation.

\paragraph{Analysis of Multi-Agent Reinforcement Learning}
The problem of cooperative multi-agent reinforcement learning has been analyzed mainly through value-based methods such as temporal difference methods. \citet{wai2018multi} formulate gradient temporal difference learning as a finite-sum primal-dual optimization problem and propose a distributed incremental aggregated gradient method with a linear convergence rate. \citet{pmlr-v97-doan19a} analyze $TD(0)$ with linear function approximation for agents with graph connections and show a convergence rate of $\mathcal{O}\left(1/KN\right)$ under constant step size, where $K$ is the number of agents and $N$ is the number of iterations. This approach towards $TD(0)$ recovers the best-known convergence rate for distributed convex optimization. 
Finite-sample analyses on multi-agent reinforcement learning are also extended to fitted Q-iterations \citep{zhang2021finite}. 
For multi-agent actor-critic methods, the best available result is the proof of the asymptotic convergence \citep{suttle2020multi}.

\section{Preliminaries}

We consider a discounted Markov decision process (MDP) denoted by the tuple $\mdp = ( \state, \action, \trans, \reward, \ga)$, where $\state$ is the state space, $\action$ is the action space, $\trans: \state \times \action \to \Delta(\state) $ is the transition probability kernel such that given a state-action pair $(s, a)$, it returns a probability distribution $s' \sim  \mathbb{P}(\cdot \mid  s, a)$ of the next state, $\reward: \state \times \action \to \mathbb{R}$ is the reward function, and $\ga$ is the discount factor. 

The goal of reinforcement learning is to learn a policy $\pi$, which takes $s \in \state$ as an input and outputs a distribution $a \sim \pi(\cdot \mid s)$ over the action space $\action$, to maximize the expected cumulative discounted reward
\[
\Expect{_{s_0} \Expect{_\pi \left[ \sum^\infty_{t=0} \ga^t \reward(s_t, a_t) \right]}} \,, 
\]
where $s_0 \sim \mu_0$ is a given initial state distribution. 

To evaluate the performance of the policy, the value function is defined to measure the long-term expected cumulative discounted reward as $V(s) = \Expect{ \left[ \sum^\infty_{t=0} \ga^t R(s_t, a_t) | s_0 = s \right]}$. 
Let $V^\ast$ be the optimal value function such that 
$V^\ast (s) = \max_{\pi} \Expect{ \left[ \sum^\infty_{t=0} \ga^t R(s_t, a_t) | s_0 = s \right]}$. 
The Bellman optimality equation states that 
\begin{align} \label{eq:opt}
    V^\ast (s_t) 
    &= \Gamma V^\ast (s_t) \nonumber\\
    &= \max_{a \in \action} \left\{ R(s_t,a_t) + \ga^t \Expect_{s_{t+1}}{[V^\ast (s_{t+1})]} \right\} \,.    
\end{align}
Equation (\ref{eq:opt}) can then be formulated into the following linear program (LP) \citep{bertsekas2000dynamic},
\begin{equation}
\begin{aligned}
V^\ast = \ & \underset{V}{\mathrm{minimize}}
& & (1 - \ga^t ) \Expect_{s_t}{[V(s_t)]} \\
& \mathrm{subject}\text{ }\mathrm{to}
& & V(s_t) \geq R(s_t, a_t) + \ga^t \Expect_{s_{t+1}}{[V (s_{t+1})]} \,. \label{eq:primal}
\end{aligned}
\end{equation}
Without loss of generality, we assume that the linear program is feasible, i.e., there exists an optimal policy for the given MDP.

The dual form (\ref{eq:dual}) of the linear program optimizes the policy directly and is hence appealing to reinforcement learning. 
One of the optimal policies can then be recovered through $\pi^\ast(a\mid s) = \frac{\rho^\ast (s,a)}{\sum_{a \in \action} \rho^\ast (s, a)}$ (Theorem 1 of \citep{dai2018boosting}), where $\rho^\ast$ is the optimal dual occupancy variable. 
\begin{equation}
\begin{aligned}
\ & \underset{\rho \geq 0}{\mathrm{maximize}}
&& \sum_{(s_t, a_t) \in \state \times \action } R(s_t, a_t) \rho (s_t, a_t), \forall s_{t+1} \in \state \\
& \mathrm{subject}\text{ }\mathrm{to}
&& \sum_{a \in \action} \rho (s_{t+1}, a) = (1 - \ga^t ) \mu(s_{t+1}) +   \ga^t \sum_{(s_t, a_t) \in \state \times \action } \rho(s_t, a_t) P(s_{t+1} \mid s_t, a_t)  \,. \label{eq:dual}
\end{aligned}
\end{equation}

When the strong duality holds, by \citep{dai2018boosting}, the equivalent saddle point problem (\ref{eq:multistep}) can be jointly optimized to learn both the policy and value functions. Note that in both discrete and continuous settings, the duality gap is zero and the strong duality holds \citep{dai2018boosting}. This approach of learning approximate policy and value functions simultaneously is known as the actor-critic method.
The formulation is
\begin{equation}
\begin{aligned}
    & \min_V \max_{\alp, \pi} \ L_k(V, \alp, \pi) \\
    = \ & \min_V \max_{\alp, \pi} \ (1 - \ga^{k+1}) \Expect_{\mu} [V(s_t)] 
    + \sum_{(s_t,a_t)^k_{t=0}, s_{k+1} } 
    \alp(s_t) \zeta (s_t, a_t, s_{t+1})\,,
\label{eq:multistep}
\end{aligned}
\end{equation}
where $\alp: \state \to \Delta(\state), \pi: \state \to \Delta (\action)$, 
$\zeta (s_t, a_t, s_{t+1}) = \p^k\li_{t=0} \pi(a_t \mid s_t) P(s_{t+1} \mid s_t, a_t) \delta((s_t,a_t)^k_{t=0}, s_{k+1})$,  and
$\delta((s_t,a_t)^k_{t=0}, s_{k+1}) = \sum^k\li_{t=0} \ga^t R(s_t, a_t) + \ga^{k+1} V(s_{k+1}) - V(s_t)$.

Assume that $\alp, \pi, V$ are parameterized by $u, \theta, \omega$, respectively. Let $\nabla_u L_k, \nabla_\theta L_k, \nabla_\omega L$ denote the gradients of Equation (\ref{eq:multistep}) with respect to each parameter. 
We obtain the gradients of Equation (\ref{eq:multistep}) with respect to each parameter as follows. 
Let $\delta_k = \delta((s_t,a_t)^k_{t=0}, s_{k+1})$ for simplicity.
\begin{align*}
    \nabla_u L_k 
    =\ &\Expect^\pi_\alp \left[ \nabla_u \log \alp (s_t)  \delta_k\right] \,, \\
    \nabla_\theta L_k 
    =\ & (1 - \ga^{k+1})  \Expect [\nabla_\theta V(s_t)] + \Expect^\pi_\alp \left[ \sum^k_{t=0} \nabla_\theta \log \pi(a_t \mid s_t) \delta_k + \nabla_\theta \delta_k  \right]  \,, \\
    \nabla_\omega L_k 
    =\ & (1 - \ga^{k+1}) \Expect [\nabla_\omega V(s_t)] + \sum_{(s_t,a_t)^k_{t=0}, s_{k+1} } 
    \alp(s_t) \nabla_\omega \zeta (s_t, a_t, s_{t+1})  \,.\label{eq:update_w}
\end{align*}
When tabular parametrization and only one-step bootstrap are applied (i.e., when $k$ in Equation (\ref{eq:multistep}) is set to 0), the minimax problem is convex-concave, and there exist efficient convergence results \citep{chen2016stochastic}. 
With nonlinear function approximation, the objective is in general nonconvex-nonconcave. \citet{dai2018boosting} solves only the dual problem by applying stochastic mirror descent with a proximal mapping. The drawback of this method is apparent as an inner minimization problem is required to be solved at each update, which demands a large computational budget. 
Similar primal-dual formulations with nonlinear approximation have established a convergence rate of $\mathcal{O}(\epsilon^{-4})$ with proximal mappings and non-adaptive gradients, which can be computationally expensive in practice \citep{dai2018sbeed}.

Inspired by recent advances on theoretical understanding of generative adversarial networks \citep{pmlr-v132-abernethy21a,liu2019towards}, which is a natural application of minimax optimization, we study the convergence of stochastic gradient descents ascent (SGDA) on actor-critic. 
We design a variant of SGDA with adaptive gradients, which dynamically incorporates the history of the gradients to construct more informative updates.

\section{Adaptive SGDA for Actor-Critic Methods}

\begin{algorithm}[htb]
    \caption{Adaptive SGDA (ASGDA)}
    \label{alg:sgda}
 \begin{algorithmic}[1]
    \STATE {\bfseries Input:} Learning rates $\eta_\omega, \eta_z = (\eta_u, \eta_\theta)$, batch size $M$, $H_0 = I$, $z=(u, \theta)$, $\hat{G}_z = G + \xi \left(D + \frac{2D}{1-\gamma}\right)^2 \cdot (D_u^2 + D_\theta^2)$, $\hat{G}_\theta = G + \xi (D + 2D_\omega)^2$
    \FOR{$k$ = 1, \dots, $N$}
    \STATE Start from $s \sim \alpha_{k} (s)$ where $\alpha_{k}$ is parametrized by $u_k$, collect samples $\tau_k =\{s_t, a_t, r_t, s_{t+1}\}^{M}_{t=0}$ following policy $\pi_{k}$ parametrized by $\hat{\theta}_{k-1}$ \\ 
    \COMMENT{ // Update gradient estimates}
    \STATE $\hat{g}_{\omega}(\hat{\omega}_k, z_k) =  \nabla_\omega L(\hat{\omega}_{k}, \hat{u}_{k}, \hat{\theta}_{k}, \tau_{k})$
    \STATE  $\hat{g}_{z}(\hat{\omega}_k, z_k) = \nabla_z L(\hat{\omega}_{k}, \hat{u}_{k}, \hat{\theta}_{k}, \tau_{k})$
    \\
    \COMMENT{ // Update primal learning parameters}
    \STATE $\hat{\omega}_{k} = \omega_{k-1} - \eta_\omega \left(I + \sqrt{\hat{H}_{\omega, k-1}^{-1}} \right)\hat{g}_{\omega}(\hat{\omega}_{k-1}, \hat{z}_{k-1})$
    \STATE $\omega_{k} = \omega_{k-1} - \eta_\omega \left(I + \sqrt{\hat{H}_{\omega, k}^{-1}} \right) \hat{g}_{\omega}(\hat{\omega}_k, \hat{z}_k)$ \\
    \COMMENT{ // Update dual learning parameters}
    \STATE $\hat{z}_{k} = z_{k-1} + \eta_z \left(I + \sqrt{\hat{H}_{z, k-1}^{-1}} \right)\hat{g}_{z}(\hat{\omega}_{k-1}, \hat{z}_{k-1})$ 
    \STATE $z_{k} = z_{k-1} + \eta_z \left(I +\sqrt{\hat{H}_{z, k}^{-1}} \right)\hat{g}_{z}(\hat{\omega}_k, \hat{z}_k)$
    \\
    \COMMENT{ // Update primal adaptive gradient parameter}
    \STATE $\hat{g}_{\omega, 0:k} = \frac{1}{\sqrt{2}\hat{G}_\omega}[\hat{g}_{\omega, 0:k-1}, \hat{g}_{\omega}(\hat{\omega}_k, \hat{z}_k)]$
    \STATE $\hat{h}_{\omega, k,i} = \|g_{\omega, 0:k, i} \|^2, i = 1, \dots, d$.
    \STATE $\hat{H}_{\omega, k} = \diag (\hat{h}_{\omega, k}) + \frac{1}{2}I$ \\

    \COMMENT{ // Update dual adaptive gradient parameter}
    \STATE $\hat{g}_{z, 0:k} = \frac{1}{\sqrt{2}\hat{G}_z}[\hat{g}_{z, 0:k-1}, \hat{g}_{z}(\hat{\omega}_k, \hat{z}_k)]$
    \STATE $\hat{h}_{z, k,i} =  \|\hat{g}_{z,k, i} \|^2$, $i = 1, \dots, d$.
    \STATE $\hat{H}_{z,k} = \diag (\hat{h}_{z, k-1}) + \frac{1}{2} I$
    \ENDFOR
 \end{algorithmic}
\end{algorithm}

To optimize our minimax objective, we present Adaptive SGDA (Algorithm \ref{alg:sgda}). 
Our algorithm is based on stochastic gradient descent ascent, where the descent step is performed with respect to the primal variable $\omega$ for the value function, and the ascent step is performed with respect to the dual variables $z=(u, \theta)$ for the policy. The primal and dual sides perform updates simultaneously based on a trajectory of collected samples from interacting with the environment. We use $\| \cdot \|$ to denote $\ell_2$ norm in this paper.

Our algorithm is implemented with an adaptive gradient described by the proximal terms $\hat{H}_\omega$, $\hat{H}_z$ (line 6 - 9). The proximal term is motivated by its use in online learning to reduce the regret, where the term aggregates the historical information to combat the online setting \citep{hazan2006logarithmic}. Such a term in our algorithm helps the step size to be robust to its initial value for a more stable learning process in practice. A similar utilization of the proximal term could be found in AdaGrad \citep{duchi2011adaptive} for more general cases in optimization. While past applications were limited to convex functions and i.i.d. sampling, we take the use of proximal function further to primal-dual optimization with nonconvex-nonconcave objective function and under Markovian sampling. Our algorithm implicitly optimizes the objective with a proximal term with Mahalanobis norm (e.g. $\hat{H}_{z,k} = \diag (\hat{h}_{z, k-1}) + \frac{1}{2} I$) and does so in a data-driven way. Notice that $\hat{H}_\omega$, $\hat{H}_z$ are chosen to be diagonal matrices for ease of computation (line 10 - 15). 

Comparing to previous results on primal-dual nonconvex-nonconcave optimization, which requires either double loop iteration \citet{nouiehed2019solving}, solution to {Minty Variational Inequality} \citet{Liu2020Towards}, or second-order information \citet{mangoubi2021greedy}, our results are with one-sided PL condition. This not only includes a wider range of RL tasks but is also much easier to verify in practice. Our algorithm is further analyzed with stochastic gradients, and our theoretical results match with best-known results for deterministic gradients up to a logarithm factor \citet{yang2020global}. Further, our algorithm performs updates based on trajectory information obtained through reinforcement learning algorithm interacting with the environment, and thus common sampling assumptions such as i.i.d. sampling no longer holds. Instead of directly employing this variant of the AdaGrad algorithm, we consider the case where the sampling is Markovian. Intuitively, when the sampled trajectory is long enough, the underlying Markov chain will tend to its stationary distribution. By bounding the deviation of the objective function along the trajectory, we can upper bound the variance of the gradients by constants related to the convergent rate of the MDP (lines 10 and 13).

We remark that our proposed method can proceed in a fully online fashion where $\omega$ and $z$ are updated after each interaction with the environment. This potentially leads to high variance. When implementing the algorithm, the gradient is averaged over a batch of updates instead. Intuitively, when the batch size $M$ is large, the experience encountered by the algorithm will be close to the unique stationary distribution if the Markov chain satisfies some requirements, e.g. the ergodicity. 

\section{Convergence Analysis}

Before presenting the theoretical guarantees for single-agent actor-critic with ASGDA, we first state and discuss a few assumptions needed for the analysis. 

Most of the previous analyses on nonconvex-nonconcave optimization problems utilize the MVI inequality assumption \citep{lin2018solving,liu2019towards,diakonikolas2021efficient}, which is unrealistic in real applications. Instead, we consider one-sided Polyak-\L{}ojasiewicz (PL) inequality for the dual variables only, which is relatively weaker compared to MVI. We can further lift the PL assumption, in which case the nonconvex-nonconcave primal-dual optimization with second-order information only can still enjoy its convergence with a $\text{poly}(\eps^{-1})$ convergence rate. This can be achieved by directly applying Theorem 3.1 of \citep{mangoubi2021greedy}. To present an algorithm that is also practically feasible, we discuss the rest of the results under PL inequality.

\begin{asmp}[PL condition for dual variables] 
\label{asmp:pl}
$L(\omega, z)$ is assumed to satisfy Polyak-\L{}ojasiewicz (PL) condition with respect to $z$ 
such that $\forall  \omega \in \mathbb{R}^{d}$ and for some constant $\mu$,
\begin{align*}
    \frac{1}{2} \left\| \nabla_z L(\omega,z) \right\|^2 
    \geq \mu (L(\omega, f^\ast(z)) - L(\omega,z))
\end{align*}
holds for all $z = (u, \theta) \in \mathbb{R}^{d} \times \mathbb{R}^{d}$.
\end{asmp}

Beyond the PL condition, we also assume that the gradients are Lipschitz continuous and bounded, respectively in Assumption \ref{asmp:wlipschitz} and \ref{asmp:bounded}. These assumptions are common among the optimization literature \citep{duchi2011adaptive,nguyen2017sarah}.

\begin{asmp}[Lipschitz continuity of the gradient]
\label{asmp:wlipschitz}
There exists a constant $C$ such that for all $(\omega, z), (\omega', z')$
\[
\left\| \nabla L(\omega, z) - \nabla L(\omega', z') \right\| 
    \leq C \left\|(\omega, z) - (\omega', z')\right\|  \,,
\]
where $\nabla L(\omega, z) = (\nabla_\omega L(\omega, z), \nabla_z L(\omega, z))$.
\end{asmp}

\begin{asmp}[Bounded gradient]
\label{asmp:bounded}
There exist constants $D_u$, $D_\theta$, $D_\omega$ such that for all $(u, \theta, \omega)$
\begin{align*}
&\| \nabla_u \log (\alpha_u(s))\| \leq D_u \,, \ \
\| \nabla_\theta \log (\pi_\theta (a|s) \| \leq D_\theta \,, \ \ \| \nabla_\omega V_\omega(s) \| \leq D_\omega \,.
\end{align*}
\end{asmp}

Immediately following the two assumptions, there exist constants $B_z, B_\omega$ such that $\| \nabla_z L \|^2 \leq B_z$, $\| \nabla_\omega L \|^2 \leq B_\omega$.

Beyond the above assumptions, we also need the following assumption regarding the underlying MDP, which is common for analyses under Markovian sampling. Time-homogeneous Markov chains with finite state space and uniformly ergodic Markov chains with general state space satisfy this assumption \citep{sun2020adaptive, xiong2020non}. 

\begin{asmp}[Geometric convergence rate of MDP]
\label{asmp:samples}
The MDP is irreducible and aperiodic for all $\pi$, and there exist constants $\xi > 0$ and $\rho \in (0,1)$ such that for all $\pi$ and $t \geq 0$
\[
\sup_{s \in \state} \left\| P( s_t , \cdot) - \kappa (\cdot ) \right\| \leq \xi \rho^t \,,
\]
where ${\kappa} (\cdot) $ is the stationary distribution of the Markov chain induced by policy $\pi$. 
\end{asmp}

The bounded reward assumption is common among the reinforcement learning literature \citep{agrawal2017posterior,wang2019privacy,sun2020adaptive}. As the value function measures the cumulative discounted reward, if the reward is allowed to be unbounded then the value function may not exist, which makes the problem ill-defined.

\begin{asmp}[Bounded reward]
\label{asmp:bounded_reward}
There exists a constant $D$ such that $|R(s, a)| \leq D, \forall s \in \state, a \in \action$.
\end{asmp}

To facilitate our analysis, \cite{rafique2018non, thekumparampil2019efficient, pmlr-v119-lin20a}, we introduce the envelope functions, which is a standard notion in nonconvex-nonconcave analysis. 
\begin{defn}[Envelop functions]
$\Phi (\omega) = L(\omega, f^\ast(\omega))$, where $f^\ast(\omega) = \argmax_{z} L(\omega,z)$.
\end{defn}

To measure the convergence of our algorithm, we consider the first-order stationary point of $\Phi(\omega)$, i.e., $\frac{1}{N} \sum^N_{k=1} \| \nabla_\omega \Phi (\omega_k)\|^2 \leq \epsilon$, which is standard in minimax optimization \citep{rafique2018non, thekumparampil2019efficient, pmlr-v119-lin20a, pmlr-v130-deng21a}. The envelop function can be seen as a surrogate of the minimax objective. As the strong duality holds for our formulation and $\omega$ parametrized the value function, the envelop function indirectly evaluates the convergence to the local optimal value function. Note that this envelop function is smooth \citep{nouiehed2019solving}, which is a property repeatedly used in our analysis. 

Armed with the above assumptions and definition, we give the following convergence guarantee for Algorithm \ref{alg:sgda}. 
\begin{thm}
\label{thm:single}
When $\eta_z = \min \left\{\frac{1}{C}, \frac{dG^2 N}{2\mu} \right\}$ and $\eta_\omega = \min \left\{\frac{1}{\beta}, \frac{\mu^2 \eta_z + 6 \eta_\omega \eta_z C^2 \mu }{3C^2 dG^2 N }, \sqrt{\frac{\mu^2 \eta_z + 6 \eta_\omega \eta_z C^2 \mu}{360 C^4 dG^2}} \right\}$, under Assumption \ref{asmp:pl}, \ref{asmp:wlipschitz}, \ref{asmp:bounded}, \ref{asmp:samples}, \ref{asmp:bounded_reward}, we have 
\begin{align*}
    \sum^N_{k=2}\Expect{\left[ \left\|  \nabla_\omega \Phi(\omega_{k})\right\|^2\right]} 
    = \mathcal{O} \left(\frac{\max \{\hat{G}_\omega^2, \hat{G}_z^2\} \ln \left(d N G^2 + \frac{1}{M} \sum^M_{m=1} \xi \rho^{2m} \right)}{N}\right) \,,
\end{align*}
where $N$ is the number of iterations and $\hat{G}_\omega$, $\hat{G}_z$ are constants.
\end{thm}

The bounded cumulative gradient of the envelop function suggests that our algorithm can achieve a stationary point. Unfortunately, this reveals little information regarding the optimality of the solution. The challenge of characterizing the optimality mainly stems from the use of nonlinear function approximation, under which our optimization objective will be nonconvex-nonconcave. If one would have obtained a characterization of the optimality, it will not only bridge the gap in the reinforcement learning literature, but also provide insights on other primal-dual optimization-based machine learning methods such as generative adversarial networks. The latter of which is known to be hard and has been open for years since it was proposed. 

We now provide a proof sketch to Theorem \ref{thm:single} while we defer the complete proof to the appendix. 

\textit{Proof sketch.} 
We start with defining a few notations for gradient estimation errors. Let $\epsilon_{z}(\hat{\omega}_k, \hat{z}_k) = \hat{g}_{z}(\hat{\omega}_k, \hat{z}_k) - \nabla_z L(\hat{\omega}_k, \hat{z}_k)$ and similarly $\epsilon_{\omega}(\hat{\omega}_k, \hat{z}_k) = \hat{g}_{\omega}(\hat{\omega}_k, \hat{z}_k) - \nabla_\omega L(\hat{\omega}_k, \hat{z}_k)$.

To bound the gradient estimation errors $\epsilon_\omega, \epsilon_z$ for $z = (u, \theta)$ under Markovian sampling, we use the property of the geometric convergence of the Markov chain to control them at a constant level. We show an example with $\epsilon_u$, and similar results can be obtained in the same way for $\epsilon_\omega, \epsilon_\theta$.
 
Intuitively, when the batch size is large, Assumption \ref{asmp:samples} on the Markov chain implies the stability of the chain and thus the variance is controlled at a constant level. By our objective function, the stochastic gradient can be expressed as
    \begin{align*}
        &\Expect{\left[\hat{g}_{u}(\hat{\omega}_k, \hat{z}_k) |s_0 = s\right]} \\
        = & \sum_{s, s' \in \state} \kappa(s) P(s' \mid s)  \left( \sum^M_{t=0} \gamma_t R(s_t, a_t) + \gamma_{k+1} V(s_{M+1}) - V(s) \right)\nabla_u \log \alpha \left( s, u_{k} \right) \\
        & + \sum_{s, s' \in \state}( P(s_M \mid s_{0} = s) - \kappa(s))P(s' \mid s)  \ \left( \sum^M_{t=0} \gamma_t R(s_t, a_t) + \gamma_{k+1} V(s_{M+1}) - V(s) \right)\nabla_u \log \alpha \left( s \right)\,,
    \end{align*}
where $\kappa( \cdot )$ is the stationary distribution of the MDP. 
    
Notice that the true gradient is exactly 
    \begin{align*}
        \nabla_u L(\hat{\omega}_{k}, \hat{u}_{k}, \hat{\theta}_{k}) 
        = & \sum_{s, s' \in \state} \kappa(s) P(s' \mid s) \left( \sum^M_{t=0} \gamma_t R(s_t, a_t)   + \gamma_{k+1} V(s_{M+1})- V(s) \right)\nabla_u \log \alpha \left( s \right)\,.
    \end{align*}
Hence, the estimation errors are upper bounded by 
\begin{align*}
        &\Expect[\| \epsilon_{z}(\hat{\omega}_k, \hat{z}_k) \|^2 ] \leq \frac{1}{M}\sum^M_{m=1}  \xi \rho^{2m} D_1\,, \\
        & \Expect[\|\epsilon_{\omega} (\hat{\omega}_k, \hat{z}_k) \|]^2  \leq \frac{1}{M}\sum^M_{m=1} \xi \rho^{2m}D_2\,,
    \end{align*}
where $D_1$, $D_2$ are constants.

Armed with the upper bound of the estimation error, we can show that the cumulative adaptive gradient is upper bounded by $\mathcal{O} (\log N)$ using the update rule and the concavity of logarithm. Then in Lemma \ref{lem:lemma1}, we use the smoothness of $\Phi(\omega)$ to obtain that
\begin{equation}
\begin{aligned} 
    &\Expect{[\Phi(\omega_{k})]} - \Expect{[\Phi(\omega_{k-1})]} \\
    \leq \ & \frac{6\eta_\omega C^2}{\mu } \left(\Phi(\omega_{k}) -  L \left(\omega_k, \hat{z}_k \right)  \right) 
    + 3 \eta_\omega C^2 \Expect{\left[ \left\|  \omega_k - \hat{\omega}_k \right\|^2\right]} \\
    & \ + 3 \eta_\omega \Expect{\left[ \left\| \epsilon_\omega\left(\hat{\omega}_k, \hat{z}_k \right)  \right\|^2\right]} 
    + \beta\eta_\omega^2  \Expect{\left[ \left\|\sqrt{\hat{H}_{\omega, k}^{-1}}\hat{g}_{\omega} \left(\hat{\omega}_k, \hat{z}_k \right)  \right\|^2\right]}  - \frac{\eta_\omega}{2}\Expect{\left[ \left\|  \nabla_\omega \Phi(\omega_{k})\right\|^2\right]} \,.
\label{eq:decompose}
\end{aligned}
\end{equation}

Then to bound the first term in Equation (\ref{eq:decompose}), instead of directly upper bounding it, we obtain an inequality for $L(\omega_k, z_{k-1}) - L(\omega_k, \hat{z}_k)$. 
By the smoothness of the $L$, we have
\begin{align*}
        L(\omega_k, z_{k-1}) - L(\omega_k, \hat{z}_k)
        \leq & \ \left\langle  \nabla L(\omega_k, \hat{z}_k),z_{k-1} - \hat{z}_k \right\rangle + \frac{C}{2} \left\| z_{k-1} - \hat{z}_k \right\|^2 \,.
\end{align*}
Then by PL inequality with respect to the dual variables, we have 
\begin{align*}
    \frac{\eta_z}{2}\left\| \nabla L(\omega_k, \hat{z}_k) \right\|^2  \geq \mu \eta_z \left( \Phi (\omega_k) - L(\omega_k, \hat{z}_k)\right) \,.
\end{align*}
By choosing the step size carefully to be within $0 < \eta_z \leq \frac{1}{C}$, we have 
\begin{align*}
    &\mu \eta_z \left( \Phi (\omega_k) - L(\omega_k, \hat{z}_k)\right) \\
    \leq & \ L(\omega_k, \hat{z}_k) - L(\omega_k, z_{k-1})  + \eta_z \left\| \nabla_z L(\omega_k, \hat{z}_k) - \hat{g}_z \left(\hat{\omega}_{k-1}, \hat{z}_{k-1} \right)  \right\|^2 + C \eta_z^2 \left\| \sqrt{\hat{H}_{z, k-1}^{-1}}\hat{g}_z \left(\hat{\omega}_{k-1}, \hat{z}_{k-1} \right)  \right\|^2 \,.
\end{align*}
By decomposing the second term,
\begin{align*}
    &\Phi(\omega_k) - L(\omega_{k-1}, z_{k-1}) \\
    = & \left(\Phi(\omega_{k-1}) - L(\omega_{k-1}, \hat{z}_{k-1}) \right) + \left(\Phi(\omega_k) - \Phi(\omega_{k-1})\right) + \left(L(\omega_{k-1}, \hat{z}_{k-1})  - L(\omega_{k-1}, z_{k-1}) \right)\,.
\end{align*} 
Repeatedly applying the above arguments, we obtain a bound for
 $\frac{1}{N}\sum^N_{k=2} \Expect \left[\Phi(\omega_k) - L(\omega_k, \hat{z}_k) \right] $. 
 
For our final result, it suffices to taking $\eta_\omega = \min \left\{\frac{1}{\beta}, \frac{\mu^2 \eta_z + 6 \eta_\omega \eta_z C^2 \mu }{3C^2 dG^2 N }, \sqrt{\frac{\mu^2 \eta_z + 6 \eta_\omega \eta_z C^2 \mu}{360 C^4 dG^2}} \right\}$ and $\eta_z = \min \left\{\frac{1}{C}, \frac{dG^2 N}{2\mu} \right\}$. Then we sum (\ref{eq:decompose}) over $N$ iterations and rearrange the terms to obtain 
\begin{align*}
    &\frac{1}{N}\sum^N_{k=2} \left(\Expect{[\Phi(\omega_{k})]} - \Expect{[\Phi(\omega_{k-1})]} \right) \\
    \leq & \ \mathcal{O} \left(\frac{\max \{\hat{G}_\omega^2, \hat{G}_z^2\}\ln \left(d N G^2 + \frac{1}{M} \sum^M_{m=1} \xi \rho^{2m} \right)}{N}\right)  + \mathcal{O} \left(\frac{1}{\mu N}\right) - \frac{\eta_\omega}{2} \sum^N_{k=2}\Expect{\left[ \left\|  \nabla_\omega \Phi(\omega_{k})\right\|^2\right]} \,.
\end{align*}
Lastly, we rearrange the terms to obtain
\begin{align*}
    \frac{1}{N}\sum^N_{k=2}\Expect{\left[ \left\|  \nabla_\omega \Phi(\omega_{k})\right\|^2\right]} 
    \leq & \  \mathcal{O} \left(\frac{\max \{\hat{G}_\omega^2, \hat{G}_z^2\} \ln \left(d N G^2 + \frac{1}{M} \sum^M_{m=1} \xi \rho^{2m} \right)}{N}\right)  + \mathcal{O} \left(\frac{1}{\mu N}\right)  \,. \quad \quad \quad \quad \quad \quad \quad \quad \quad \quad \quad \quad \quad \quad \ \ \ 
    \hfill\qedsymbol
\end{align*}

\begin{table*}[tbh]
\small
\begin{tabular}{c|c|c|c|c|c|l}
\hline
\textit{}                                                               & Pendulum-v0                                           & \begin{tabular}[c]{@{}c@{}}Inverted Double\\ Pendulum-v2\end{tabular} & Swimmer-v2                                              & Reacher-v2                                           & HalfCheetah-v2                                             & Hopper-v2                                                  \\ \hline
\begin{tabular}[c]{@{}c@{}}DualAC \\ w/ ASGDA\end{tabular}              & \textbf{-158.25}                                      & \textbf{7814.78}                                                        & \textbf{249.63}                                         & \textbf{-7.81}                                       & 1601.60                                                    & \textbf{2892.56}                                           \\ \hline
\begin{tabular}[c]{@{}c@{}}DualAC \\ w/ RMSProp\end{tabular}            & -162.18                                               & 207.21                                                                  & 178.75                                                  & -8.27                                                & 1589.92                                                    & 2719.55                                                    \\ \hline
\begin{tabular}[c]{@{}c@{}}DualAC\\ w/ Adam\end{tabular}                & -250.42                                               & 184.3                                                                   & 152.93                                                  & -8.13                                                & 1583.37                                                    & 2571.33                                                    \\ \hline
\begin{tabular}[c]{@{}c@{}}DualAC\\ w/ Gradient \\ Descent\end{tabular} & -161.69                                               & 216.3                                                                   & 238.27                                                  & -7.85                                                & 1638.49                                                    & 2799.16                                                    \\ \hline
PPO                                                                     & \begin{tabular}[c]{@{}c@{}}NA/\\ -266.98\end{tabular} & \begin{tabular}[c]{@{}c@{}}7034.41/\\ 1776.26\end{tabular}              & \begin{tabular}[c]{@{}c@{}}85.2/\\ 223.13\end{tabular}  & \begin{tabular}[c]{@{}c@{}}-10.52/\\ NA\end{tabular} & \begin{tabular}[c]{@{}c@{}}1280.48/\\ 2249.10\end{tabular} & \begin{tabular}[c]{@{}l@{}}2376.15/\\ 2306.41\end{tabular} \\ \hline
TRPO                                                                    & \begin{tabular}[c]{@{}c@{}}NA/\\ -245.11\end{tabular} & \begin{tabular}[c]{@{}c@{}}7026.27/\\ 3070.96\end{tabular}              & \begin{tabular}[c]{@{}c@{}}173.1/\\ 232.89\end{tabular} & \begin{tabular}[c]{@{}c@{}}-14.53/\\ NA\end{tabular} & \begin{tabular}[c]{@{}c@{}}520.41/\\ \textbf{2347.19}\end{tabular}  & \begin{tabular}[c]{@{}l@{}}2483.57/\\ 2299.58\end{tabular} \\ \hline
\end{tabular}
\caption{Scores achieved by our proposed algorithm and the baseline algorithms after 300 iterations. For PPO and TRPO, we use the scores that are reported in previous works. There are two references that report this score: top row is from \citep{fujita2018clipped}; bottom row is from \citep{dai2018boosting}.}
\label{table:score}
\end{table*}

\paragraph{Extension to Multi-Agent Reinforcement Learning}
The algorithm and the analysis developed in this paper are general enough to be extended to various reinforcement learning settings. We extend the results to cooperative multi-agent reinforcement learning (MARL) as an example, where each agent $i$ has a local copy of the parameters $\omega_i, z_i$ and may communicate the parameters through a communication network. To our best knowledge, this is the first finite-sample analysis for decentralized multi-agent primal-dual actor-critic algorithms. 

We present the main theoretical results for the multi-agent reinforcement learning problem here, while we defer the algorithm details and the proof to the appendix. We choose the collective cumulative gradients of all  $A$ agents, $\frac{1}{N} \sum^N_{k=2}  \sum^A_{i=1}  \left\| \nabla_\omega \Phi(\omega_k^i)\right\|^2 $, as our convergence criteria. When the collective cumulative gradient norm of the envelop function is bounded, each agent's cumulative gradient norm must also be bounded. 

\begin{thm} \label{thm:multi_thm}
When $\eta_\omega = \min \left\{\frac{1}{\beta}, 2(1 - \lambda), \sqrt{\frac{1}{dG^2 N}}\right\}$ and $\eta_z = \min \left\{\frac{1}{C}, \frac{\mu }{3C^2 d G_z^2 A N} \right\}$, under Assumptions \ref{asmp:pl}, \ref{asmp:wlipschitz}, \ref{asmp:bounded}, \ref{asmp:samples}, \ref{asmp:bounded_reward}, \ref{asmp:comm}, there is an algorithm that achieves
\begin{align*}
    \frac{1}{N} \sum^N_{k=2}  \sum^A_{i=1}  \left\| \nabla_\omega \Phi(\omega_k^i)\right\|^2 
    = & \ \mathcal{O} \left(\frac{A\max \{\hat{G}_\omega^2, \hat{G}_z^2\} \ln \left(d N G^2 + \frac{1}{M} \sum^M_{m=1} \xi \rho^{2m} \right)}{N \left(1 - \lambda\right)}\right) \,,
\end{align*}
where $N$ is the number of iterations, $A$ is the number of agents, $\lambda$ is the second largest eigenvalue of the communication network, and $\hat{G}_\omega$, $\hat{G}_z$ are constants.
\end{thm}

\section{Empirical Results}
We empirically evaluate our proposed algorithm on six different continuous control tasks from the Open AI Gym continuous control tasks \citep{brockman2016openai}. We compare our method against well-known optimization-based reinforcement learning algorithms such as proximal policy optimization (PPO) \citep{schulman2017proximal} and trust region proximal optimization (TRPO) \citep{schulman2015trust}. Then we discuss the performance of the primal-dual formulation with various commonly used optimizers, including Adam \cite{kingma2015adam}, RMSProp, and gradient descent to show the effectiveness of our method. Our results indicate that ASDGA is not only theoretically efficient but also at least as effective and efficient as other well-known optimizers in practice. To ensure reproducibility, every set of experiments is repeated with five different random seeds with the mean reward obtained presented. We describe the configuration of each algorithm and optimizer in detail, along with more experimental results, in the appendix.

The score achieved by each algorithm after 300 iterations in summarized in Table \ref{table:score}, where we use the reported results from \citet{fujita2018clipped} and \citet{dai2018boosting} for the performances of PPO and TRPO. As is indicated by the table, the primal-dual formulation outperforms PPO and TRPO and our optimizer ASGDA achieves the best final scores for five tasks out of six. Beyond this performance on practical tasks, the primal-dual formulation is also more intuitive to understand and easier to interpret. Compare to the commonly used approach which involves two separate optimization objectives for value function approximation and policy improvement, our formulation combines the objective through the duality and provides insights on the relationship between the actor and the critic during training. 


\section{Conclusion and Future Works}

We investigate the primal-dual formulation of the actor-critic method in reinforcement learning. We presented the first finite-sample analysis for single-scale algorithms and nonlinear function approximation.  Under Markovian sampling and adaptive gradients, we establish a convergence rate of $\Tilde{\mathcal{O}}\left(\eps^{-2}\right)$. This guarantee is under PL condition for only the dual variables. Our method is tested against various baseline algorithms and optimizers and outperforms the baselines in five out of six OpenAI Gym continuous control tasks. Our analysis is general enough to be applied to different reinforcement learning settings, where we show off a convergence guarantee for multi-agent actor-critic as an example.


Future work could include verifying conditions for the one-sided PL inequality under different methods of value function approximations and potentially lifting it with a convergence rate of some low-order polynomial. It is also possible to improve the convergence rate by investigating certain function approximation methods, such as neural networks. 

\bibliography{ref}

\begin{thebibliography}{60}
\providecommand{\natexlab}[1]{#1}
\providecommand{\url}[1]{\texttt{#1}}
\expandafter\ifx\csname urlstyle\endcsname\relax
  \providecommand{\doi}[1]{doi: #1}\else
  \providecommand{\doi}{doi: \begingroup \urlstyle{rm}\Url}\fi

\bibitem[Abernethy et~al.(2021)Abernethy, Lai, and
  Wibisono]{pmlr-v132-abernethy21a}
Jacob Abernethy, Kevin~A. Lai, and Andre Wibisono.
\newblock Last-iterate convergence rates for min-max optimization: Convergence
  of hamiltonian gradient descent and consensus optimization.
\newblock In \emph{International Conference on Algorithmic Learning Theory},
  2021.

\bibitem[Agrawal and Jia(2017)]{agrawal2017posterior}
Shipra Agrawal and Randy Jia.
\newblock Posterior sampling for reinforcement learning: {Worst}-case regret
  bounds.
\newblock In \emph{Advances in Neural Information Processing Systems}, 2017.

\bibitem[Barto et~al.(1983)Barto, Sutton, and Anderson]{barto1983neuronlike}
Andrew~Gehret Barto, Richard~S Sutton, and Charles~W Anderson.
\newblock Neuronlike adaptive elements that can solve difficult learning
  control problems.
\newblock \emph{IEEE Transactions on Systems, Man, and Cybernetics}, pages
  834--846, 1983.

\bibitem[Barto et~al.(1989)Barto, Sutton, and Watkins]{barto1989learning}
Andrew~Gehret Barto, Richard~S Sutton, and Christopher~JCH Watkins.
\newblock Learning and sequential decision making.
\newblock \emph{COINS Technical Report 89-95}, 1989.

\bibitem[Bertsekas(2000)]{bertsekas2000dynamic}
Dimitri Bertsekas.
\newblock \emph{Dynamic programming and optimal control: Vol. 1}.
\newblock Athena scientific Belmont, 2000.

\bibitem[Bhatnagar et~al.(2009)Bhatnagar, Sutton, Ghavamzadeh, and
  Lee]{bhatnagar2009natural}
Shalabh Bhatnagar, Richard~S Sutton, Mohammad Ghavamzadeh, and Mark Lee.
\newblock Natural actor-critic algorithms.
\newblock \emph{Automatica}, 45\penalty0 (11):\penalty0 2471--2482, 2009.

\bibitem[Brockman et~al.(2016)Brockman, Cheung, Pettersson, Schneider,
  Schulman, Tang, and Zaremba]{brockman2016openai}
Greg Brockman, Vicki Cheung, Ludwig Pettersson, Jonas Schneider, John Schulman,
  Jie Tang, and Wojciech Zaremba.
\newblock {OpenAI Gym}.
\newblock \emph{arXiv preprint arXiv:1606.01540}, 2016.

\bibitem[Castro and Meir(2010)]{castro2010convergent}
Dotan~Di Castro and Ron Meir.
\newblock A convergent online single time scale actor critic algorithm.
\newblock \emph{The Journal of Machine Learning Research}, 11:\penalty0
  367--410, 2010.

\bibitem[Chen and Wang(2016)]{chen2016stochastic}
Yichen Chen and Mengdi Wang.
\newblock Stochastic primal-dual methods and sample complexity of reinforcement
  learning.
\newblock \emph{arXiv preprint arXiv:1612.02516}, 2016.

\bibitem[Chen et~al.(2021)Chen, Khodadadian, and
  Maguluri]{chen2021finitesample}
Zaiwei Chen, Sajad Khodadadian, and Siva~Theja Maguluri.
\newblock Finite-sample analysis of off-policy natural actor-critic with linear
  function approximation.
\newblock \emph{arXiv preprint arXiv:2105.12540}, 2021.

\bibitem[Cutkosky and Orabona(2019)]{cutkosky2019momentum}
Ashok Cutkosky and Francesco Orabona.
\newblock Momentum-based variance reduction in non-convex {SGD}.
\newblock \emph{Advances in Neural Information Processing Systems}, 2019.

\bibitem[Dai et~al.(2018{\natexlab{a}})Dai, Shaw, He, Li, and
  Song]{dai2018boosting}
Bo~Dai, Albert Shaw, Niao He, Lihong Li, and Le~Song.
\newblock Boosting the actor with dual critic.
\newblock In \emph{International Conference on Learning Representations},
  2018{\natexlab{a}}.

\bibitem[Dai et~al.(2018{\natexlab{b}})Dai, Shaw, Li, Xiao, He, Liu, Chen, and
  Song]{dai2018sbeed}
Bo~Dai, Albert Shaw, Lihong Li, Lin Xiao, Niao He, Zhen Liu, Jianshu Chen, and
  Le~Song.
\newblock Sbeed: Convergent reinforcement learning with nonlinear function
  approximation.
\newblock In \emph{International Conference on Machine Learning},
  2018{\natexlab{b}}.

\bibitem[Deng and Mahdavi(2021)]{pmlr-v130-deng21a}
Yuyang Deng and Mehrdad Mahdavi.
\newblock Local stochastic gradient descent ascent: Convergence analysis and
  communication efficiency.
\newblock In \emph{International Conference on Artificial Intelligence and
  Statistics}, 2021.

\bibitem[Diakonikolas et~al.(2021)Diakonikolas, Daskalakis, and
  Jordan]{diakonikolas2021efficient}
Jelena Diakonikolas, Constantinos Daskalakis, and Michael Jordan.
\newblock Efficient methods for structured nonconvex-nonconcave min-max
  optimization.
\newblock In \emph{International Conference on Artificial Intelligence and
  Statistics}, 2021.

\bibitem[Doan et~al.(2019)Doan, Maguluri, and Romberg]{pmlr-v97-doan19a}
Thinh Doan, Siva Maguluri, and Justin Romberg.
\newblock Finite-time analysis of distributed {TD}(0) with linear function
  approximation on multi-agent reinforcement learning.
\newblock In \emph{International Conference on Machine Learning}, 2019.

\bibitem[Doan(2021)]{doan2021finitetime}
Thinh~T Doan.
\newblock Finite-time convergence rates of nonlinear two-time-scale stochastic
  approximation under {Markovian} noise.
\newblock \emph{arXiv preprint arXiv:2104.01627}, 2021.

\bibitem[Duchi et~al.(2011)Duchi, Hazan, and Singer]{duchi2011adaptive}
John Duchi, Elad Hazan, and Yoram Singer.
\newblock Adaptive subgradient methods for online learning and stochastic
  optimization.
\newblock \emph{Journal of Machine Learning Research}, 12\penalty0 (7), 2011.

\bibitem[Fujimoto et~al.(2018)Fujimoto, Hoof, and
  Meger]{fujimoto2018addressing}
Scott Fujimoto, Herke Hoof, and David Meger.
\newblock Addressing function approximation error in actor-critic methods.
\newblock In \emph{International Conference on Machine Learning}, 2018.

\bibitem[Fujita and Maeda(2018)]{fujita2018clipped}
Yasuhiro Fujita and Shin-ichi Maeda.
\newblock Clipped action policy gradient.
\newblock In \emph{International Conference on Machine Learning}, 2018.

\bibitem[Haarnoja et~al.(2018)Haarnoja, Zhou, Abbeel, and
  Levine]{haarnoja2018soft}
Tuomas Haarnoja, Aurick Zhou, Pieter Abbeel, and Sergey Levine.
\newblock Soft actor-critic: Off-policy maximum entropy deep reinforcement
  learning with a stochastic actor.
\newblock In \emph{International Conference on Machine Learning}, 2018.

\bibitem[Hazan et~al.(2006)Hazan, Kalai, Kale, and
  Agarwal]{hazan2006logarithmic}
Elad Hazan, Adam Kalai, Satyen Kale, and Amit Agarwal.
\newblock Logarithmic regret algorithms for online convex optimization.
\newblock In \emph{International Conference on Computational Learning Theory},
  2006.

\bibitem[Hester et~al.(2018)Hester, Vecerik, Pietquin, Lanctot, Schaul, Piot,
  Horgan, Quan, Sendonaris, Osband, et~al.]{hester2018deep}
Todd Hester, Matej Vecerik, Olivier Pietquin, Marc Lanctot, Tom Schaul, Bilal
  Piot, Dan Horgan, John Quan, Andrew Sendonaris, Ian Osband, et~al.
\newblock Deep {Q}-learning from demonstrations.
\newblock In \emph{The AAAI Conference on Artificial Intelligence}, 2018.

\bibitem[Hong et~al.(2020)Hong, Wai, Wang, and Yang]{hong2020two}
Mingyi Hong, Hoi-To Wai, Zhaoran Wang, and Zhuoran Yang.
\newblock A two-timescale framework for bilevel optimization: Complexity
  analysis and application to actor-critic.
\newblock \emph{arXiv preprint arXiv:2007.05170}, 2020.

\bibitem[Iqbal and Sha(2019)]{iqbal2019actor}
Shariq Iqbal and Fei Sha.
\newblock Actor-attention-critic for multi-agent reinforcement learning.
\newblock In \emph{International Conference on Machine Learning}, 2019.

\bibitem[Kakade(2001)]{kakade2001natural}
Sham Kakade.
\newblock A natural policy gradient.
\newblock In \emph{Advances in Neural Information Processing Systems}, 2001.

\bibitem[Karimi et~al.(2016)Karimi, Nutini, and
  Schmidt]{10.1007/978-3-319-46128-1_50}
Hamed Karimi, Julie Nutini, and Mark Schmidt.
\newblock Linear convergence of gradient and proximal-gradient methods under
  the {P}olyak-{Ł}ojasiewicz condition.
\newblock In \emph{European Conference on Machine Learning and Knowledge
  Discovery in Databases}, 2016.

\bibitem[Kingma and Ba(2015)]{kingma2015adam}
Diederik~P Kingma and Jimmy Ba.
\newblock Adam: A method for stochastic optimization.
\newblock In \emph{International Conference on Learning Representations}, 2015.

\bibitem[Konda and Tsitsiklis(1999)]{konda1999actor}
Vijay~R Konda and John~N Tsitsiklis.
\newblock Actor-citic agorithms.
\newblock In \emph{International Conference on Neural Information Processing
  Systems}, 1999.

\bibitem[Lian et~al.(2017)Lian, Zhang, Zhang, Hsieh, Zhang, and
  Liu]{lian2017can}
Xiangru Lian, Ce~Zhang, Huan Zhang, Cho-Jui Hsieh, Wei Zhang, and Ji~Liu.
\newblock Can decentralized algorithms outperform centralized algorithms? {A}
  case study for decentralized parallel stochastic gradient descent.
\newblock In \emph{Advances in Neural Information Processing Systems}, 2017.

\bibitem[Lin et~al.(2018)Lin, Liu, Rafique, and Yang]{lin2018solving}
Qihang Lin, Mingrui Liu, Hassan Rafique, and Tianbao Yang.
\newblock Solving weakly-convex-weakly-concave saddle-point problems as
  weakly-monotone variational inequality.
\newblock \emph{arXiv preprint arXiv:1810.10207}, 2018.

\bibitem[Lin et~al.(2020)Lin, Jin, and Jordan]{pmlr-v119-lin20a}
Tianyi Lin, Chi Jin, and Michael Jordan.
\newblock On gradient descent ascent for nonconvex-concave minimax problems.
\newblock In \emph{International Conference on Machine Learning}, 2020.

\bibitem[Liu et~al.(2019)Liu, Mroueh, Ross, Zhang, Cui, Das, and
  Yang]{liu2019towards}
Mingrui Liu, Youssef Mroueh, Jerret Ross, Wei Zhang, Xiaodong Cui, Payel Das,
  and Tianbao Yang.
\newblock Towards better understanding of adaptive gradient algorithms in
  generative adversarial nets.
\newblock In \emph{International Conference on Learning Representations}, 2019.

\bibitem[Liu et~al.(2020)Liu, Mroueh, Ross, Zhang, Cui, Das, and
  Yang]{Liu2020Towards}
Mingrui Liu, Youssef Mroueh, Jerret Ross, Wei Zhang, Xiaodong Cui, Payel Das,
  and Tianbao Yang.
\newblock Towards better understanding of adaptive gradient algorithms in
  generative adversarial nets.
\newblock In \emph{International Conference on Learning Representations}, 2020.

\bibitem[Maei(2018)]{maei2018convergent}
Hamid~Reza Maei.
\newblock Convergent actor-critic algorithms under off-policy training and
  function approximation.
\newblock \emph{arXiv preprint arXiv:1802.07842}, 2018.

\bibitem[Mangoubi and Vishnoi(2021)]{mangoubi2021greedy}
Oren Mangoubi and Nisheeth~K Vishnoi.
\newblock Greedy adversarial equilibrium: {An} efficient alternative to
  nonconvex-nonconcave min-max optimization.
\newblock In \emph{Symposium on Theory of Computing}, 2021.

\bibitem[Nguyen et~al.(2017)Nguyen, Liu, Scheinberg, and
  Tak{\'a}{\v{c}}]{nguyen2017sarah}
Lam~M Nguyen, Jie Liu, Katya Scheinberg, and Martin Tak{\'a}{\v{c}}.
\newblock Sarah: A novel method for machine learning problems using stochastic
  recursive gradient.
\newblock In \emph{International Conference on Machine Learning}, 2017.

\bibitem[Nouiehed et~al.(2019)Nouiehed, Sanjabi, Huang, Lee, and
  Razaviyayn]{nouiehed2019solving}
Maher Nouiehed, Maziar Sanjabi, Tianjian Huang, Jason~D Lee, and Meisam
  Razaviyayn.
\newblock Solving a class of non-convex min-max games using iterative first
  order methods.
\newblock \emph{Advances in Neural Information Processing Systems}, 2019.

\bibitem[Rafique et~al.(2018)Rafique, Liu, Lin, and Yang]{rafique2018non}
Hassan Rafique, Mingrui Liu, Qihang Lin, and Tianbao Yang.
\newblock Non-convex min-max optimization: Provable algorithms and applications
  in machine learning.
\newblock \emph{arXiv preprint arXiv:1810.02060}, 2018.

\bibitem[Schulman et~al.(2015)Schulman, Levine, Abbeel, Jordan, and
  Moritz]{schulman2015trust}
John Schulman, Sergey Levine, Pieter Abbeel, Michael Jordan, and Philipp
  Moritz.
\newblock Trust region policy optimization.
\newblock In \emph{International Conference on Machine Learning}, 2015.

\bibitem[Schulman et~al.(2017)Schulman, Wolski, Dhariwal, Radford, and
  Klimov]{schulman2017proximal}
John Schulman, Filip Wolski, Prafulla Dhariwal, Alec Radford, and Oleg Klimov.
\newblock Proximal policy optimization algorithms.
\newblock \emph{arXiv preprint arXiv:1707.06347}, 2017.

\bibitem[Silver et~al.(2014)Silver, Lever, Heess, Degris, Wierstra, and
  Riedmiller]{silver2014deterministic}
David Silver, Guy Lever, Nicolas Heess, Thomas Degris, Daan Wierstra, and
  Martin Riedmiller.
\newblock Deterministic policy gradient algorithms.
\newblock In \emph{International Conference on Machine Learning}, 2014.

\bibitem[Sun et~al.(2020)Sun, Shen, Chen, and Li]{sun2020adaptive}
Tao Sun, Han Shen, Tianyi Chen, and Dongsheng Li.
\newblock Adaptive temporal difference learning with linear function
  approximation.
\newblock \emph{arXiv preprint arXiv:2002.08537}, 2020.

\bibitem[Suttle et~al.(2020)Suttle, Yang, Zhang, Wang, Ba{\c{s}}ar, and
  Liu]{suttle2020multi}
Wesley Suttle, Zhuoran Yang, Kaiqing Zhang, Zhaoran Wang, Tamer Ba{\c{s}}ar,
  and Ji~Liu.
\newblock A multi-agent off-policy actor-critic algorithm for distributed
  reinforcement learning.
\newblock \emph{IFAC}, 53\penalty0 (2):\penalty0 1549--1554, 2020.

\bibitem[Sutton(1988)]{sutton1988learning}
Richard~S Sutton.
\newblock Learning to predict by the methods of temporal differences.
\newblock \emph{Machine Learning}, 3\penalty0 (1):\penalty0 9--44, 1988.

\bibitem[Sutton et~al.(2000)Sutton, McAllester, Singh, and
  Mansour]{sutton2000policy}
Richard~S Sutton, David~A McAllester, Satinder~P Singh, and Yishay Mansour.
\newblock Policy gradient methods for reinforcement learning with function
  approximation.
\newblock In \emph{Advances in Neural Information Processing Systems}, 2000.

\bibitem[Tesauro(1992)]{tesauro1992practical}
Gerald Tesauro.
\newblock Practical issues in temporal difference learning.
\newblock \emph{Machine Learning}, 8\penalty0 (3):\penalty0 257--277, 1992.

\bibitem[Thekumparampil et~al.(2019)Thekumparampil, Jain, Netrapalli, and
  Oh]{thekumparampil2019efficient}
Kiran~K Thekumparampil, Prateek Jain, Praneeth Netrapalli, and Sewoong Oh.
\newblock Efficient algorithms for smooth minimax optimization.
\newblock \emph{Advances in Neural Information Processing Systems}, 32, 2019.

\bibitem[Wai et~al.(2018)Wai, Yang, Hong, and Wang]{wai2018multi}
Hoi~To Wai, Zhuoran Yang, Mingyi Hong, and Zhaoran Wang.
\newblock Multi-agent reinforcement learning via double averaging primal-dual
  optimization.
\newblock \emph{Advances in Neural Information Processing Systems}, 2018.

\bibitem[Wang and Hegde(2019)]{wang2019privacy}
Baoxiang Wang and Nidhi Hegde.
\newblock Privacy-preserving q-learning with functional noise in continuous
  state spaces.
\newblock \emph{Advances in Neural Information Processing Systems}, 2019.

\bibitem[Wang et~al.(2019)Wang, Cai, Yang, and Wang]{wang2019neural}
Lingxiao Wang, Qi~Cai, Zhuoran Yang, and Zhaoran Wang.
\newblock Neural policy gradient methods: Global optimality and rates of
  convergence.
\newblock \emph{arXiv preprint arXiv:1909.01150}, 2019.

\bibitem[Watkins and Dayan(1992)]{watkins1992q}
Christopher~JCH Watkins and Peter Dayan.
\newblock Q-learning.
\newblock \emph{Machine Learning}, 8\penalty0 (3-4):\penalty0 279--292, 1992.

\bibitem[Williams(1992)]{williams1992simple}
Ronald~J Williams.
\newblock Simple statistical gradient-following algorithms for connectionist
  reinforcement learning.
\newblock \emph{Machine Learning}, 8\penalty0 (3-4):\penalty0 229--256, 1992.

\bibitem[Wu et~al.(2020)Wu, Zhang, Xu, and Gu]{wu2020finite}
Yue Wu, Weitong Zhang, Pan Xu, and Quanquan Gu.
\newblock A finite time analysis of two time-scale actor critic methods.
\newblock \emph{arXiv preprint arXiv:2005.01350}, 2020.

\bibitem[Xiong et~al.(2020)Xiong, Xu, Liang, and Zhang]{xiong2020non}
Huaqing Xiong, Tengyu Xu, Yingbin Liang, and Wei Zhang.
\newblock Non-asymptotic convergence of {Adam}-type reinforcement learning
  algorithms under {Markovian} sampling.
\newblock \emph{arXiv preprint arXiv:2002.06286}, 2020.

\bibitem[Xu et~al.(2020)Xu, Wang, and Liang]{xu2020non}
Tengyu Xu, Zhe Wang, and Yingbin Liang.
\newblock Non-asymptotic convergence analysis of two time-scale (natural)
  actor-critic algorithms.
\newblock \emph{arXiv preprint arXiv:2005.03557}, 2020.

\bibitem[Xu et~al.(2021)Xu, Yang, Wang, and Liang]{xu2021doubly}
Tengyu Xu, Zhuoran Yang, Zhaoran Wang, and Yingbin Liang.
\newblock Doubly robust off-policy actor-critic: Convergence and optimality.
\newblock In \emph{International Conference on Machine Learning}, 2021.

\bibitem[Yang et~al.(2020)Yang, Kiyavash, and He]{yang2020global}
Junchi Yang, Negar Kiyavash, and Niao He.
\newblock Global convergence and variance reduction for a class of
  nonconvex-nonconcave minimax problems.
\newblock \emph{Advances in Neural Information Processing Systems}, 2020.

\bibitem[Yang et~al.(2018)Yang, Zhang, Hong, and Ba{\c{s}}ar]{yang2018finite}
Zhuoran Yang, Kaiqing Zhang, Mingyi Hong, and Tamer Ba{\c{s}}ar.
\newblock A finite sample analysis of the actor-critic algorithm.
\newblock In \emph{IEEE Conference on Decision and Control}, 2018.

\bibitem[Zhang et~al.(2021)Zhang, Yang, Liu, Zhang, and Basar]{zhang2021finite}
Kaiqing Zhang, Zhuoran Yang, Han Liu, Tong Zhang, and Tamer Basar.
\newblock Finite-sample analysis for decentralized batch multi-agent
  reinforcement learning with networked agents.
\newblock \emph{IEEE Transactions on Automatic Control}, 2021.

\end{thebibliography}

\clearpage
\appendix
\onecolumn

\section{Notations and auxiliary Lemmas}
\subsection{Notations}
We first define a few notations. Recall that for notational convenience, we let $z = (u, \theta)$. Let $\epsilon_{z}(\hat{\omega}_k, \hat{z}_k) = \hat{g}_{z}(\hat{\omega}_k, \hat{z}_k) - \nabla_z L(\hat{\omega}_k, \hat{z}_k)$ and similarly $\epsilon_{\omega}(\hat{\omega}_k, \hat{z}_k) = \hat{g}_{\omega}(\hat{\omega}_k, \hat{z}_k) - \nabla_\omega L(\hat{\omega}_k, \hat{z}_k)$. We let $G$ to be an element-wise upper bound of the gradient such that $\max_i \nabla_z L \leq G$ and $\max_i \nabla_\omega L \leq G$ for $i = 1,\dots, d$. 

\subsection{Proof of Lemma \ref{lem:noise}}
\begin{lem} 
    \label{lem:noise}
    With Assumption \ref{asmp:bounded}, \ref{asmp:samples}, and \ref{asmp:bounded_reward},
    \begin{align*}
        &\Expect[\| \epsilon_{z}(\hat{\omega}_k, \hat{z}_k) \|^2 ] \leq \frac{1}{M}\sum^M_{m=1}  \xi \rho^{2m} \left(D + \frac{2D}{1-\gamma}\right)^2 \cdot (D_u^2 + D_\theta^2) \,, \\
        & \Expect[\|\epsilon_{\omega} (\hat{\omega}_k, \hat{z}_k) \|]^2  \leq \frac{1}{M}\sum^M_{m=1} \xi \rho^{2m} (D + 2D_\omega)^2 \,,
    \end{align*}
    where $D$, $D_\omega$, $D_u$, $D_\theta$ are constants.
\end{lem}

\begin{proof}
    We provide the proof of $\Expect{[ \|\epsilon_{u}(\hat{\omega}_k, \hat{z}_k) \|^2]}$ in detail, while $\Expect{[ \|\epsilon_{\omega} (\hat{\omega}_k, \hat{z}_k) \|^2]}$ 
    and $\Expect{[ \|\epsilon_{\theta} (\hat{\omega}_k, \hat{z}_k) \|^2]}$ can be upper bounded in similar ways. 
    Our proof holds similarity with Lemma 2 from \cite{sun2020adaptive}, though their result only holds for temporal difference learning.
    
    Intuitively, when the batch size is large, the assumptions on the Markov chain \ref{asmp:samples} implies some stability and thus the variance is controlled at a constant level. By our objective function, the stochastic gradient can be expressed as
    \begin{align*}
        &\Expect{\left[\hat{g}_{u}(\hat{\omega}_k, \hat{z}_k) |s_0 = s\right]}\\
        =& \sum_{s, s' \in \state} \kappa(s) P(s' | s) \left( \sum^M_{t=0} \gamma_t R(s_t, a_t) + \gamma_{k+1} V(s_{M+1}) - V(s) \right)\nabla_u \log \alpha \left( s, u_{k} \right) \\
         &+  \sum_{s, s' \in \state}( P(s_M | s_{0} = s) - \kappa(s))P(s' | s) \left( \sum^M_{t=0} \gamma_t R(s_t, a_t) + \gamma_{k+1} V(s_{M+1}) - V(s) \right)\nabla_u \log \alpha \left( s \right)\,,
    \end{align*}
    where $\kappa( \cdot )$ is the stationary distribution of the MDP. 
    
    Notice that the true gradient is exactly 
    \begin{align*}
        \nabla_u L(\hat{\omega}_{k}, \hat{u}_{k}, \hat{\theta}_{k}) = \sum_{s, s' \in \state} \kappa(s) P(s' | s) \left( \sum^M_{t=0} \gamma_t R(s_t, a_t) + \gamma_{k+1} V(s_{M+1}) - V(s) \right)\nabla_u \log \alpha \left( s \right)\,,
    \end{align*}
    
    For a discounted MDP with discount factor $\gamma$, the value function can be upper bounded by $\frac{D}{1 - \gamma}$. 
    Thus by Assumption \ref{asmp:bounded}, \ref{asmp:samples}, and \ref{asmp:bounded_reward}, and $\Expect{\left[\Expect{[\cdot | s_0 = 0]}\right]} = \Expect{[\cdot]}$, 
    $\Expect[\| \epsilon_{u} (\hat{\omega}_k, \hat{z}_k) \|^2 ] \leq \frac{1}{M}\sum^M_{m=1}  \xi \rho^{2M} \left(D + \frac{2D}{1-\gamma}\right)^2 \cdot D_u^2$. By rewriting the above equation for $\theta$ and using $\Expect{\left[\Expect{[\cdot | s_0 = 0]}\right]} = \Expect{[\cdot]}$ again, we have
    $\Expect[\| \epsilon_{z} (\hat{\omega}_k, \hat{z}_k) \|^2 ] \leq \frac{1}{M}\sum^M_{m=1}  \xi \rho^{2M} \left(D + \frac{2D}{1-\gamma}\right)^2 \cdot (D_u^2 + D_\theta^2)$. 

    For $\Expect[\|\epsilon_{\omega} (\hat{\omega}_k, \hat{z}_k) \|]^2 $, notice that $\alpha(s)$ is a probability distribution assigned to the states. Thus we have $\|\alpha(s)\|^2  \leq 1$. 
    Hence, we have $\Expect[\|\epsilon_{\omega} (\hat{\omega}_k, \hat{z}_k) \|]^2  \leq \frac{1}{M}\sum^M_{m=1} \xi \rho^{2M} (D + 2D_\omega)^2$. \qedhere
\end{proof}

\subsection{Proof of Lemma \ref{lem:ada_gradient}}
\begin{lem}\label{lem:ada_gradient}
    With Algorithm \ref{alg:sgda}, the cumulative adaptive gradients are upper bounded as
    \begin{align*}
        \sum^N_{k=1} \left\| \sqrt{\hat{H}_{z, k}^{-1}} \hat{g}_{z}(\omega_k, z_k) \right\|^2 
        \leq \ &  2\hat{G}_z^2\ln \left(1 + 2\sum^N_{k=1} \left(d G^2 +   \frac{1}{M}\sum^M_{m=1}  \left(\xi \rho^{2m} \left(D + \frac{2D}{1-\gamma}\right)^2  (D_u^2 + D_\theta^2)\right)\right)\right)\,,
    \end{align*}
    \begin{align*}
        \sum^N_{k=1} \left\| \sqrt{\hat{H}_{\omega, k}^{-1}} \hat{g}_{\omega}(\omega_k, z_k) \right\|^2 
        \leq \ 2\hat{G}_\omega^2\ln \left(1 + 2\sum^N_{k=1} \left(d G^2 +  \frac{1}{M}\sum^M_{m=1} \xi \rho^{2m} (D + 2D_\omega)^2\right)\right)\,, 
    \end{align*}
    \begin{align*}
        \sum^N_{k=1}  \Expect{\left[ \left\| \epsilon_{\omega}(\omega_k, z_k)  \right\|^2\right]}
        \leq \ &  2\hat{G}_\omega^2\ln \left(1 + 2\sum^N_{k=1} \left(d G^2 +  \frac{1}{M}\sum^M_{m=1} \xi \rho^{2m} (D + 2D_\omega)^2\right)\right)\,,
    \end{align*}
    and
    \begin{align*}
        \sum^N_{k=1}  \Expect{\left[ \left\| \epsilon_{z}(\omega_k, z_k)  \right\|^2\right]} 
        \leq \ & 2\hat{G}_z^2\ln \left(1 + 2\sum^N_{k=1} \left(d G^2 +   \frac{1}{M}\sum^M_{m=1}  \left(\xi \rho^{2m} \left(D + \frac{2D}{1-\gamma}\right)^2  (D_u^2 + D_\theta^2)\right)\right)\right)\,.
    \end{align*}
\end{lem}
\begin{proof}
    By the setting of $\hat{H}_{z, k}^{-1}$, 
    \begin{align*}
        \sum^N_{k=1} \left\| \sqrt{\hat{H}_{z, k}^{-1}} \hat{g}_{z}(\omega_k, z_k) \right\|^2 
        = \ & \sum^N_{k=1} \sum^d_{i=1} \frac{\hat{g}_{z, i}(\omega_k, z_k)^2}{ \left(\hat{h}_{z, k, i} + \frac{1}{2} \right) } \\
        = \ &\sum^N_{k=1}  \frac{\sum^d_{i=1} \hat{g}_{z}(\omega_k, z_k)^2}{\sum^{k}_{t=1} \sum^d_{i=1} \frac{g_{z, t,i}^2}{2 \hat{G}_z^2} + \frac{1}{2}} \\
        \leq \ & 2\hat{G}_z^2 \sum^N_{k=1}  \frac{\sum^d_{i=1} \hat{g}_{z}(\omega_k, z_k)^2}{\sum^{k}_{t=1} \sum^d_{i=1} g_{z, t,i}^2 + \frac{1}{2}}\,.
    \end{align*}
    By Lemma 4 of \citet{cutkosky2019momentum}, which takes advantage of the concavity of the log function, 
    we have 
    \begin{align*}
        \sum^N_{k=1} \frac{\sum^d_{i=1} \hat{g}_{z}(\omega_k, z_k)^2}{\sum^{k}_{t=1} \sum^d_{i=1}  g_{z,i}(\omega_t, z_t)^2 + \frac{1}{2}}
        \leq \ln \left(1 + 2\sum^{N}_{k=1} \sum^d_{i=1}  g_{z,i}(\omega_k, z_k)^2 \right)\,.
    \end{align*}
    Recall that $G$ is a constant such that $\max_i \nabla_z L \leq G$ for $i = 1,\dots, d$. 

    Then taking expectation and using \ref{lem:noise}, 
    \[\Expect{\left[\sum^d_{i=1}  g_{z,i}(\omega_k, z_k)^2\right]} \leq 
                            d G^2 + \Expect{\left[\| \epsilon_{z} (\omega_k, z_k) \|^2\right]} 
                             \leq d G^2 + \frac{1}{M}\sum^M_{m=1}  \left(\xi \rho^{2m} \left(D + \frac{2D}{1-\gamma}\right)^2  (D_u^2 + D_\theta^2)\right)\,.\]
    Hence,
    \begin{align*}
        \sum^N_{k=1} \left\| \sqrt{\hat{H}_{z, k}^{-1}} \hat{g}_{z}(\omega_k, z_k) \right\|^2 
        \leq \ 2\hat{G}_z^2\ln \left(1 + 2\sum^N_{k=1} \left(d G^2 +   \frac{1}{M}\sum^M_{m=1}  \left(\xi \rho^{2m} \left(D + \frac{2D}{1-\gamma}\right)^2  (D_u^2 + D_\theta^2)\right)\right)\right) \,. 
    \end{align*}

    Similarily, we have 
    \begin{align*}
        \sum^N_{k=1} \left\| \sqrt{\hat{H}_{\omega, k}^{-1}} \hat{g}_{\omega}(\omega_k, z_k) \right\|^2 
        \leq 2\hat{G}_\omega^2\ln \left(1 + 2\sum^N_{k=1} \left(d G^2 +  \frac{1}{M}\sum^M_{m=1} \xi \rho^{2m} (D + 2D_\omega)^2\right)\right)\,. 
    \end{align*}

    With the cumulative noise and the definition of $\epsilon_{\omega} (\omega_k, z_k)  $, notice that 
    \begin{align*}
        \sum^N_{k=1}  \Expect{\left[ \left\| \epsilon_{\omega} (\omega_k, z_k)  \right\|^2\right]}  
        = \ & \sum^N_{k=1}  \Expect{\left[ \left\| \hat{g}_{\omega} (\omega_k, z_k) - \nabla_\omega L (\hat{\omega}_k, \hat{z}_k)\right\|^2\right]} \\
        \leq \ &  \sum^N_{k=1}  \Expect{\left[ \left\| \hat{g}_{\omega} (\omega_k, z_k) \right\|^2\right]} \\
        \leq \ & \sum^N_{k=1}  \Expect{\left[ \left\| \sqrt{\hat{H}_{\omega, k}^{-1}}\hat{g}_{\omega} (\omega_k, z_k) \right\|^2\right]} \,,
    \end{align*}
    where the last inequality is from the observation that $\sqrt{\hat{H}_{\omega, k}^{-1}} \succeq I$ by the update rules. 
    Thus, 
    \begin{align*}
        \sum^N_{k=1}  \Expect{\left[ \left\| \epsilon_{\omega} (\omega_k, z_k)  \right\|^2\right]} 
        \leq \ &  \sum^N_{k=1}  \Expect{\left[ \left\| \sqrt{\hat{H}_{\omega, k}^{-1}}\hat{g}_{\omega} (\omega_k, z_k) \right\|^2\right]} \\
        \leq \ & 2\hat{G}_\omega^2\ln \left(1 + 2\sum^N_{k=1} \left(d G^2 +  \frac{1}{M}\sum^M_{m=1} \xi \rho^{2m} (D + 2D_\omega)^2\right)\right) \,.
    \end{align*}
    Similarily, the argument holds for $\sum^N_{k=1}  \Expect{\left[ \left\| \epsilon_{z} (\omega_k, z_k)  \right\|^2\right]} $.
\end{proof}

\subsection{Proof of Lemma \ref{lem:smooth}}

\begin{lem}
    \label{lem:smooth}
    Under Assumption \ref{asmp:wlipschitz}, for any two $\omega, \omega'$ and $z, z'$,
    \begin{align*}
        \left\| \nabla_\omega L(\omega, z) - \nabla_\omega L(\omega', z') \right\|^2 
        \leq C^2 \| \omega - \omega'\|^2 + C^2 \| z - z'\|^2 \,,
     \end{align*}
     and 
     \begin{align*}
        \left\| \nabla_z L(\omega, z) - \nabla_z L(\omega', z') \right\|^2 
        \leq C^2 \| \omega - \omega'\|^2 + C^2 \| z - z'\|^2 \,.
     \end{align*}
\end{lem}
\begin{proof}
    By Assumption \ref{asmp:wlipschitz},
    \begin{align*}
       \left\| \nabla_\omega L(\omega, z) - \nabla_\omega L(\omega', z') \right\|^2 
       \leq \left\| \nabla_z L(\omega, z) - \nabla_z L(\omega', z') \right\|^2 
       \leq C^2 \| \omega - \omega'\|^2 + C^2 \| z - z'\|^2 \,.
    \end{align*}
    Similarly, the argument holds for $\left\| \nabla_z L(\omega, z) - \nabla_z L(\omega', z') \right\|^2$.
\end{proof}

\subsection{Smoothness of envelop function}
\begin{prop}[\cite{nouiehed2019solving}, smoothness of envelop functions]
\label{prop:envelop}
If $L(\omega, u, \theta)$ satisfies $\mu$-PL condition and is $C$ smooth, then $\Phi(\omega)$ 
is $\beta = \kappa C/2  + C$ smooth, where $\kappa = C/\mu$.
\end{prop}
\section{Convergence analysis for single-agent case}
\subsection{Proof of Lemma \ref{lem:lemma1}}
\begin{lem}
\label{lem:lemma1}
If Assumption \ref{asmp:wlipschitz}, \ref{asmp:bounded} hold and $ 0 < \eta_\omega \leq  \frac{1}{\beta}$, then 
\begin{align*}
    &\Expect{[\Phi(\omega_{k})]} - \Expect{[\Phi(\omega_{k-1})]} \\
    \leq \ & \frac{6\eta_\omega C^2}{\mu } \left(\Phi(\omega_{k}) -  L \left(\omega_k, \hat{z}_k \right)  \right) 
    + 3 \eta_\omega C^2 \Expect{\left[ \left\|  \omega_k - \hat{\omega}_k \right\|^2\right]}\\
    \ &+ 3 \eta_\omega \Expect{\left[ \left\| \epsilon_\omega\left(\hat{\omega}_k, \hat{z}_k \right)  \right\|^2\right]} 
    + \beta\eta_\omega^2  \Expect{\left[ \left\|\sqrt{\hat{H}_{\omega, k}^{-1}}\hat{g}_{\omega} \left(\hat{\omega}_k, \hat{z}_k \right)  \right\|^2\right]}
    - \frac{\eta_\omega}{2}\Expect{\left[ \left\|  \nabla_\omega \Phi(\omega_{k})\right\|^2\right]} \,,
\end{align*}
and 
\begin{align*}
    &\Expect{[\Phi(\omega_{k})]} - \Expect{[\Phi(\omega_{k-1})]} \\
    \leq \ & \frac{6\eta_\omega C^2}{\mu } \left(\Phi(\omega_{k-1}) -  L \left(\omega_{k-1}, \hat{z}_{k-1} \right)  \right) 
    + 3 \eta_\omega C^2 \Expect{\left[ \left\|  \omega_{k-1} - \hat{\omega}_k \right\|^2\right]}\\
    \ &+ 3 \eta_\omega C^2 \Expect{\left[ \left\|  \hat{z}_k - \hat{z}_{k-1} \right\|^2\right]}
    + 3 \eta_\omega \Expect{\left[ \left\| \epsilon_\omega\left(\hat{\omega}_k, \hat{z}_k \right)  \right\|^2\right]} 
    + \beta\eta_\omega^2  \Expect{\left[ \left\|\sqrt{\hat{H}_{\omega, k}^{-1}}\hat{g}_{\omega} \left(\hat{\omega}_k, \hat{z}_k \right)  \right\|^2\right]} \,.
\end{align*}
\end{lem}

\begin{proof}
By the smoothness property of the envelop function $\Phi$ from Proposition \ref{prop:envelop} and taking expectation, we have 
\begin{align*}
     \Expect{[\Phi(\omega_{k})]} - \Expect{[\Phi(\omega_{k-1})]} 
    \leq -\Expect{\left[ \left\langle \nabla_\omega \Phi(\omega_{k}), \omega_{k-1} - \omega_{k} \right\rangle \right]} + \frac{\beta}{2}\Expect{\left[ \left\| \omega_{k-1} - \omega_{k}\right\|^2\right]} \,.
\end{align*}
Due to the update rules of Algorithm \ref{alg:sgda}, 
we have 
\begin{align*}
    \omega_{k-1} - \omega_k 
    &= \omega_{k-1} - \left(\omega_{k-1} - \eta_\omega \left(I + \sqrt{\hat{H}_{\omega, k}^{-1}} \right)\hat{g}_\omega \left(\hat{\omega}_k, \hat{z}_k \right)  \right) \\
    &=  \eta_\omega \left(I + \sqrt{\hat{H}_{\omega, k}^{-1}} \right)\hat{g}_\omega \left(\hat{\omega}_k, \hat{z}_k \right)   \,.
\end{align*}
By the identity that $\langle a, b \rangle = \frac{1}{2} \| a\|^2 + \frac{1}{2} \| b\|^2 - \frac{1}{2} \| a-b\|^2 $, 
\begin{align*}
    &\Expect{[\Phi(\omega_{k})]} - \Expect{[\Phi(\omega_{k-1})]} \\
    = \ & - \eta_\omega \Expect{\left[ \left\langle \nabla_\omega \Phi(\omega_k),  \left(I + \sqrt{\hat{H}_{\omega, k}^{-1}} \right)\hat{g}_{\omega} \left(\hat{\omega}_k, \hat{z}_k \right)  \right\rangle \right]} 
        + \frac{\beta \eta_\omega^2 }{2}\Expect{\left[ \left\|  \left(I + \sqrt{\hat{H}_{\omega, k}^{-1}} \right)\hat{g}_{\omega} \left(\hat{\omega}_k, \hat{z}_k \right)   \right\|^2\right]} \\
    = \ &\frac{\eta_\omega}{2} \Expect{\left[ \left\|  \nabla_\omega \Phi(\omega_{k}) -  \left(I + \sqrt{\hat{H}_{\omega, k}^{-1}} \right)\hat{g}_{\omega} \left(\hat{\omega}_k, \hat{z}_k \right)   \right\|^2\right]} 
    -\frac{\eta_\omega}{2}\Expect{\left[ \left\|  \nabla_\omega \Phi(\omega_{k})\right\|^2\right]} \\
    + \ &\frac{\beta \eta_\omega^2 - \eta_\omega}{2} \Expect{\left[ \left\|  \left(I + \sqrt{\hat{H}_{\omega, k}^{-1}} \right)\hat{g}_{\omega} \left(\hat{\omega}_k, \hat{z}_k \right)   \right\|^2\right]} \,.
\end{align*}
Using the inequality $\left\| a + b\right\|^2 \leq 2\left\| a\right\|^2 + 2\left\| b\right\|^2$,
\begin{align*}
    \Expect{[\Phi(\omega_{k})]} - \Expect{[\Phi(\omega_{k-1})]} 
    \leq \ & \frac{\eta_\omega}{2} \left( \Expect{\left[ \left\|  \nabla_\omega \Phi(\omega_{k}) - \hat{g}_{\omega} \left(\hat{\omega}_k, \hat{z}_k \right)    \right\|\right]} +  \Expect{\left[ \left\| \sqrt{\hat{H}_{\omega, k}^{-1}}\hat{g}_{\omega} \left(\hat{\omega}_k, \hat{z}_k \right)   \right\|\right]} \right)^2 \\
    \ &+\frac{\beta \eta_\omega^2 - \eta_\omega}{2}\left( \Expect{\left[ \left\|\hat{g}_{\omega} \left(\hat{\omega}_k, \hat{z}_k \right)    \right\|^2\right]} + \Expect{\left[ \left\|\sqrt{\hat{H}_{\omega, k}^{-1}}\hat{g}_{\omega} \left(\hat{\omega}_k, \hat{z}_k \right)    \right\|^2\right]} \right)^2\\
    \ &-\frac{\eta_\omega}{2}\Expect{\left[ \left\|  \nabla_\omega \Phi(\omega_{k})\right\|^2\right]} \\
    \leq \ & \eta_\omega \Expect{\left[ \left\|  \nabla_\omega \Phi(\omega_{k}) -  \hat{g}_{\omega} \left(\hat{\omega}_k, \hat{z}_k \right)     \right\|^2\right]} + \eta_\omega \Expect{\left[ \left\| \sqrt{\hat{H}_{\omega, k}^{-1}}\hat{g}_{\omega} \left(\hat{\omega}_k, \hat{z}_k \right)   \right\|^2\right]} \\
    \ &+ (\beta \eta_\omega^2 - \eta_\omega) \Expect{\left[ \left\| \hat{g}_{\omega} \left(\hat{\omega}_k, \hat{z}_k \right)  \right\|^2\right]} + (\beta \eta_\omega^2 - \eta_\omega) \Expect{\left[ \left\| \sqrt{\hat{H}_{\omega, k}^{-1}}\hat{g}_{\omega} \left(\hat{\omega}_k, \hat{z}_k \right)  \right\|^2\right]}\\
    \ &- \frac{\eta_\omega}{2}\Expect{\left[ \left\|  \nabla_\omega \Phi(\omega_{k})\right\|^2\right]}\\
    \leq \ & \eta_\omega \Expect{\left[ \left\|  \nabla_\omega \Phi(\omega_{k}) -  \hat{g}_{\omega} \left(\hat{\omega}_k, \hat{z}_k \right)      \right\|^2\right]} 
    - \frac{\eta_\omega}{2}\Expect{\left[ \left\|  \nabla_\omega \Phi(\omega_{k})\right\|^2\right]} \\
    \ &+ \beta\eta_\omega^2  \Expect{\left[ \left\|\sqrt{\hat{H}_{\omega, k}^{-1}}\hat{g}_{\omega} \left(\hat{\omega}_k, \hat{z}_k \right)  \right\|^2\right]}
    + \left(\beta \eta_\omega^2 - \eta_\omega \right) \Expect{\left[  \left\|\hat{g}_{\omega} \left(\hat{\omega}_k, \hat{z}_k \right)  \right\|^2\right]} \,.
\end{align*}
When $ 0 < \eta_\omega \leq  \frac{1}{\beta}$, the last term is smaller or less than 0. 
Then by the inequality $\left\| a + b + c \right\|^2 \leq 3\left\| a\right\|^2 + 3\left\| b\right\|^2 + 3\left\| c\right\|^2$, 
\begin{align*}
    &\eta_\omega \Expect{\left[ \left\|  \nabla_\omega \Phi(\omega_{k}) -  \hat{g}_{\omega} \left(\hat{\omega}_k, \hat{z}_k \right)    \right\|^2\right]} \\
    \leq \ & 3 \eta_\omega \Expect{\left[ \left\|  \nabla_\omega \Phi(\omega_{k}) -  \nabla L \left(\omega_k, \hat{z}_k \right)    \right\|^2\right]}  
    + 3 \eta_\omega \Expect{\left[ \left\|  \nabla L \left(\omega_k, \hat{z}_k \right)   -  \nabla L \left(\hat{\omega}_k, \hat{z}_k \right)  \right\|^2\right]}
    + 3 \eta_\omega \Expect{\left[ \left\| \epsilon_\omega\left(\hat{\omega}_k, \hat{z}_k \right)  \right\|^2\right]} \\
    \leq \ & 3 \eta_\omega C^2 \Expect{\left[ \left\| f^\ast (\omega_{k}) -   \hat{z}_k  \right\|^2\right]}  
    + 3 \eta_\omega C^2 \Expect{\left[ \left\|  \omega_k - \hat{\omega}_k \right\|^2\right]}
    + 3 \eta_\omega \Expect{\left[ \left\| \epsilon_\omega\left(\hat{\omega}_k, \hat{z}_k \right)  \right\|^2\right]} \\
    \leq \ & \frac{6\eta_\omega C^2}{\mu } \left(\Phi(\omega_{k}) -  L \left(\omega_k, \hat{z}_k \right)  \right) 
    + 3 \eta_\omega C^2 \Expect{\left[ \left\| \omega_k - \hat{\omega}_k \right\|^2\right]}
    + 3 \eta_\omega \Expect{\left[ \left\| \epsilon_\omega\left(\hat{\omega}_k, \hat{z}_k \right)  \right\|^2\right]}\,.
\end{align*}
The second inequality is by the smoothness property of Assumption \ref{asmp:wlipschitz} and Lemma \ref{lem:smooth}. The third inequality is by PL condition for the dual variable from Assumption \ref{asmp:pl}, which implies quadratic growth from Appendix A of \citet{10.1007/978-3-319-46128-1_50}. Combining the terms, we have the result for the first inequality in the lemma.

For the second inequality in the lemma, notice that by the smoothness property of the envelope function $\Phi$ from Proposition \ref{prop:envelop}, we have 
\begin{align*}
    \Expect{[\Phi(\omega_{k})]} - \Expect{[\Phi(\omega_{k-1})]} 
   \leq \Expect{\left[ \left\langle \nabla_\omega \Phi(\omega_{k-1}), \omega_{k} - \omega_{k-1} \right\rangle \right]} + \frac{\beta}{2}\Expect{\left[ \left\| \omega_{k} - \omega_{k-1}\right\|^2\right]} \,.
\end{align*}
By the update rule, 
\begin{align*}
    \omega_{k} - \omega_{k-1} 
    = - \eta_\omega \left(I + \sqrt{\hat{H}_{\omega, k}^{-1}} \right)\hat{g}_\omega \left(\hat{\omega}_k, \hat{z}_k \right)   \,.
\end{align*}
Using the same techniques, by the identity that $\langle a, b \rangle = \frac{1}{2} \| a\|^2 + \frac{1}{2} \| b\|^2 - \frac{1}{2} \| a-b\|^2 $, 
\begin{align*}
    &\Expect{[\Phi(\omega_{k})]} - \Expect{[\Phi(\omega_{k-1})]} \\
    = \ & - \eta_\omega \Expect{\left[ \left\langle \nabla_\omega \Phi(\omega_{k-1}),  \left(I + \sqrt{\hat{H}_{\omega, k}^{-1}} \right)\hat{g}_{\omega} \left(\hat{\omega}_k, \hat{z}_k \right)  \right\rangle \right]} 
    + \frac{\beta \eta_\omega^2 }{2}\Expect{\left[ \left\|  \left(I + \sqrt{\hat{H}_{\omega, k}^{-1}} \right)\hat{g}_{\omega} \left(\hat{\omega}_k, \hat{z}_k \right)   \right\|^2\right]} \\
    = \ &\frac{\eta_\omega}{2} \Expect{\left[ \left\|  \nabla_\omega \Phi(\omega_{k-1}) -  \left(I + \sqrt{\hat{H}_{\omega, k}^{-1}} \right)\hat{g}_{\omega} \left(\hat{\omega}_k, \hat{z}_k \right)   \right\|^2\right]} 
    -\frac{\eta_\omega}{2}\Expect{\left[ \left\|  \nabla_\omega \Phi(\omega_{k-1})\right\|^2\right]} \\
    \ &+\frac{\beta \eta_\omega^2 - \eta_\omega}{2} \Expect{\left[ \left\|  \left(I + \sqrt{\hat{H}_{\omega, k}^{-1}} \right)\hat{g}_{\omega} \left(\hat{\omega}_k, \hat{z}_k \right)   \right\|^2\right]} \\
    \leq \ & \frac{\eta_\omega}{2} \left( \Expect{\left[ \left\|  \nabla_\omega \Phi(\omega_{k-1}) - \hat{g}_{\omega} \left(\hat{\omega}_k, \hat{z}_k \right)    \right\|\right]} +  \Expect{\left[ \left\| \sqrt{\hat{H}_{\omega, k}^{-1}}\hat{g}_{\omega} \left(\hat{\omega}_k, \hat{z}_k \right)   \right\|\right]} \right)^2 \\
    \ &-\frac{\eta_\omega}{2}\Expect{\left[ \left\|  \nabla_\omega \Phi(\omega_{k-1})\right\|^2\right]} 
    +\frac{\beta \eta_\omega^2 - \eta_\omega}{2} \Expect{\left[ \left\|\left(I + \sqrt{\hat{H}_{\omega, k}^{-1}} \right)\hat{g}_{\omega} \left(\hat{\omega}_k, \hat{z}_k \right)    \right\|^2\right]} \\
    \leq \ & \eta_\omega \Expect{\left[ \left\|  \nabla_\omega \Phi(\omega_{k-1}) -  \hat{g}_{\omega} \left(\hat{\omega}_k, \hat{z}_k \right)     \right\|^2\right]} + \eta_\omega \Expect{\left[ \left\| \sqrt{\hat{H}_{\omega, k}^{-1}}\hat{g}_{\omega} \left(\hat{\omega}_k, \hat{z}_k \right)   \right\|^2\right]} \\
    \ &- \frac{\eta_\omega}{2}\Expect{\left[ \left\|  \nabla_\omega \Phi(\omega_{k-1})\right\|^2\right]}
    + \frac{\beta \eta_\omega^2 - \eta_\omega}{2} \Expect{\left[ \left\| \left(I + \sqrt{\hat{H}_{\omega, k}^{-1}} \right)\hat{g}_{\omega} \left(\hat{\omega}_k, \hat{z}_k \right)  \right\|^2\right]}\\
    \leq \ & \eta_\omega \Expect{\left[ \left\|  \nabla_\omega \Phi(\omega_{k-1}) -  \hat{g}_{\omega} \left(\hat{\omega}_k, \hat{z}_k \right)      \right\|^2\right]} 
    - \frac{\eta_\omega}{2}\Expect{\left[ \left\|  \nabla_\omega \Phi(\omega_{k-1})\right\|^2\right]} \\
    \ &+ \beta\eta_\omega^2  \Expect{\left[ \left\|\sqrt{\hat{H}_{\omega, k}^{-1}}\hat{g}_{\omega} \left(\hat{\omega}_k, \hat{z}_k \right)  \right\|^2\right]}
    + \left(\beta \eta_\omega^2 - \eta_\omega \right) \Expect{\left[  \left\|\hat{g}_{\omega} \left(\hat{\omega}_k, \hat{z}_k \right)  \right\|^2\right]} \\
    \leq \ & \eta_\omega \Expect{\left[ \left\|  \nabla_\omega \Phi(\omega_{k-1}) -  \hat{g}_{\omega} \left(\hat{\omega}_k, \hat{z}_k \right)      \right\|^2\right]} 
    - \frac{\eta_\omega}{2}\Expect{\left[ \left\|  \nabla_\omega \Phi(\omega_{k-1})\right\|^2\right]} 
    + \beta\eta_\omega^2  \Expect{\left[ \left\|\sqrt{\hat{H}_{\omega, k}^{-1}}\hat{g}_{\omega} \left(\hat{\omega}_k, \hat{z}_k \right)  \right\|^2\right]}\,.
\end{align*}
The last inequality is due to our choice of $\eta_\omega$.
By the inequality $\left\| a + b + c \right\|^2 \leq 3\left\| a\right\|^2 + 3\left\| b\right\|^2 + 3\left\| c\right\|^2$ and Lemma \ref{lem:smooth}, 
\begin{align*}
    \eta_\omega \Expect{\left[ \left\|  \nabla_\omega \Phi(\omega_{k-1}) -  \hat{g}_{\omega} \left(\hat{\omega}_k, \hat{z}_k \right)    \right\|^2\right]} 
    \leq \ & 3 \eta_\omega \Expect{\left[ \left\|  \nabla_\omega \Phi(\omega_{k-1}) -  \nabla L \left(\omega_{k-1}, \hat{z}_{k-1} \right)    \right\|^2\right]}  \\
    \ &+ 3 \eta_\omega \Expect{\left[ \left\|  \nabla L \left(\omega_{k-1}, \hat{z}_{k-1} \right)   -  \nabla L \left(\hat{\omega}_k, \hat{z}_k \right)  \right\|^2\right]} \\
    \ &+ 3 \eta_\omega \Expect{\left[ \left\| \epsilon_\omega\left(\hat{\omega}_k, \hat{z}_k \right)  \right\|^2\right]} \\
    \leq \  & 3 \eta_\omega C^2 \Expect{\left[ \left\| f^\ast (\omega_{k-1}) -   \hat{z}_{k-1}  \right\|^2\right]}  
    + 3 \eta_\omega C^2 \Expect{\left[ \left\|  \omega_{k-1} - \hat{\omega}_k \right\|^2\right]} \\
    \ &+ 3 \eta_\omega C^2 \Expect{\left[ \left\|  \hat{z}_k - \hat{z}_{k-1} \right\|^2\right]}
    + 3 \eta_\omega \Expect{\left[ \left\| \epsilon_\omega\left(\hat{\omega}_k, \hat{z}_k \right)  \right\|^2\right]} \\
    \leq \ & \frac{6\eta_\omega C^2}{\mu } \left(\Phi(\omega_{k-1}) -  L \left(\omega_{k-1}, \hat{z}_{k-1} \right)  \right) 
    + 3 \eta_\omega C^2 \Expect{\left[ \left\|  \omega_{k-1} - \hat{\omega}_k \right\|^2\right]} \\
    \ &+ 3 \eta_\omega C^2 \Expect{\left[ \left\|  \hat{z}_k - \hat{z}_{k-1} \right\|^2\right]}
    + 3 \eta_\omega \Expect{\left[ \left\| \epsilon_\omega\left(\hat{\omega}_k, \hat{z}_k \right)  \right\|^2\right]}\,.
\end{align*}
Combining the terms, we have
\begin{align*}
    \Expect{[\Phi(\omega_{k})]} -  \Expect{[\Phi(\omega_{k-1})]} 
    \leq \ &  \frac{6\eta_\omega C^2}{\mu } \left(\Phi(\omega_{k-1}) -  L \left(\omega_{k-1}, \hat{z}_{k-1} \right)  \right) 
    + 3 \eta_\omega C^2 \Expect{\left[ \left\|  \omega_{k-1} - \hat{\omega}_k \right\|^2\right]}\\
    & \ + 3 \eta_\omega C^2 \Expect{\left[ \left\|  \hat{z}_k - \hat{z}_{k-1} \right\|^2\right]}
    + 3 \eta_\omega \Expect{\left[ \left\| \epsilon_\omega\left(\hat{\omega}_k, \hat{z}_k \right)  \right\|^2\right]} \\
    & \ + \beta\eta_\omega^2  \Expect{\left[ \left\|\sqrt{\hat{H}_{\omega, k}^{-1}}\hat{g}_{\omega} \left(\hat{\omega}_k, \hat{z}_k \right)  \right\|^2\right]} \,. &\qedhere
\end{align*}

\end{proof}
\subsection{Proof of Lemma \ref{lem:gap_iterate}}
\begin{lem}
    \label{lem:gap_iterate}
    With Algorithm \ref{alg:sgda} and by Lemma \ref{lem:ada_gradient}, there exist constants $E_1$, $E_2$ such that 
    \begin{align*}
         \frac{1}{N}\sum^N_{k=2} \Expect \left[\left\|\omega_k - \hat{\omega}_{k-1} \right\|^2 \right]\leq \frac{E_1 \hat{G}_\omega^2 \ln \left(d N G^2 + \frac{1}{M} \sum^M_{m=1} \xi \rho^{2m} \right)}{N} +6dG^2\,,
    \end{align*}
    and 
    \begin{align*}
         \frac{1}{N}\sum^N_{k=2} \Expect \left[\left\|\hat{z}_{k} - \hat{z}_{k-1}  \right\|^2\right] =   \frac{1}{N}\sum^N_{k=2} \Expect \left[\left\|\hat{z}_{k-1} - z_{k-1}  \right\|^2 \right]\leq  \frac{E_2 \hat{G}_z^2 \ln \left(d N G^2 + \frac{1}{M} \sum^M_{m=1} \xi \rho^{2m} \right)}{N} +6dG^2\,.
    \end{align*}
\end{lem}
\begin{proof}
    By the update rule, 
    \begin{align*}
        \omega_k - \hat{\omega}_{k-1} 
        = \ & \eta_\omega \left(I + \sqrt{\hat{H}_{\omega, k-2}^{-1}} \right)\hat{g}_\omega \left(\hat{\omega}_{k-2}, \hat{z}_{k-2} \right) - \eta_\omega \left(I + \sqrt{\hat{H}_{\omega, k-1}^{-1}} \right)\hat{g}_\omega \left(\hat{\omega}_{k-1}, \hat{z}_{k-1} \right) \\
        \ & - \eta_\omega \left(I + \sqrt{\hat{H}_{\omega, k}^{-1}} \right)\hat{g}_\omega \left(\hat{\omega}_{k}, \hat{z}_{k} \right) \,.
    \end{align*}
    Thus we have 
    \begin{align*}
        &\left\|\omega_k - \hat{\omega}_{k-1} \right\|^2 \\
        \leq \ &  \eta_\omega^2 \left\|\sqrt{\hat{H}_{\omega, k-2}^{-1}} \hat{g}_\omega \left(\hat{\omega}_{k-2}, \hat{z}_{k-2} \right) + \sqrt{\hat{H}_{\omega, k-1}^{-1}}\hat{g}_\omega \left(\hat{\omega}_{k-1}, \hat{z}_{k-1} \right) - \eta_\omega \sqrt{\hat{H}_{\omega, k}^{-1}} \hat{g}_\omega \left(\hat{\omega}_{k}, \hat{z}_{k} \right) \right\|^2 \\
        \leq \ & 3 \eta_\omega^2 \left\|\sqrt{\hat{H}_{\omega, k-2}^{-1}} \hat{g}_\omega \left(\hat{\omega}_{k-2}, \hat{z}_{k-2} \right) \right\|^2 + 3 \eta_\omega^2 \left\|\sqrt{\hat{H}_{\omega, k-1}^{-1}}\hat{g}_\omega \left(\hat{\omega}_{k-1}, \hat{z}_{k-1} \right) \right\|^2 + 3 \eta_\omega^2 \left\|\sqrt{\hat{H}_{\omega, k}^{-1}}\hat{g}_\omega \left(\hat{\omega}_{k}, \hat{z}_{k} \right) \right\|^2 \\
        \ &+ 6 \eta_\omega^2 \left\|\epsilon_\omega \left(\hat{\omega}_{k-2}, \hat{z}_{k-2} \right) \right\|^2 + 6 \eta_\omega^2 \left\|\epsilon_\omega \left(\hat{\omega}_{k-1}, \hat{z}_{k-1} \right) \right\|^2 + 6 \eta_\omega^2 \left\|\epsilon_\omega \left(\hat{\omega}_{k}, \hat{z}_{k} \right) \right\|^2 + 6dG^2\,.
    \end{align*}
    Summing over $N$ iterations, taking expectation and by Lemma \ref{lem:ada_gradient}, for some constant $E_1$, 
    \begin{align*}
        \frac{1}{N}\sum^N_{k=2} \Expect \left[\left\|\omega_k - \hat{\omega}_{k-1} \right\|^2 \right]\leq \frac{E_1 \hat{G}_\omega^2 \ln \left(d N G^2 + \frac{1}{M} \sum^M_{m=1} \xi \rho^{2m} \right)}{N} + 6dG^2\,.
    \end{align*}
    Similarly, by the update rule we have 
    \begin{align*}
        &\left\|\hat{z}_{k} - \hat{z}_{k-1}  \right\|^2 = \left\|\hat{z}_{k-1} - z_{k-1}  \right\|^2 \\
        =& \left\|\eta_z \left(I + \sqrt{\hat{H}_{z, k-2}^{-1}} \right)\hat{g}_z \left(\hat{\omega}_{k-2}, \hat{z}_{k-2} \right) - \eta_z \left(I + \sqrt{\hat{H}_{z, k-1}^{-1}} \right)\hat{g}_z \left(\hat{\omega}_{k-1}, \hat{z}_{k-1} \right) \right\|^2 \,.
    \end{align*}
    Summing over $N$ iterations, taking expectation and by Lemma \ref{lem:ada_gradient}, for some constant $E_2$, 
    \begin{align*}
        \frac{1}{N}\sum^N_{k=2} \Expect \left[\left\|\hat{z}_{k} - \hat{z}_{k-1}  \right\|^2\right] =  \frac{1}{N}\sum^N_{k=2} \Expect \left[\left\|\hat{z}_{k-1} - z_{k-1}  \right\|^2 \right]\leq  \frac{E_2 \hat{G}_z^2 \ln \left(d N G^2 + \frac{1}{M} \sum^M_{m=1} \xi \rho^{2m} \right)}{N} + 6dG^2\,. 
        & \qedhere
    \end{align*}
\end{proof}
\subsection{Proof of Lemma \ref{lem:bk}}
\begin{lem}
    \label{lem:bk}
    With Algorithm \ref{alg:sgda} and $0 < \eta_\omega \leq \frac{1}{\beta}$, $0 < \eta_z \leq \frac{1}{C}$, we have
    \begin{align*}
        \frac{1}{N}\sum^N_{k=2} \Expect \left(\Phi(\omega_k) - L(\omega_k, \hat{z}_k) \right) 
        \leq \ &\frac{E_3 \max \{\hat{G}_\omega^2, \hat{G}_z^2\} \ln \left(d N G^2 + \frac{1}{M} \sum^M_{m=1} \xi \rho^{2m} \right)}{N} + \frac{dG^2}{2 - 2(1 - \mu \eta_z) \left(1 + \frac{6\eta_\omega C^2}{\mu } \right)}  \\
        \ &+ \frac{B_1}{N - N(1 - \mu \eta_z) \left(1 + \frac{6\eta_\omega C^2}{\mu } \right)} + \frac{60\eta_\omega C^2 dG^2}{1 - (1 - \mu \eta_z) \left(1 + \frac{6\eta_\omega C^2}{\mu } \right)}\,.
    \end{align*}
\end{lem}
\begin{proof}
    By the smoothness assumption of the objective function (\ref{asmp:wlipschitz}), we have 
    \begin{align*}
        L(\omega_k, z_{k-1}) - L(\omega_k, \hat{z}_k)
        \leq \left\langle  \nabla L(\omega_k, \hat{z}_k),z_{k-1} - \hat{z}_k \right\rangle + \frac{C}{2} \left\| z_{k-1} - \hat{z}_k \right\|^2 \,.
    \end{align*}
    By the update rules of the algorithm, we have 
    \begin{align*}
        z_{k-1} - \hat{z}_k = - \eta_z \left(I + \sqrt{\hat{H}_{z, k-1}^{-1}} \right)\hat{g}_z \left(\hat{\omega}_{k-1}, \hat{z}_{k-1} \right) \,.
    \end{align*}
    Using the identity that $\langle a, b \rangle = \frac{1}{2} \left\| a \right\|^2 + \frac{1}{2}  \left\| b \right\|^2 - \frac{1}{2} \left\| a - b \right\|^2$, we have 
    \begin{align*}
        L(\omega_k, z_{k-1}) - L(\omega_k, \hat{z}_k)
        \leq \ & -\left\langle  \nabla L(\omega_k, \hat{z}_k),\left(I + \sqrt{\hat{H}_{z, k-1}^{-1}} \right)\hat{g}_z \left(\hat{\omega}_{k-1}, \hat{z}_{k-1} \right) \right\rangle \\
        \ &+ \frac{C \eta_z^2}{2} \left\| \left(I + \sqrt{\hat{H}_{z, k-1}^{-1}} \right)\hat{g}_z \left(\hat{\omega}_{k-1}, \hat{z}_{k-1} \right) \right\|^2 \\
        = \ & \frac{\eta_z}{2} \left\| \nabla L(\omega_k, \hat{z}_k) - \left(I + \sqrt{\hat{H}_{z, k-1}^{-1}} \right)\hat{g}_z \left(\hat{\omega}_{k-1}, \hat{z}_{k-1} \right) \right\|^2 \\
        \ &- \frac{\eta_z}{2}\left\| \nabla L(\omega_k, \hat{z}_k) \right\|^2 
        + \frac{C \eta_z^2 - \eta_z}{2} \left\| \left(I + \sqrt{\hat{H}_{z, k-1}^{-1}} \right)\hat{g}_z \left(\hat{\omega}_{k-1}, \hat{z}_{k-1} \right) \right\|^2 \\
        \leq \ & \frac{\eta_z}{2} \left\| \nabla L(\omega_k, \hat{z}_k) - \hat{g}_z \left(\hat{\omega}_{k-1}, \hat{z}_{k-1} \right) - \sqrt{\hat{H}_{z, k-1}^{-1}}\hat{g}_z \left(\hat{\omega}_{k-1}, \hat{z}_{k-1} \right) \right\|^2 \\
        \ &- \frac{\eta_z}{2}\left\| \nabla L(\omega_k, \hat{z}_k) \right\|^2 + \frac{C \eta_z^2 - \eta_z}{2} \left\| \left(I + \sqrt{\hat{H}_{z, k-1}^{-1}} \right)\hat{g}_z \left(\hat{\omega}_{k-1}, \hat{z}_{k-1} \right) \right\|^2 \\
        \leq \ & \eta_z \left\| \nabla L(\omega_k, \hat{z}_k) - \hat{g}_z \left(\hat{\omega}_{k-1}, \hat{z}_{k-1} \right)  \right\|^2 + \left(C \eta_z^2 - \eta_z\right) \left\| \hat{g}_z \left(\hat{\omega}_{k-1}, \hat{z}_{k-1} \right)  \right\|^2 \\
        \ &+ C \eta_z^2 \left\| \sqrt{\hat{H}_{z, k-1}^{-1}}\hat{g}_z \left(\hat{\omega}_{k-1}, \hat{z}_{k-1} \right)  \right\|^2 - \frac{\eta_z}{2}\left\| \nabla L(\omega_k, \hat{z}_k) \right\|^2 \,.
    \end{align*}
    Choose $0 < \eta_z \leq \frac{1}{C}$. Then the second term becomes less than or equal to 0. By the assumption that PL condition holds for the dual variables, we have 
    \begin{align*}
        \frac{\eta_z}{2}\left\| \nabla L(\omega_k, \hat{z}_k) \right\|^2  \geq \mu \eta_z \left( \Phi (\omega_k) - L(\omega_k, \hat{z}_k)\right) \,.
    \end{align*}
    Thus, rearranging the terms, 
    \begin{align*}
        \mu \eta_z \left( \Phi (\omega_k) - L(\omega_k, \hat{z}_k)\right) 
        \leq \ & L(\omega_k, \hat{z}_k) - L(\omega_k, z_{k-1}) + \eta_z \left\| \nabla L(\omega_k, \hat{z}_k) - \hat{g}_z \left(\hat{\omega}_{k-1}, \hat{z}_{k-1} \right)  \right\|^2 \\
        &+ C \eta_z^2 \left\| \sqrt{\hat{H}_{z, k-1}^{-1}}\hat{g}_z \left(\hat{\omega}_{k-1}, \hat{z}_{k-1} \right)  \right\|^2 \,.
    \end{align*}
    Rearranging the terms again, we have 
    \begin{align*}
        \Phi(\omega_k) - L(\omega_k, \hat{z}_k)
        \leq \ & (1 - \mu \eta_z) \left(\Phi(\omega_k) - L(\omega_{k-1}, z_{k-1}) \right)+ \eta_z \left\| \nabla L(\omega_k, \hat{z}_k) - \hat{g}_z \left(\hat{\omega}_{k-1}, \hat{z}_{k-1} \right)  \right\|^2 \\
        \ &+ C \eta_z^2 \left\| \sqrt{\hat{H}_{z, k-1}^{-1}}\hat{g}_z \left(\hat{\omega}_{k-1}, \hat{z}_{k-1} \right)  \right\|^2 \,.
    \end{align*}
    Notice that 
    \begin{align*}
        \Phi(\omega_k) - L(\omega_{k-1}, z_{k-1})
        = \ & \left(\Phi(\omega_{k-1}) - L(\omega_{k-1}, \hat{z}_{k-1}) \right) + \left(\Phi(\omega_k) - \Phi(\omega_{k-1})\right) \\
        \ &+ \left(L(\omega_{k-1}, \hat{z}_{k-1})  - L(\omega_{k-1}, z_{k-1}) \right)\,.
    \end{align*}
    For the remaining terms, by Lemma \ref{lem:lemma1}, the expectation of the second term is bounded as 
    \begin{align*}
        \Expect{[\Phi(\omega_{k})]} - \Expect{[\Phi(\omega_{k-1})]} 
        \leq \ & \frac{6\eta_\omega C^2}{\mu } \left(\Phi(\omega_{k-1}) -  L \left(\omega_{k-1}, \hat{z}_{k-1} \right)  \right) 
        + 3 \eta_\omega C^2 \Expect{\left[ \left\|  \omega_{k-1} -  \hat{\omega}_k \right\|^2\right]}\\
        \ &+ 3 \eta_\omega C^2 \Expect{\left[ \left\|  z_{k-1} - \hat{z}_k \right\|^2\right]}
        + 3 \eta_\omega \Expect{\left[ \left\| \epsilon_\omega\left(\hat{\omega}_k, \hat{z}_k \right)  \right\|^2\right]} \\
        \ &+ \beta\eta_\omega^2  \Expect{\left[ \left\|\sqrt{\hat{H}_{\omega, k}^{-1}}\hat{g}_{\omega} \left(\hat{\omega}_k, \hat{z}_k \right)  \right\|^2\right]} \,.
    \end{align*}
    For the third term, by the smoothness assumption on the objective function \ref{asmp:wlipschitz} and the identity $\langle a, b \rangle = \frac{1}{2} \| a\|^2 + \frac{1}{2} \| b\|^2 - \frac{1}{2} \| a-b\|^2 $, we have 
    \begin{align*}
        L(\omega_{k-1}, \hat{z}_{k-1})  - L(\omega_{k-1}, z_{k-1}) 
        \leq \ & \left\langle \nabla  L(\omega_{k-1}, z_{k-1}), \hat{z}_{k-1} - z_{k-1}\right\rangle + \frac{C}{2} \left\| \hat{z}_{k-1} - z_{k-1}\right\|^2 \\
        \leq \ & \frac{1}{2} \left\|\nabla  L(\omega_{k-1}, z_{k-1}) \right\|^2 + \frac{C + 1}{2} \left\| \hat{z}_{k-1} - z_{k-1}\right\|^2 \\
        \leq \ & \frac{dG^2}{2} + \frac{C + 1}{2} \left\| \hat{z}_{k-1} - z_{k-1}\right\|^2 \,.
    \end{align*}
    By the smoothness (Assumption \ref{asmp:wlipschitz}) and Lemma \ref{lem:smooth}, we have 
    \begin{align*}
        &\eta_z \left\| \nabla_z L(\omega_k, \hat{z}_k) - \hat{g}_z \left(\hat{\omega}_{k-1}, \hat{z}_{k-1} \right)  \right\|^2 \\
        \leq \ & 2\eta_z \left\| \nabla_z L(\omega_k, \hat{z}_k) - \nabla L \left(\omega_k, \hat{z}_{k-1} \right)  \right\|^2 + 2\eta_z \left\| \epsilon_z \left(\hat{\omega}_{k-1}, \hat{z}_{k-1} \right)  \right\|^2  \\
        \leq \ & 2\eta_z C^2 \left\| \omega_k - \hat{\omega}_{k-1}\right\|^2 + 2\eta_z C^2 \left\| \hat{z}_k- \hat{z}_{k-1} \right\|^2 + 2\eta_z \left\| \epsilon_z \left(\hat{\omega}_{k-1}, \hat{z}_{k-1} \right)  \right\|^2 \,.  
    \end{align*}
    Thus combining the terms and taking expectation, we have 
    \begin{align*}
        &\Expect \left(\Phi(\omega_k) - L(\omega_k, \hat{z}_k) \right) \\
        \leq \ & (1 - \mu \eta_z) \left(1 + \frac{6\eta_\omega C^2}{\mu } \right)\Expect  \left(\Phi(\omega_{k-1}) - L(\omega_{k-1}, \hat{z}_{k-1}) \right) + \beta\eta_\omega^2  \Expect{\left[ \left\|\sqrt{\hat{H}_{\omega, k}^{-1}}\hat{g}_{\omega} \left(\hat{\omega}_k, \hat{z}_k \right)  \right\|^2\right]} \\
        \ &+ 3 \eta_\omega C^2 \Expect{\left[ \left\|  \omega_{k-1} -  \hat{\omega}_k \right\|^2\right]}
        + 3 \eta_\omega C^2 \Expect{\left[ \left\|  z_{k-1} - \hat{z}_k \right\|^2\right]}
        + 3 \eta_\omega \Expect{\left[ \left\| \epsilon_\omega\left(\hat{\omega}_k, \hat{z}_k \right)  \right\|^2\right]} \\
        \ &+ 2\eta_z C^2 \Expect \left[\left\| \omega_k - \hat{\omega}_{k-1}\right\|^2 \right]+ 2\eta_z C^2 \Expect \left[\left\| \hat{z}_k- \hat{z}_{k-1} \right\|^2 \right]+ 2\eta_z \Expect \left[\left\| \epsilon_z \left(\hat{\omega}_{k-1}, \hat{z}_{k-1} \right)  \right\|^2 \right]\\
        \ &+ C \eta_z^2 \Expect \left[\left\| \sqrt{\hat{H}_{z, k-1}^{-1}}\hat{g}_z \left(\hat{\omega}_{k-1}, \hat{z}_{k-1} \right)  \right\|^2 \right] + \frac{dG^2}{2} + \frac{C + 1}{2} \Expect \left[\left\| \hat{z}_{k-1} - z_{k-1}\right\|^2 \right]\,.
    \end{align*}
    Let $B_k = \Expect \left(\Phi(\omega_k) - L(\omega_k, \hat{z}_k) \right)$, it is then easy to verify that $B_k \geq 0$ for all $k$. Summing over $N$ iterations, we have 
    \begin{align*}
        \frac{1}{N}\sum^N_{k=2} B_k 
        = \ & \frac{(1 - \mu \eta_z) \left(1 + \frac{6\eta_\omega C^2}{\mu } \right)}{N} \sum^{N-1}_{k=1} B_k 
        + \frac{\beta\eta_\omega^2}{N}\sum^{N}_{k=2}  \Expect{\left[ \left\|\sqrt{\hat{H}_{\omega, k}^{-1}}\hat{g}_{\omega} \left(\hat{\omega}_k, \hat{z}_k \right)  \right\|^2\right]} \\
        \ &+ \frac{3 \eta_\omega C^2 }{N} \sum^{N}_{k=2}\Expect{\left[ \left\|  \omega_{k-1} -  \hat{\omega}_k \right\|^2\right]}
        + \frac{3 \eta_\omega C^2 }{N} \sum^{N}_{k=2}\Expect{\left[ \left\|  z_{k-1} - \hat{z}_k \right\|^2\right]} \\
        \ &+ \frac{3 \eta_\omega}{N}  \sum^{N}_{k=2}\Expect{\left[ \left\| \epsilon_\omega\left(\hat{\omega}_k, \hat{z}_k \right)  \right\|^2\right]} 
        + \frac{2 \eta_\omega C^2 }{N} \sum^{N}_{k=2}\Expect \left[\left\| \omega_k - \hat{\omega}_{k-1}\right\|^2 \right] \\
        \ &+ \frac{2 \eta_\omega C^2 }{N} \sum^{N}_{k=2}\Expect \left[\left\| \hat{z}_k- \hat{z}_{k-1} \right\|^2 \right]+ \frac{2 \eta_\omega }{N} \sum^{N}_{k=2} \Expect \left[\left\| \epsilon_z \left(\hat{\omega}_{k-1}, \hat{z}_{k-1} \right)  \right\|^2 \right]\\
        \ &+ \frac{C \eta_z^2}{N} \sum^{N}_{k=2}\Expect \left[\left\| \sqrt{\hat{H}_{z, k-1}^{-1}}\hat{g}_z \left(\hat{\omega}_{k-1}, \hat{z}_{k-1} \right)  \right\|^2 \right] + \frac{dG^2}{2} 
        + \frac{C + 1}{2N} \sum^{N}_{k=2}\Expect \left[\left\| \hat{z}_{k-1} - z_{k-1}\right\|^2 \right] \\
        \leq \ & \frac{(1 - \mu \eta_z) \left(1 + \frac{6\eta_\omega C^2}{\mu } \right)}{N} \left(B_1 +\sum^{N}_{k=2}  B_k \right) + \frac{\beta\eta_\omega^2}{N}\sum^{N}_{k=2}  \Expect{\left[ \left\|\sqrt{\hat{H}_{\omega, k}^{-1}}\hat{g}_{\omega} \left(\hat{\omega}_k, \hat{z}_k \right)  \right\|^2\right]} \\
        \ &+ \frac{3 \eta_\omega C^2 }{N} \sum^{N}_{k=2}\Expect{\left[ \left\|  \omega_{k-1} -  \hat{\omega}_k \right\|^2\right]}
        + \frac{3 \eta_\omega C^2 }{N} \sum^{N}_{k=2}\Expect{\left[ \left\|  z_{k-1} - \hat{z}_k \right\|^2\right]} \\
        \ &+ \frac{3 \eta_\omega}{N} \sum^{N}_{k=2}\Expect{\left[ \left\| \epsilon_\omega\left(\hat{\omega}_k, \hat{z}_k \right)  \right\|^2\right]} 
        + \frac{2 \eta_\omega C^2 }{N} \sum^{N}_{k=2}\Expect \left[\left\| \omega_k - \hat{\omega}_{k-1}\right\|^2 \right] \\
        \ &+ \frac{2 \eta_\omega C^2 }{N} \sum^{N}_{k=2}\Expect \left[\left\| \hat{z}_k- \hat{z}_{k-1} \right\|^2 \right]+ \frac{2 \eta_\omega }{N} \sum^{N}_{k=2} \Expect \left[\left\| \epsilon_z \left(\hat{\omega}_{k-1}, \hat{z}_{k-1} \right)  \right\|^2 \right]\\
        \ &+ \frac{C \eta_z^2}{N} \sum^{N}_{k=2}\Expect \left[\left\| \sqrt{\hat{H}_{z, k-1}^{-1}}\hat{g}_z \left(\hat{\omega}_{k-1}, \hat{z}_{k-1} \right)  \right\|^2 \right] + \frac{dG^2}{2} 
        + \frac{C + 1}{2N} \sum^{N}_{k=2}\Expect \left[\left\| \hat{z}_{k-1} - z_{k-1}\right\|^2 \right]\,.
    \end{align*}
    Rearranging the terms, we have 
    \begin{align*}
        &\frac{1}{N}\sum^N_{k=2} B_k \\
        \leq \ & \frac{1}{1 - (1 - \mu \eta_z) \left(1 + \frac{6\eta_\omega C^2}{\mu } \right)} \frac{B_1}{N}  \\
        \ &+ \frac{1}{1 - (1 - \mu \eta_z) \left(1 + \frac{6\eta_\omega C^2}{\mu } \right)}  \left(\frac{\beta\eta_\omega^2}{N}\sum^{N}_{k=2} \Expect{\left[ \left\|\sqrt{\hat{H}_{\omega, k}^{-1}}\hat{g}_{\omega} \left(\hat{\omega}_k, \hat{z}_k \right)  \right\|^2\right]} + \frac{3 \eta_\omega}{N} \sum^{N}_{k=2}\Expect{\left[ \left\| \epsilon_\omega\left(\hat{\omega}_k, \hat{z}_k \right)  \right\|^2\right]} \right) \\
        \ &+ \frac{1}{1 - (1 - \mu \eta_z) \left(1 + \frac{6\eta_\omega C^2}{\mu } \right)} \left(\frac{3 \eta_\omega C^2 }{N} \sum^{N}_{k=2}\Expect{\left[ \left\|  \omega_{k-1} -  \hat{\omega}_k \right\|^2\right]}
        + \frac{3 \eta_\omega C^2 }{N} \sum^{N}_{k=2}\Expect{\left[ \left\|  z_{k-1} - \hat{z}_k \right\|^2\right]}\right)\\
        \ & + \frac{1}{1 - (1 - \mu \eta_z) \left(1 + \frac{6\eta_\omega C^2}{\mu } \right)} \left(\frac{2 \eta_\omega C^2 }{N} \sum^{N}_{k=2}\Expect \left[\left\| \omega_k - \hat{\omega}_{k-1}\right\|^2 \right]+ \frac{2 \eta_\omega C^2 }{N} \sum^{N}_{k=2}\Expect \left[\left\| \hat{z}_k- \hat{z}_{k-1} \right\|^2 \right]\right)\\
        \ &+ \frac{1}{1 - (1 - \mu \eta_z) \left(1 + \frac{6\eta_\omega C^2}{\mu } \right)} \left(\frac{2 \eta_\omega }{N} \sum^{N}_{k=2} \Expect \left[\left\| \epsilon_z \left(\hat{\omega}_{k-1}, \hat{z}_{k-1} \right)  \right\|^2 \right] \right)\\
        \ &+ \frac{1}{1 - (1 - \mu \eta_z) \left(1 + \frac{6\eta_\omega C^2}{\mu } \right)} \left(\frac{C \eta_z^2}{N} \sum^{N}_{k=2}\Expect \left[\left\| \sqrt{\hat{H}_{z, k-1}^{-1}}\hat{g}_z \left(\hat{\omega}_{k-1}, \hat{z}_{k-1} \right)  \right\|^2 \right] \right)\\
        \ &+ \frac{dG^2}{2 - 2(1 - \mu \eta_z) \left(1 + \frac{6\eta_\omega C^2}{\mu } \right)} + \frac{C + 1}{2N - 2N (1 - \mu \eta_z) \left(1 + \frac{6\eta_\omega C^2}{\mu } \right)} \sum^{N}_{k=2}\Expect \left[ \left\| \hat{z}_{k-1} - z_{k-1}\right\|^2 \right]\,.
    \end{align*}
    By Lemma \ref{lem:gap_iterate} and the update rule of Algorithm \ref{alg:sgda}, there exists a constant $E_3$ such that 
    \begin{align*}
        \frac{1}{N}\sum^N_{k=2} B_k 
        \leq \ &\frac{E_3 \max \{\hat{G}_\omega^2, \hat{G}_z^2\} \ln \left(d N G^2 + \frac{1}{M} \sum^M_{m=1} \xi \rho^{2m} \right)}{N} + \frac{dG^2}{2 - 2(1 - \mu \eta_z) \left(1 + \frac{6\eta_\omega C^2}{\mu } \right)}  \\
        \ &+ \frac{B_1}{N - N(1 - \mu \eta_z) \left(1 + \frac{6\eta_\omega C^2}{\mu } \right)} + \frac{60\eta_\omega C^2 dG^2}{1 - (1 - \mu \eta_z) \left(1 + \frac{6\eta_\omega C^2}{\mu } \right)}\,. &\qedhere
    \end{align*}
\end{proof}
\subsection{Proof of Theorem \ref{thm:single}}
Recall that Theorem \ref{thm:single} stated that when $\eta_\omega = \min \left\{\frac{1}{\beta}, \frac{\mu^2 \eta_z + 6 \eta_\omega \eta_z C^2 \mu }{3C^2 dG^2 N }, \sqrt{\frac{\mu^2 \eta_z + 6 \eta_\omega \eta_z C^2 \mu}{360 C^4 dG^2}} \right\}$ and $\eta_z = \min \left\{\frac{1}{C}, \frac{dG^2 N}{2\mu} \right\}$ and with Assumption \ref{asmp:pl}, \ref{asmp:wlipschitz}, \ref{asmp:bounded}, \ref{asmp:samples}, \ref{asmp:bounded_reward}, we have 
    \begin{align*}
        \sum^N_{k=2}\Expect{\left[ \left\|  \nabla_\omega \Phi(\omega_{k})\right\|^2\right]} \leq \mathcal{O} \left(\frac{\max \{\hat{G}_\omega^2, \hat{G}_z^2\} \ln \left(d N G^2 + \frac{1}{M} \sum^M_{m=1} \xi \rho^{2m} \right)}{N}\right) + \mathcal{O} \left(\frac{1}{\mu N}\right)  \,.
    \end{align*}

\begin{proof}
    By Lemma \ref{lem:lemma1}, we have 
    \begin{align*}
        \Expect{[\Phi(\omega_{k})]} - \Expect{[\Phi(\omega_{k-1})]} 
        \leq \ & \frac{6\eta_\omega C^2}{\mu } \left(\Phi(\omega_{k}) -  L \left(\omega_k, \hat{z}_k \right)  \right) 
        + 3 \eta_\omega C^2 \Expect{\left[ \left\|  \omega_k - \hat{\omega}_k \right\|^2\right]}\\
        \ &+ 3 \eta_\omega C^2 \Expect{\left[ \left\|  z_{k} - \hat{z}_k \right\|^2\right]}
        + 3 \eta_\omega \Expect{\left[ \left\| \epsilon_\omega\left(\hat{\omega}_k, \hat{z}_k \right)  \right\|^2\right]} \\
        \ &+ \beta\eta_\omega^2  \Expect{\left[ \left\|\sqrt{\hat{H}_{\omega, k}^{-1}}\hat{g}_{\omega} \left(\hat{\omega}_k, \hat{z}_k \right)  \right\|^2\right]}
        - \frac{\eta_\omega}{2}\Expect{\left[ \left\|  \nabla_\omega \Phi(\omega_{k})\right\|^2\right]} \,.
    \end{align*}
    Summing over $N$ iterations and by Lemma \ref{lem:bk}, 
    \begin{align*}
        &\frac{1}{N}\sum^N_{k=2} \left(\Expect{[\Phi(\omega_{k})]} - \Expect{[\Phi(\omega_{k-1})]} \right) \\
        \leq \ & \frac{6\eta_\omega C^2}{\mu N } \sum^N_{k=2}\left(\Phi(\omega_{k}) -  L \left(\omega_k, \hat{z}_k \right)  \right) 
        + \frac{3 \eta_\omega C^2}{N} \sum^N_{k=2}\Expect{\left[ \left\|  \omega_k - \hat{\omega}_k \right\|^2\right]}\\
        \ &+ \frac{3 \eta_\omega C^2}{N}\sum^N_{k=2}\Expect{\left[ \left\|  z_{k} - \hat{z}_k \right\|^2\right]}
        + \frac{3 \eta_\omega}{N} \sum^N_{k=2}\Expect{\left[ \left\| \epsilon_\omega\left(\hat{\omega}_k, \hat{z}_k \right)  \right\|^2\right]} \\
        \ &+ \frac{\beta\eta_\omega^2}{N} \sum^N_{k=2} \Expect{\left[ \left\|\sqrt{\hat{H}_{\omega, k}^{-1}}\hat{g}_{\omega} \left(\hat{\omega}_k, \hat{z}_k \right)  \right\|^2\right]}
        - \frac{\eta_\omega}{2N} \sum^N_{k=2}\Expect{\left[ \left\|  \nabla_\omega \Phi(\omega_{k})\right\|^2\right]} \\
        \leq \ & \frac{ 6\eta_\omega C^2 E_3 \max \{\hat{G}_\omega^2, \hat{G}_z^2\} \ln \left(d N G^2 + \frac{1}{M} \sum^M_{m=1} \xi \rho^{2m} \right)}{\mu N} + \frac{6\eta_\omega C^2 dG^2}{2\mu - 2\mu(1 - \mu \eta_z) \left(1 + \frac{6\eta_\omega C^2}{\mu } \right)}  \\
        \ &+ \frac{6\eta_\omega C^2 B_1}{\mu N - \mu N(1 - \mu \eta_z) \left(1 + \frac{6\eta_\omega C^2}{\mu } \right)} + \frac{3 \eta_\omega C^2}{N} \sum^N_{k=2}\Expect{\left[ \left\|  \omega_k - \hat{\omega}_k \right\|^2\right]}\\
        &+ \frac{3 \eta_\omega C^2}{N}\sum^N_{k=2}\Expect{\left[ \left\|  z_{k} - \hat{z}_k \right\|^2\right]}
        + \frac{3 \eta_\omega}{N} \sum^N_{k=2}\Expect{\left[ \left\| \epsilon_\omega\left(\hat{\omega}_k, \hat{z}_k \right)  \right\|^2\right]} + \frac{360\eta_\omega^2 C^4 dG^2}{\mu - \mu(1 - \mu \eta_z) \left(1 + \frac{6\eta_\omega C^2}{\mu } \right)} \\
        \ &+ \frac{\beta\eta_\omega^2}{N} \sum^N_{k=2} \Expect{\left[ \left\|\sqrt{\hat{H}_{\omega, k}^{-1}}\hat{g}_{\omega} \left(\hat{\omega}_k, \hat{z}_k \right)  \right\|^2\right]}
        - \frac{\eta_\omega}{2N} \sum^N_{k=2}\Expect{\left[ \left\|  \nabla_\omega \Phi(\omega_{k})\right\|^2\right]}\,.
    \end{align*}
    By Lemma \ref{lem:ada_gradient}, there exists a constant $E_4$ such that 
    \begin{align*}
        &\frac{3 \eta_\omega C^2}{N} \sum^N_{k=2}\Expect{\left[ \left\|  \omega_k - \hat{\omega}_k \right\|^2\right]}
        + \frac{3 \eta_\omega C^2}{N}\sum^N_{k=2}\Expect{\left[ \left\|  z_{k} - \hat{z}_k \right\|^2\right]}\\
        \ &+ \frac{3 \eta_\omega}{N} \sum^N_{k=2}\Expect{\left[ \left\| \epsilon_\omega\left(\hat{\omega}_k, \hat{z}_k \right)  \right\|^2\right]} + \frac{\beta\eta_\omega^2}{N} \sum^N_{k=2} \Expect{\left[ \left\|\sqrt{\hat{H}_{\omega, k}^{-1}}\hat{g}_{\omega} \left(\hat{\omega}_k, \hat{z}_k \right)  \right\|^2\right]}
        \\
        \leq \ & \frac{E_4 \max \{\hat{G}_\omega^2, \hat{G}_z^2\} \ln \left(d N G^2 + \frac{1}{M} \sum^M_{m=1} \xi \rho^{2m} \right)}{N} \,.
    \end{align*}
    For 
   $
        \frac{6\eta_\omega C^2 dG^2}{2\mu - 2\mu(1 - \mu \eta_z) \left(1 + \frac{6\eta_\omega C^2}{\mu } \right)} \leq \frac{1}{N} 
    $,
    we need the following inequality to hold
    \begin{align*}
        3\eta_\omega C^2 dG^2 N 
        \leq& \mu - \mu(1 - \mu \eta_z) \left(1 + \frac{6\eta_\omega C^2}{\mu } \right) \\
        \leq & \mu^2 \eta_z + 6 \eta_\omega \eta_z C^2 \mu \,.
    \end{align*}
    For 
    \begin{align*}
        \frac{360\eta_\omega^2 C^4 dG^2}{\mu - \mu(1 - \mu \eta_z) \left(1 + \frac{6\eta_\omega C^2}{\mu } \right)} \leq \frac{1}{N} \,,
    \end{align*}
    we need the following inequality to hold
    \begin{align*}
        360\eta_\omega^2 C^4 dG^2
        \leq \ & \mu - \mu(1 - \mu \eta_z) \left(1 + \frac{6\eta_\omega C^2}{\mu } \right) \\
        \leq \ & \mu^2 \eta_z + 6 \eta_\omega \eta_z C^2 \mu \,.
    \end{align*}
    It suffice to take $\eta_\omega = \min \left\{\frac{1}{\beta}, \frac{\mu^2 \eta_z + 6 \eta_\omega \eta_z C^2 \mu }{3C^2 dG^2 N }, \sqrt{\frac{\mu^2 \eta_z + 6 \eta_\omega \eta_z C^2 \mu}{360 C^4 dG^2}} \right\}$ and $\eta_z = \min \left\{\frac{1}{C}, \frac{dG^2 N}{2\mu} \right\}$.
    Combining the terms, we have 
    \begin{align*}
        \frac{1}{N}\sum^N_{k=2} \left(\Expect{[\Phi(\omega_{k})]} - \Expect{[\Phi(\omega_{k-1})]} \right)
        \leq \ &\mathcal{O} \left(\frac{\max \{\hat{G}_\omega^2, \hat{G}_z^2\}\ln \left(d N G^2 + \frac{1}{M} \sum^M_{m=1} \xi \rho^{2m} \right)}{N}\right) + \mathcal{O} \left(\frac{1}{\mu N}\right) \\
        \ & - \frac{\eta_\omega}{2} \sum^N_{k=2}\Expect{\left[ \left\|  \nabla_\omega \Phi(\omega_{k})\right\|^2\right]} \,.
    \end{align*}
    Rearranging the terms, we have
    \begin{align*}
        \frac{1}{N}\sum^N_{k=2}\Expect{\left[ \left\|  \nabla_\omega \Phi(\omega_{k})\right\|^2\right]} \leq \mathcal{O} \left(\frac{\max \{\hat{G}_\omega^2, \hat{G}_z^2\} \ln \left(d N G^2 + \frac{1}{M} \sum^M_{m=1} \xi \rho^{2m} \right)}{N}\right) + \mathcal{O} \left(\frac{1}{\mu N}\right)  \,.
        & \qedhere
    \end{align*}
\end{proof}

\section{Extension to Multi-Agent Actor-Critic}

The algorithm and the analysis developed in this paper can be extended to various reinforcement learning settings. We extend the results to cooperative multi-agent reinforcement learning (MARL) as an example. To our best knowledge, this is the first finite-sample analysis for decentralized multi-agent primal-dual actor-critic algorithms.

We consider a multi-agent discounted Markov decision process (MDP), denoted by the tuple $\mdp = ( \state, \{\action_i \}_i^A,$ $ \mathcal{P}, \{\reward_i\}_i^A, \ga)$, where $A$ is the number of agents. While the agents share the same state space, they may have different action spaces. Thus we use $\action_i$, $\reward_i$ to denote agent $i$'s action space and reward function, respectively. The transition kernel $\mathcal{P}: \state \times \{\action_i\} \to \state$ is then determined by the shared state and the joint action. While this MDP may prescribe various multi-agent reinforcement learning settings under different tasks, we assume that the agents are fully cooperative, where the common goal is to maximize the sum of their expected cumulative discounted rewards
\[
\Expect{_{s_0} \Expect{_\pi \left[ \sum^A_{i=1} \sum^\infty_{t=0} \ga^t \reward_i(s_t, a^i_t) \right]}} \,.
\]
Further, we model the interaction between agents with a network, described by a doubly stochastic matrix $C \in \mathbb{R}^{A \times A}$ where each entry is between $0$ and $1$. This matrix captures the communication network among the agents as $C_{ij}$, $1 \leq i, j \leq A$, $i \neq j$, which characterizes the extent that agent $j$ influences agent $i$. Let $\lambda_i(\cdot)$ denote the $i$-th largest eigenvalue of $C$. Since $C$ is doubly stochastic, $\lambda_1(C)$ is $1$. To ensure that this networked multi-agent extension is well-defined (for example, to eliminate the case where no agent communicates), we maintain the following assumption on the network topology. 
\begin{asmp}
    \label{asmp:comm}
    The communication matrix $C \in \mathbb{R}^{A \times A}$ is doubly stochastic ($C^\top = C$, $\sum^A_{i=1}C_{ji} = 1$ for $j = 1, \dots, A$). Further, $\lambda = \max \{\lambda_2(C), \lambda_\infty(C)\} < 1$.
\end{asmp}
Note that for a network of agents such that the communication graph is connected and the communication matrix is doubly stochastic, $\lambda_2(C) < 1$ holds. 

Let $\omega_i$ and $z_i$ be the local parameter held by agent $i$. We choose $\frac{1}{N} \sum^N_{k=2}  \sum^A_{i=1}  \left\| \nabla_\omega \Phi(\omega_k^i)\right\|^2 $ as our convergence criteria. When the collective cumulative gradient norm of the envelop function is bounded, each agent's cumulative gradient norm must also be bounded. 

We first define a few notations needed for simplicity. Let $W_k = [\omega_1, \dots, \omega_A]^\top$, $\hat{W}_k = [\hat{\omega}_1, \dots, \hat{\omega}_A]^\top$, $Z_k = [z_1, \dots, z_A]^\top$, $\hat{Z}_k = [\hat{z}_1, \dots, \hat{z}_A]^\top$. Then the multi-agent update rules (counterpart of line $6$-$9$ in Algorithm \ref{alg:sgda}), with each agent communicating $t$ times on each update, can be formulated as 
\begin{align*}
    &\hat{W}_k = W_{k-1} C^t  - \eta_\omega  \sqrt{H_{k-1}^{-1}} \hat{g}_\omega \left(\hat{W}_{k-1}, \hat{Z}_{k-1} \right) \,,\\
    &W_k = C^t \left( W_{k-1} - \eta_\omega \sqrt{H_k^{-1}} \hat{g}_\omega \left(\hat{W}_k, \hat{Z}_k \right)\right) \,, \\
    &\hat{Z}_k =  Z_{k-1} C^t  + \eta_\omega \sqrt{H_{k-1}^{-1}} \hat{g}_z \left(\hat{W}_{k-1}, \hat{Z}_{k-1} \right) \,,\\
    &Z_k = C^t \left( Z_{k-1} + \eta_\omega \sqrt{H_k^{-1}} \hat{g}_Z \left(\hat{W}_{k}, \hat{Z}_{k} \right)\right) \,.
\end{align*}
We now describe the full algorithm for the multi-agent extension in Algorithm \ref{alg:multi_sgda}. 
\begin{algorithm}[htb]
    \caption{Adaptive SGDA (ASGDA)}
    \label{alg:multi_sgda}
 \begin{algorithmic}[1]
    \STATE {\bfseries Input:} Learning rates $\eta_\omega, \eta_z = (\eta_u, \eta_\theta)$, batch size $M$, $H_0 = I$, $z=(u, \theta)$, $\hat{G}_z = G + \xi \left(D + \frac{2D}{1-\gamma}\right)^2 \cdot (D_u^2 + D_\theta^2)$, $\hat{G}_\theta = G + \xi (D + 2D_\omega)^2$
    \FOR{$k$ = 1, \dots, $N$}
    \STATE Start from $s \sim \alpha_{k} (s)$ where $\alpha_{k}$ is parametrized by $u_k$, collect samples $\tau_k =\{s_t, a_t, r_t, s_{t+1}\}^{M}_{t=0}$ following policy $\pi_{k}$ parametrized by $\hat{\theta}_{k-1}$ \\ 
    \COMMENT{ // Update gradient estimates}
    \STATE $\hat{g}_{\omega}(\hat{\omega}_k, z_k) =  \nabla_\omega L(\hat{\omega}_{k}, \hat{u}_{k}, \hat{\theta}_{k}, \tau_{k})$
    \STATE  $\hat{g}_{z}(\hat{\omega}_k, z_k) = \nabla_z L(\hat{\omega}_{k}, \hat{u}_{k}, \hat{\theta}_{k}, \tau_{k})$
    \\
    \COMMENT{ // Update primal learning parameters}
    \STATE $\hat{W}_k = W_{k-1} C^t  - \eta_\omega  \sqrt{H_{k-1}^{-1}} \hat{g}_\omega \left(\hat{W}_{k-1}, \hat{Z}_{k-1} \right) $
    \STATE $W_k = C^t \left( W_{k-1} - \eta_\omega \sqrt{H_k^{-1}} \hat{g}_\omega \left(\hat{W}_k, \hat{Z}_k \right)\right) $
    \\
     \COMMENT{ // Update dual learning parameters}
    \STATE $\hat{Z}_k =  Z_{k-1} C^t  + \eta_\omega \sqrt{H_{k-1}^{-1}} \hat{g}_z \left(\hat{W}_{k-1}, \hat{Z}_{k-1} \right)$ 
    \STATE $Z_k = C^t \left( Z_{k-1} + \eta_\omega \sqrt{H_k^{-1}} \hat{g}_Z \left(\hat{W}_{k}, \hat{Z}_{k} \right)\right)$
    \\
    \COMMENT{ // Update primal adaptive gradient parameter}
    \STATE $\hat{g}_{\omega, 0:k} = \frac{1}{\sqrt{2}\hat{G}_\omega}[\hat{g}_{\omega, 0:k-1}, \hat{g}_{\omega}(\hat{\omega}_k, \hat{z}_k)]$
    \STATE $\hat{h}_{\omega, k,i} = \|g_{\omega, 0:k, i} \|^2, i = 1, \dots, d$
    \STATE $\hat{H}_{\omega, k} = \diag (\hat{h}_{\omega, k}) + \frac{1}{2}I$ 
    \\
    \COMMENT{ // Update dual adaptive gradient parameter}
    \STATE $\hat{g}_{z, 0:k} = \frac{1}{\sqrt{2}\hat{G}_z}[\hat{g}_{z, 0:k-1}, \hat{g}_{z}(\hat{\omega}_k, \hat{z}_k)]$
    \STATE $\hat{h}_{z, k,i} =  \|\hat{g}_{z,k, i} \|^2$, $i = 1, \dots, d$
    \STATE $\hat{H}_{z,k} = \diag (\hat{h}_{z, k-1}) + \frac{1}{2} I$
    \ENDFOR
 \end{algorithmic}
\end{algorithm}

Recall in Theorem \ref{thm:multi_thm}, we stated that under Assumption \ref{asmp:pl}, \ref{asmp:wlipschitz}, \ref{asmp:bounded},  \ref{asmp:samples}, \ref{asmp:bounded_reward}, \ref{asmp:comm} and with Algorithm \ref{alg:sgda}, $\min \left\{\frac{1}{\beta}, 2(1 - \lambda), \sqrt{\frac{1}{adG^2 N}}\right\}$, $\eta_z = \min \left\{\frac{1}{C}, \frac{\mu }{3C^2 d G_z^2 A N} \right\}$, we have
    \begin{align*}
        \frac{1}{N} \sum^N_{k=2}  \sum^A_{i=1}  \left\| \nabla_\omega \Phi(\omega_k^i)\right\|^2 
        \leq \mathcal{O} \left(\frac{A}{(1 - \lambda) N } \right)
        + \mathcal{O} \left(\frac{A\max \{\hat{G}_\omega^2, \hat{G}_z^2\} \ln \left(d N G^2 + \frac{1}{M} \sum^M_{m=1} \xi \rho^{2m} \right)}{N \left(1 - \lambda\right)}\right) \,. 
    \end{align*}

Important ingredients in the convergence of single-agent actor-critic, such as large batch size $M$, continue to be crucial for fast multi-agent convergence, as is shown in Theorem \ref{thm:multi_thm}. Moreover, the communication matrix is essential for fast convergence by the  $\mathcal{O}(\frac{A\ln(N)}{(1 - \lambda) N})$ convergence rate. In particular, the convergence is faster when the second to the maximum eigenvalue $\lambda$ is smaller, which suggests the importance of communication in multi-agent reinforcement learning.

In the following sections, we give the proof of the above convergence guarantee. 

\section{Convergence analysis for multi-agent case}
\subsection{Notations} Let $\ei = [0, \dots, 1, \dots, 0]$, the $i$-th canonical basis vector and $\ia$ be a vector of length $A$ where every entry is 1. Let $\bar{\omega} = \frac{1}{A} \sum^A_{i=1} \omega_i$, $\tilde{\omega} = \frac{1}{A} \sum^A_{i=1} \hat{\omega}_i$, $\bar{z} = \frac{1}{A} \sum^A_{i=1} z_i$, $\tilde{z} = \frac{1}{A} \sum^A_{i=1} \hat{z}_i$.

\subsection{Auxiliary Lemma}
\begin{lem}[Lemma 5 \citet{lian2017can}]
    \label{lem:comm}
    Under Assumption \ref{asmp:comm}, 
    \begin{align*}
        \left\| \frac{\ia}{A} - C^t \ei \right\|^2 \leq \lambda^t \,.
    \end{align*}
\end{lem}

\subsection{Proof of Lemma \ref{lem:multi_lem1}}
\begin{lem}
    \label{lem:multi_lem1}
    Under Assumption \ref{asmp:pl}, \ref{asmp:wlipschitz}, \ref{asmp:comm}, 
    \begin{align*}
        \Phi(\omega_k^i) - \Phi(\bar{\omega}_k)
        \leq \ & \frac{6\eta_\omega C^2}{\mu} \left(\Phi (\bar{\omega}_k) -  L \left(\bar{\omega}_k, \tilde{z}_k \right) \right) + 3\eta_\omega C^2 \left\| \bar{\omega}_k -\tilde{\omega}_k\right\|^2 + \frac{\beta}{2} \left\|  \bar{\omega}_k - \omega_k^i \right\|^2 \\
        \ &+  \frac{3\eta_\omega C^2 }{A} \sum^A_{i=1} \left\| \tilde{\omega}_k - \hat{\omega}_k^i \right\|^2 +  \frac{3\eta_\omega C^2 }{A} \sum^A_{i=1} \left\| \tilde{z}_k - \hat{z}_k^i \right\|^2  + \frac{3\eta_\omega }{A}\sum^A_{i=1}  \left\| \epsilon \left(\hat{\omega}_k^i, \hat{z}_k^i\right)\right\|^2 \\
        \ &+ \frac{6\eta_\omega }{A \left(1 - \lambda\right)} \sum^A_{i=1}  \left\| \left(I + \sqrt{H_k^{-1}}\right) \hat{g}_\omega \left(\hat{\omega}^i_k, \hat{z}^i_k\right) \right\|^2 - \frac{\eta_\omega}{2} \left\| \nabla \Phi(\bar{\omega}_k)\right\|^2\,.
    \end{align*}
\end{lem}
\begin{proof}
    By Proposition \ref{prop:envelop}, for $i \in [A]$, 
    \begin{align*}
        \Phi(\omega_k^i) - \Phi(\bar{\omega}_k)
        \leq - \left\langle \nabla \Phi(\bar{\omega}_k), \bar{\omega}_k - \omega_k^i \right\rangle + \frac{\beta}{2} \left\|  \bar{\omega}_k - \omega_k^i \right\|^2 \,.
    \end{align*}
    By the update rule of the algorithm and $W_0 = 0$ and $C^t \ia = \ia$, 
    \begin{align*}
        \bar{\omega}_k - \omega_k^i 
        = \ & W_0 C^t \ia - \frac{\eta_\omega}{A} \sum^k_{j=1} \left(I + \sqrt{H_j^{-1}}\right) \hat{g}_\omega \left(\hat{W}_j, \hat{Z}_j\right) \ia \\
        \ &- \left(W_0 C^t \ei - \eta_\omega \sum^k_{j=1} \left(I + \sqrt{H_j^{-1}}\right) \hat{g}_\omega \left(\hat{W}_j, \hat{Z}_j\right) C^{t(k-j)}\ei \right)\\
        = \ & \eta_\omega\sum^k_{j=1} \left(I + \sqrt{H_j^{-1}}\right) \hat{g}_\omega \left(\hat{W}_j, \hat{Z}_j\right) \left( C^{t(k+1-j)}\ei - \frac{\ia}{A} \right) \,.
    \end{align*}
    By the identity that $\langle a, b \rangle = \frac{1}{2} \| a\|^2 + \frac{1}{2} \| b\|^2 - \frac{1}{2} \| a-b\|^2 $ 
    \begin{align*}
        &\left\langle  \nabla\Phi(\bar{\omega}_k), \bar{\omega}_k - \omega_k^i \right\rangle \\
        = \ & \frac{\eta_\omega}{2} \left\| \nabla\Phi(\bar{\omega}_k) - \sum^k_{j=1} \left(I + \sqrt{H_j^{-1}}\right) \hat{g}_\omega \left(\hat{W}_j, \hat{Z}_j\right) \left( C^{t(k+1-j)}\ei - \frac{\ia}{A} \right) \right\|^2 - \frac{\eta_\omega}{2} \left\| \nabla\Phi(\bar{\omega}_k)\right\|^2 \,.
    \end{align*}
    We then decompose $ \nabla\Phi(\bar{\omega}_k) - \sum^k_{j=1} \left(I + \sqrt{H_j^{-1}}\right) \hat{g}_\omega \left(\hat{W}_j, \hat{Z}_j\right) \left( C^{t(k+1-j)}\ei - \frac{\ia}{A} \right) $ as 
    \begin{align*}
        &\nabla\Phi(\bar{\omega}_k) - \sum^k_{j=1} \left(I + \sqrt{H_j^{-1}}\right) \hat{g}_\omega \left(\hat{W}_j, \hat{Z}_j\right) \left( C^{t(k+1-j)}\ei - \frac{\ia}{A} \right) \\
        = \ & \left(\nabla\Phi(\bar{\omega}_k) - \nabla_\omega L\left(\tilde{\omega}_k, \tilde{z}_k\right) \right) + \left( \nabla_\omega L\left(\tilde{\omega}_k, \tilde{z}_k\right) - \frac{1}{A} \sum^A_{i=1}  \hat{g}_\omega \left(\hat{\omega}_k^i, \hat{z}_k^i\right) \right) \\
        \ &+ \left(\left(I + \sqrt{H_j^{-1}}\right) \hat{g}_\omega \left(\hat{W}_j, \hat{Z}_j\right) C^t\ei - \sum^{k-1}_{j=1} \left(I + \sqrt{H_j^{-1}}\right) \hat{g}_\omega \left(\hat{W}_j, \hat{Z}_j\right) \left( C^{t(k+1-j)}\ei - \frac{\ia}{A} \right) \right)\,.
    \end{align*}
    Using the inequalities $(a+b+c)^2 \leq 3a^2 + 3b^2 + 3c^2$ and $(a+b)^2 \leq 2a^2 + 2b^2$,
    \begin{align*}
        &\left\langle  \nabla\Phi(\bar{\omega}_k), \bar{\omega}_k - \omega_k^i \right\rangle \\
        \leq \ & \frac{3\eta_\omega}{2} \left\| \nabla\Phi(\bar{\omega}_k) - \nabla_\omega L\left(\tilde{\omega}_k, \tilde{z}_k\right) \right\|^2 + \frac{3\eta_\omega}{2} \left\| \nabla_\omega L\left(\tilde{\omega}_k, \tilde{z}_k\right) - \frac{1}{A} \sum^A_{i=1}  \hat{g}_\omega \left(\hat{\omega}_k^i, \hat{z}_k^i\right)\right\|^2 \\
        \ &+ \frac{3\eta_\omega}{2} \left\| \left(I + \sqrt{H_j^{-1}}\right) \hat{g}_\omega \left(\hat{W}_j, \hat{Z}_j\right) C^t\ei - \sum^{k-1}_{j=1} \left(I + \sqrt{H_j^{-1}}\right) \hat{g}_\omega \left(\hat{W}_j, \hat{Z}_j\right) \left( C^{t(k+1-j)}\ei - \frac{\ia}{A} \right) \right\|^2  \\
        \ &- \frac{\eta_\omega}{2} \left\| \nabla \Phi(\bar{\omega}_k)\right\|^2 \\
        \leq \ & \frac{3\eta_\omega}{2} \left\| \nabla\Phi(\bar{\omega}_k) - \nabla_\omega L\left(\tilde{\omega}_k, \tilde{z}_k\right) \right\|^2 + \frac{3\eta_\omega}{2} \left\| \nabla_\omega L\left(\tilde{\omega}_k, \tilde{z}_k\right) - \frac{1}{A} \sum^A_{i=1}  \hat{g}_\omega \left(\hat{\omega}_k^i, \hat{z}_k^i\right)\right\|^2 \\
        \ &+ 3\eta_\omega \sum^{k-1}_{j=1}  \left\| \left(I + \sqrt{H_j^{-1}}\right) \hat{g}_\omega \left(\hat{W}_j, \hat{Z}_j\right) \left( C^{t(k+1-j)}\ei - \frac{\ia}{A} \right) \right\|^2  \\
        \ &+ 3\eta_\omega \left\| \left(I + \sqrt{H_j^{-1}}\right) \hat{g}_\omega \left(\hat{W}_j, \hat{Z}_j\right) C^t\ei \right\|^2 
        - \frac{\eta_\omega}{2} \left\| \nabla \Phi(\bar{\omega}_k)\right\|^2 \,.
    \end{align*}

    For the first term, by Assumption \ref{asmp:wlipschitz} and Assumption \ref{asmp:pl}, which implies quadratic growth from Appendix A of \citet{10.1007/978-3-319-46128-1_50},
    \begin{align*}
        &\frac{3\eta_\omega}{2} \left\| \nabla\Phi(\bar{\omega}_k) - \nabla_\omega L\left(\tilde{\omega}_k, \tilde{z}_k\right) \right\|^2 \\
        = \ & 3\eta_\omega \left\| \nabla\Phi(\bar{\omega}_k) - \nabla_\omega L\left(\bar{\omega}_k, \tilde{z}_k\right) \right\|^2 + 3\eta_\omega \left\| \nabla_\omega L\left(\bar{\omega}_k, \tilde{z}_k\right) - \nabla_\omega L\left(\tilde{\omega}_k, \tilde{z}_k\right) \right\|^2 \\
        = \ & 3\eta_\omega C^2 \left\| f^\ast (\bar{\omega}_k) -  \bar{\omega}_k \right\|^2 + 3\eta_\omega C^2 \left\| \bar{\omega}_k -\tilde{\omega}_k\right\|^2 \\
        \leq \ & \frac{6\eta_\omega C^2}{\mu} \left(\Phi (\bar{\omega}_k) -  L \left(\bar{\omega}_k, \tilde{z}_k \right) \right) + 3\eta_\omega C^2 \left\| \bar{\omega}_k -\tilde{\omega}_k\right\|^2 \,.
    \end{align*}
    By Jensen inequality and Assumption \ref{asmp:wlipschitz}, Lemma \ref{lem:smooth}, 
    \begin{align*}
        &\frac{3\eta_\omega}{2} \left\| \nabla_\omega L\left(\tilde{\omega}_k, \tilde{z}_k\right) - \frac{1}{A} \sum^A_{i=1}  \hat{g}_\omega \left(\hat{\omega}_k^i, \hat{z}_k^i\right)\right\|^2 \\
        \leq \ & \frac{3\eta_\omega}{2A} \sum^A_{i=1} \left\| \nabla_\omega L\left(\tilde{\omega}_k, \tilde{z}_k\right) - \hat{g}_\omega \left(\hat{\omega}_k^i, \hat{z}_k^i\right)\right\|^2 \\
        \leq \ &  \frac{3\eta_\omega C^2 }{A} \sum^A_{i=1} \left\| \tilde{\omega}_k - \hat{\omega}_k^i \right\|^2 + \frac{3\eta_\omega C^2 }{A} \sum^A_{i=1} \left\| \tilde{z}_k - \hat{z}_k^i \right\|^2  + \frac{3\eta_\omega }{A} \sum^A_{i=1} \left\| \epsilon \left(\hat{\omega}_k^i, \hat{z}_k^i\right)\right\|^2 \,.
    \end{align*}
    For the third term, by Lemma \ref{lem:comm}, we have $\left\| C^{t(k+1-j)}\ei - \frac{\ia}{A} \right\|^2  \leq \lambda^{t(k-j)}$.
    For the fourth term, notice that since $C$ is a symmetric matrix $\left\| C^{t(k-j)}\ei \right\|^2 \leq \lambda^{t(k-j)}$.
    Thus, by the Cauchy-Schwarz inequality, these two terms are upper bounded by 
    \begin{align*}
        6\eta_\omega \sum^{k}_{j=1}  \left\| \left(I + \sqrt{H_j^{-1}}\right) \hat{g}_\omega \left(\hat{W}_j, \hat{Z}_j\right) \right\|^2  \lambda^{t(k-j)} 
        \leq \ & \frac{6\eta_\omega }{A} \sum^A_{i=1} \sum^{k}_{j=1}  \left\| \left(I + \sqrt{H_j^{-1}}\right) \hat{g}_\omega \left(\hat{\omega}^i_j, \hat{z}^i_j\right) \right\|^2  \lambda^{t(k-j)} \\
        \leq \ &  \frac{6\eta_\omega }{A} \sum^A_{i=1}  \left\| \left(I + \sqrt{H_k^{-1}}\right) \hat{g}_\omega \left(\hat{\omega}^i_k, \hat{z}^i_k\right) \right\|^2 \sum^\infty_{h=0} \lambda^{th} \\
        \leq \ & \frac{6\eta_\omega }{A \left(1 - \lambda\right)} \sum^A_{i=1}  \left\| \left(I + \sqrt{H_k^{-1}}\right) \hat{g}_\omega \left(\hat{\omega}^i_k, \hat{z}^i_k\right) \right\|^2 \,. 
    \end{align*}

    Combining the terms, we have 
    \begin{align*}
        \Phi(\omega_k^i) - \Phi(\bar{\omega}_k)
        \leq \ & \frac{6\eta_\omega C^2}{\mu} \left(\Phi (\bar{\omega}_k) -  L \left(\bar{\omega}_k, \tilde{z}_k \right) \right) + 3\eta_\omega C^2 \left\| \bar{\omega}_k -\tilde{\omega}_k\right\|^2 + \frac{\beta}{2} \left\|  \bar{\omega}_k - \omega_k^i \right\|^2 \\
        \ &+  \frac{3\eta_\omega C^2 }{A} \sum^A_{i=1} \left\| \tilde{\omega}_k - \hat{\omega}_k^i \right\|^2 +  \frac{3\eta_\omega C^2 }{A} \sum^A_{i=1} \left\| \tilde{z}_k - \hat{z}_k^i \right\|^2  + \frac{3\eta_\omega }{A}\sum^A_{i=1}  \left\| \epsilon \left(\hat{\omega}_k^i, \hat{z}_k^i\right)\right\|^2 \\
        \ &+ \frac{6\eta_\omega }{A \left(1 - \lambda\right)} \sum^A_{i=1}  \left\| \left(I + \sqrt{H_k^{-1}}\right) \hat{g}_\omega \left(\hat{\omega}^i_k, \hat{z}^i_k\right) \right\|^2 - \frac{\eta_\omega}{2} \left\| \nabla \Phi(\bar{\omega}_k)\right\|^2\,. &\qedhere
    \end{align*}
\end{proof}
\subsection{Proof of Lemma \ref{lem:misc}}
\begin{lem} \label{lem:misc}
    Under Assumption \ref{asmp:wlipschitz} and with Algorithm \ref{alg:sgda}, we have 
    \begin{align*}
        \Phi (\bar{\omega}_k) - \Phi (\bar{\omega}_{k-1})
        \leq \ &\frac{3\eta_\omega C^2}{\mu}  \left( \Phi (\bar{\omega}_{k-1}) - L \left(\bar{\omega}_{k-1}, \tilde{z}_{k-1} \right) \right) + \frac{3\eta_\omega C^2}{2} \sum^A_{i=1} \left\|  \bar{\omega}_{k-1} - \hat{\omega}^i_k\right\|^2 \\
        \ &+  \frac{3\eta_\omega C^2}{2} \sum^A_{i=1} \left\| \tilde{z}_{k-1} - \hat{z}^i_k \right\|^2 
        + \frac{12\eta_\omega + \beta \eta_\omega^2}{2} \sum^A_{i=1}\left\| \sqrt{H_{k}^{-1}} \hat{g}_\omega \left(\hat{\omega}^i_k, \hat{z}^i_k \right) \right\|^2 \,,
    \end{align*}
    and
    \begin{align*}
        L \left(\bar{\omega}_{k-1}, \bar{z}_{k-1} \right)  - L \left(\bar{\omega}_{k}, \bar{z}_{k-1} \right)  
        \leq \ & \frac{\eta_z dG_z^2}{2} + \frac{C \eta_z^2 + \eta_z}{2}  \sum^A_{i=1}\left\|\left(I + \sqrt{H_{k-1}^{-1}}\right) \hat{g}_\omega \left(\hat{\omega}^i_{k-1}, \hat{z}^i_{k-1} \right) \right\|^2 \,.
    \end{align*}
    
\end{lem}
\begin{proof}
    By Assumption \ref{asmp:wlipschitz} and Proposition \ref{prop:envelop}, we have
    \begin{align*}
        \Phi (\bar{\omega}_k) - \Phi (\bar{\omega}_{k-1})
        \leq \left\langle \nabla_\omega \Phi (\bar{\omega}_{k-1}), \bar{\omega}_k - \bar{\omega}_{k-1}\right\rangle
        + \frac{\beta}{2} \left\| \bar{\omega}_k - \bar{\omega}_{k-1}\right\|^2 \,.
    \end{align*}
    By the update rules of Algorithm \ref{alg:sgda}, we have 
    \begin{align*}
        \bar{\omega}_k - \bar{\omega}_{k-1} 
        = -\eta_\omega \left(I + \sqrt{H_{k}^{-1}}\right) \hat{g}_\omega \left(\hat{W}_k, \hat{Z}_k \right) \frac{\ia}{A}\,.
    \end{align*}
    By the identity $\langle a, b \rangle = \frac{1}{2} \| a\|^2 + \frac{1}{2} \| b\|^2 - \frac{1}{2} \| a-b\|^2 $, Assumption \ref{asmp:wlipschitz}, and the fact that Assumption \ref{asmp:pl} implies quadratic growth, we have
    \begin{align*}
        &\left\langle \nabla_\omega \Phi (\bar{\omega}_{k-1}), \bar{\omega}_k - \bar{\omega}_{k-1}\right\rangle \\
        \leq \ & \frac{\eta_\omega}{2} \left\| \nabla_\omega \Phi (\bar{\omega}_{k-1}) - \left(I + \sqrt{H_{k}^{-1}}\right) \hat{g}_\omega \left(\hat{W}_k, \hat{Z}_k \right) \frac{\ia}{A}\right\|^2 \\
        \leq \ & \frac{3\eta_\omega}{2}\left\| \nabla_\omega \Phi (\bar{\omega}_{k-1}) - \nabla_\omega L \left(\bar{\omega}_{k-1}, \tilde{z}_{k-1} \right) \right\|^2 + \frac{3\eta_\omega}{2} \sum^A_{i=1} \left\|  \nabla_\omega L \left(\bar{\omega}_{k-1}, \tilde{z}_{k-1} \right) - \hat{g}_\omega \left(\hat{\omega}^i_k, \hat{z}^i_k \right)\right\|^2 \\
        \ &+ \frac{12\eta_\omega}{2} \left\| \sqrt{H_{k}^{-1}} \hat{g}_\omega \left(\hat{W}_k, \hat{Z}_k \right) \frac{\ia}{A} \right\|^2 \\
        \leq \ &\frac{3\eta_\omega C^2}{2} \left\| f^\ast (\bar{\omega}_{k-1}) - \bar{\omega}_{k-1} \right\|^2 + \frac{3\eta_\omega}{2} \sum^A_{i=1} \left\|  \nabla_\omega L \left(\bar{\omega}_{k-1}, \tilde{z}_{k-1} \right) - \hat{g}_\omega \left(\hat{\omega}^i_k, \hat{z}^i_k \right)\right\|^2 \\
        \ &+ \frac{12\eta_\omega}{2} \sum^A_{i=1}\left\| \sqrt{H_{k}^{-1}} \hat{g}_\omega \left(\hat{\omega}^i_k, \hat{z}^i_k \right) \right\|^2 \\
        \leq \ &\frac{3\eta_\omega C^2}{\mu}  \left( \Phi (\bar{\omega}_{k-1}) - L \left(\bar{\omega}_{k-1}, \tilde{z}_{k-1} \right) \right) + \frac{3\eta_\omega}{2} \sum^A_{i=1} \left\|  \nabla_\omega L \left(\bar{\omega}_{k-1}, \tilde{z}_{k-1} \right) - \hat{g}_\omega \left(\hat{\omega}^i_k, \hat{z}^i_k \right)\right\|^2\\
        \ &+ \frac{12\eta_\omega}{2} \sum^A_{i=1}\left\| \sqrt{H_{k}^{-1}} \hat{g}_\omega \left(\hat{\omega}^i_k, \hat{z}^i_k \right) \right\|^2\,.
    \end{align*}
    Combining the terms and by Assumption \ref{asmp:wlipschitz}, Lemma \ref{lem:smooth}, we have the first inequality
    \begin{align*}
        \Phi (\bar{\omega}_k) - \Phi (\bar{\omega}_{k-1})
        \leq \ &\frac{3\eta_\omega C^2}{\mu}  \left( \Phi (\bar{\omega}_{k-1}) - L \left(\bar{\omega}_{k-1}, \tilde{z}_{k-1} \right) \right) + \frac{3\eta_\omega}{2} \sum^A_{i=1} \left\|  \nabla_\omega L \left(\bar{\omega}_{k-1}, \tilde{z}_{k-1} \right) - \hat{g}_\omega \left(\hat{\omega}^i_k, \hat{z}^i_k \right)\right\|^2\\
        \ &+ \frac{12\eta_\omega + \beta \eta_\omega^2}{2} \sum^A_{i=1}\left\| \sqrt{H_{k}^{-1}} \hat{g}_\omega \left(\hat{\omega}^i_k, \hat{z}^i_k \right) \right\|^2 \\
        \leq \ & \frac{3\eta_\omega C^2}{\mu}  \left( \Phi (\bar{\omega}_{k-1}) - L \left(\bar{\omega}_{k-1}, \tilde{z}_{k-1} \right) \right) + \frac{3\eta_\omega C^2}{2} \sum^A_{i=1} \left\|  \bar{\omega}_{k-1} - \hat{\omega}^i_k\right\|^2 \\
        \ &+  \frac{3\eta_\omega C^2}{2} \sum^A_{i=1} \left\| \tilde{z}_{k-1} - \hat{z}^i_k \right\|^2 
        + \frac{12\eta_\omega + \beta \eta_\omega^2}{2} \sum^A_{i=1}\left\| \sqrt{H_{k}^{-1}} \hat{g}_\omega \left(\hat{\omega}^i_k, \hat{z}^i_k \right) \right\|^2 \,.
    \end{align*}

    For the second statement, by Assumption \ref{asmp:wlipschitz}, 
    \begin{align*}
        L \left(\bar{\omega}_{k-1}, \bar{z}_{k-1} \right)  - L \left(\bar{\omega}_{k}, \bar{z}_{k-1} \right) 
        \leq  \left\langle \nabla_\omega L \left(\bar{\omega}_k,\bar{z}_{k-1} \right), \bar{\omega}_{k-1}  - \bar{\omega}_k\right\rangle
        + \frac{C}{2} \left\| \bar{\omega}_{k-1}  - \bar{\omega}_k \right\|^2 \,.
    \end{align*}
    By the multi-agent update rules, 
    \begin{align*}
        \bar{\omega}_{k-1}  - \bar{\omega}_k = 
        \eta_\omega \left(I + \sqrt{H_{k}^{-1}}\right) \hat{g}_\omega \left(\hat{W}_{k}, \hat{Z}_{k} \right) \frac{\ia}{A} \,.
    \end{align*}
    Lastly, the last inequality can be derived by the identity $\langle a, b \rangle = \frac{1}{2} \| a\|^2 + \frac{1}{2} \| b\|^2 - \frac{1}{2} \| a-b\|^2 $, 
    \begin{align*}
        &L \left(\bar{\omega}_{k-1}, \bar{z}_{k-1} \right)  - L \left(\bar{\omega}_{k}, \bar{z}_{k-1} \right)  \\
        \leq \ & \frac{\eta_\omega}{2} \left\| \nabla_\omega L \left(\bar{\omega}_k,\bar{z}_{k-1} \right) \right\|^2 
        + \frac{C \eta_z^2 + \eta_z}{2} \left\|\left(I + \sqrt{H_{k}^{-1}}\right) \hat{g}_\omega \left(\hat{W}_{k}, \hat{Z}_{k} \right) \frac{\ia}{A} \right\|^2 \\
        \leq \ & \frac{\eta_z dG_z^2}{2} + \frac{C \eta_z^2 + \eta_z}{2}  \sum^A_{i=1}\left\|\left(I + \sqrt{H_{k-1}^{-1}}\right) \hat{g}_\omega \left(\hat{\omega}^i_{k-1}, \hat{z}^i_{k-1} \right) \right\|^2 \,. &\qedhere
    \end{align*}
\end{proof}
\subsection{Proof of Lemma \ref{lem:multi_bk}}
\begin{lem}
    \label{lem:multi_bk}
    Under Assumption \ref{asmp:wlipschitz} and with Algorithm \ref{alg:sgda}, 
    \begin{align*}
        \frac{1}{N} \sum^N_{k=2} \left(\Phi (\bar{\omega}_k) - L \left(\bar{\omega}_k, \tilde{z}_k \right) \right)
        \leq \ & \frac{B_1 }{N-N\left(1 - \mu \eta_z\right) \left(1 + \frac{3\eta_\omega C^2}{\mu}\right)} + \frac{\eta_z dG_z^2 \left(1 - \mu \eta_z\right) }{2} \\
        \ &+ \mathcal{O} \left(\frac{A\max \{\hat{G}_\omega^2, \hat{G}_z^2\} \ln \left(d N G^2 + \frac{1}{M} \sum^M_{m=1} \xi \rho^{2m} \right)}{N}\right) + \mathcal{O}\eta_\omega^2 AdG^2\,,
    \end{align*} 
    where $B_1$ is a constant. 
\end{lem}
\begin{proof}
    By Assumption \ref{asmp:wlipschitz}, we have
    \begin{align*}
        L \left(\bar{\omega}_k, \bar{z}_{k-1} \right) -  L \left(\bar{\omega}_k, \tilde{z}_k \right) 
        \leq \ & \left\langle \nabla_z L \left(\bar{\omega}_k, \tilde{z}_k  \right) , \bar{z}_{k-1}- \tilde{z}_k \right\rangle + \frac{C}{2} \left\| \bar{z}_{k-1}- \tilde{z}_k \right\|^2 \,.
    \end{align*}
    By the multi-agent update rules of Algorithm \ref{alg:sgda}, 
    \begin{align*}
        \bar{z}_{k-1}- \tilde{z}_k = - \eta_z \left(I + \sqrt{H_{k-1}^{-1}}\right) \hat{g}_z (\hat{W}_{k-1}, \hat{Z}_{k-1}) \frac{\ia}{A}\,.
    \end{align*}
    Using the identity $\langle a, b \rangle = \frac{1}{2} \| a\|^2 + \frac{1}{2} \| b\|^2 - \frac{1}{2} \| a-b\|^2 $, Assumption \ref{asmp:wlipschitz}, Lemma \ref{lem:smooth} and our choice of step size $0 < \eta_z \leq \frac{1}{C}$, 
    \begin{align*}
        &\left\langle \nabla_z L \left(\bar{\omega}_k, \tilde{z}_k  \right) , \bar{z}_{k-1}- \tilde{z}_k \right\rangle + \frac{C}{2} \left\| \bar{z}_{k-1}- \tilde{z}_k \right\|^2 \\
        \leq \ & \frac{\eta_z}{2}  \left\| \nabla_z L \left(\bar{\omega}_k, \tilde{z}_k  \right) - \left(I + \sqrt{H_{k-1}^{-1}}\right) \hat{g}_z (\hat{W}_{k-1}, \hat{Z}_{k-1}) \frac{\ia}{A}\right\|^2 - \frac{\eta_z}{2}\left\| \nabla L \left(\bar{\omega}_k, \hat{z}_k^i \right) \right\|^2 \\
        \ &+ \frac{C\eta_z^2 - \eta_z}{2} \left\| \left(I + \sqrt{H_{k-1}^{-1}}\right) \hat{g}_z (\hat{W}_{k-1}, \hat{Z}_{k-1}) \frac{\ia}{A}\right\|^2 \\
        \leq \ & \eta_z \sum^A_{i=1} \left\| \nabla L \left(\bar{\omega}_k, \tilde{z}_k \right) - \hat{g}_z (\hat{\omega}_{k-1}^i, \hat{z}_{k-1}^i )\right\|^2 + C\eta_z^2 \sum^A_{i=1}\left\| \sqrt{H_{k-1}^{-1}} \hat{g}_z (\hat{\omega}_{k-1}^i,\tilde{z}_k  )\right\|^2 \\
        \ &+ \left(C\eta_z^2 - \eta_z\right) \sum^A_{i=1}\left\|  \hat{g}_z (\hat{\omega}_{k-1}^i, \hat{z}_{k-1}^i )\right\|^2 - \frac{\eta_z}{2}\left\| \nabla L \left(\bar{\omega}_k, \hat{z}_k^i \right) \right\|^2 \\
        \leq \ & 2\eta_z C^2 \sum^A_{i=1}\left\| \bar{\omega}_k - \hat{\omega}_{k-1}^i\right\|^2 + 2\eta_z C^2 \sum^A_{i=1}\left\| \bar{z}_k - \hat{z}_{k-1}^i \right\|^2 + 2\eta_z \sum^A_{i=1}\left\| \epsilon_z (\hat{\omega}_{k-1}^i, \hat{z}_{k-1}^i )\right\|^2  \\
        \ &+ C\eta_z^2 \sum^A_{i=1}\left\| \sqrt{H_{k-1}^{-1}} \hat{g}_z (\hat{\omega}_{k-1}^i, \hat{z}_{k-1}^i )\right\|^2 
        - \frac{\eta_z}{2}\left\| \nabla_z L \left(\bar{\omega}_k, \tilde{z}_k \right) \right\|^2 \,.
    \end{align*}
    By Assumption \ref{asmp:pl}, 
    \begin{align*}
        \frac{\eta_z}{2}\left\| \nabla_z L \left(\bar{\omega}_k, \tilde{z}_k \right) \right\|^2 \geq \mu \eta_z \left(\Phi (\bar{\omega}_k) - L \left(\bar{\omega}_k, \tilde{z}_k  \right)\right) \,.
    \end{align*}
    Thus, rearranging the terms, we have 
    \begin{align*}
        &\mu \eta_z \left(\Phi (\bar{\omega}_k) - L \left(\bar{\omega}_k, \tilde{z}_k  \right)\right) \\
        \leq \ & L \left(\bar{\omega}_k, \tilde{z}_k  \right) - L \left(\bar{\omega}_k, \bar{z}_{k-1}  \right) + 2\eta_z C^2 \sum^A_{i=1}\left\| \bar{\omega}_k - \hat{\omega}_{k-1}^i\right\|^2 + 2\eta_z C^2 \sum^A_{i=1}\left\| \bar{z}_k - \hat{z}_{k-1}^i \right\|^2 \\
        \ &+ 2\eta_z \sum^A_{i=1}\left\| \epsilon_z (\hat{\omega}_{k-1}^i, \hat{z}_{k-1}^i )\right\|^2  
        + C\eta_z^2 \sum^A_{i=1}\left\| \sqrt{H_{k-1}^{-1}} \hat{g}_z (\hat{\omega}_{k-1}^i, \hat{z}_{k-1}^i )\right\|^2  \,.
    \end{align*}
    Rearranging the terms again, 
    \begin{align*}
        &\Phi (\bar{\omega}_k) - L \left(\bar{\omega}_k, \tilde{z}_k \right) \\
        \leq \ & \left(1 - \mu \eta_z\right) \left(\Phi (\bar{\omega}_k) - L \left(\bar{\omega}_k, \bar{z}_{k-1} \right)\right)  + 2\eta_z C^2 \sum^A_{i=1}\left\| \bar{\omega}_k - \hat{\omega}_{k-1}^i\right\|^2 + 2\eta_z C^2 \sum^A_{i=1}\left\| \bar{z}_k - \hat{z}_{k-1}^i \right\|^2 \\
        \ &+ 2\eta_z \sum^A_{i=1}\left\| \epsilon_z (\hat{\omega}_{k-1}^i, \hat{z}_{k-1}^i )\right\|^2  
        + C\eta_z^2 \sum^A_{i=1}\left\| \sqrt{H_{k-1}^{-1}} \hat{g}_z (\hat{\omega}_{k-1}^i, \hat{z}_{k-1}^i )\right\|^2 \,.
    \end{align*}
    Notice that we can decompose $\Phi (\bar{\omega}_k) - L \left(\bar{\omega}_k, \bar{z}_{k-1} \right)$ as
    \begin{align*}
        \Phi (\bar{\omega}_k) - L \left(\bar{\omega}_k, \bar{z}_{k-1} \right)
        = \ & \left(\Phi (\bar{\omega}_{k-1}) - L \left(\bar{\omega}_{k-1}, \bar{z}_{k-1} \right) \right) + \left(\Phi (\bar{\omega}_k) - \Phi (\bar{\omega}_{k-1})\right) \\
        \ &+ \left(L \left(\bar{\omega}_{k-1}, \bar{z}_{k-1} \right)  - L \left(\bar{\omega}_{k}, \bar{z}_{k-1} \right)  \right) \,.
    \end{align*}
    Then by Lemma \ref{lem:misc}, we have 
    \begin{align*}
        \Phi (\bar{\omega}_k) - L \left(\bar{\omega}_k, \tilde{z}_k \right) 
        \leq \ & \left(1 - \mu \eta_z\right) \left(1 + \frac{3\eta_\omega C^2}{\mu}\right) \left(\Phi (\bar{\omega}_{k-1}) - L \left(\bar{\omega}_{k-1}, \bar{z}_{k-1} \right)\right)  +  \frac{\eta_z dG_z^2 \left(1 - \mu \eta_z\right) }{2}\\
        \ &+ \frac{(12\eta_\omega + \beta \eta_\omega^2)\left(1 - \mu \eta_z\right)}{2} \sum^A_{i=1}\left\| \sqrt{H_{k}^{-1}} \hat{g}_\omega \left(\hat{\omega}^i_k, \hat{z}^i_k \right) \right\|^2
        + \frac{3\eta_\omega C^2\left(1 - \mu \eta_z\right)}{2} \sum^A_{i=1} \left\|  \bar{\omega}_{k-1} - \hat{\omega}^i_k\right\|^2\\ 
        \ &+  \frac{3\eta_\omega C^2\left(1 - \mu \eta_z\right)}{2}\sum^A_{i=1} \left\| \tilde{z}_{k-1} - \hat{z}^i_k \right\|^2 
        + 2\eta_z C^2 \sum^A_{i=1}\left\| \bar{\omega}_k - \hat{\omega}_{k-1}^i\right\|^2\\
        \ & + 2\eta_z C^2 \sum^A_{i=1}\left\| \bar{z}_k - \hat{z}_{k-1}^i \right\|^2 + \frac{(C \eta_z^2 + \eta_z)\left(1 - \mu \eta_z\right) }{2}  \sum^A_{i=1}\left\|\left(I + \sqrt{H_{k-1}^{-1}}\right) \hat{g}_\omega \left(\hat{\omega}^i_{k-1}, \hat{z}^i_{k-1} \right) \right\|^2 \\
        \ &+ 2\eta_z \sum^A_{i=1}\left\| \epsilon_z (\hat{\omega}_{k-1}^i, \hat{z}_{k-1}^i )\right\|^2  
        + C\eta_z^2 \sum^A_{i=1}\left\| \sqrt{H_{k-1}^{-1}} \hat{g}_z (\hat{\omega}_{k-1}^i, \hat{z}_{k-1}^i )\right\|^2  \,.
    \end{align*}
    Let $J_k$ denotes the terms from second term and $B_k = \Phi (\bar{\omega}_k) - L \left(\bar{\omega}_k, \tilde{z}_k \right)$. Notice that by the definition of envelop function $\Phi$, $B_k \geq 0 $ for any $k$. Summing over $N$ iterations, we have 
    \begin{align*}
        \frac{1}{N} \sum^N_{k=2} B_k 
        \leq \ & \left(1 - \mu \eta_z\right) \left(1 + \frac{3\eta_\omega C^2}{\mu}\right)  \frac{1}{N} \sum^N_{k=1} B_k  + \frac{\eta_z dG_z^2 \left(1 - \mu \eta_z\right) }{2} + \frac{1}{N} \sum^N_{k=2} J_k \\
        \leq \ & \left(1 - \mu \eta_z\right) \left(1 + \frac{3\eta_\omega C^2}{\mu}\right) \left( B_1 +  \frac{1}{N} \sum^N_{k=2} B_k  \right)+ \frac{\eta_z dG_z^2 \left(1 - \mu \eta_z\right) }{2} + \frac{1}{N} \sum^N_{k=2} J_k \,.
    \end{align*} 
    Rearranging the terms, 
    \begin{align*}
        \frac{1}{N} \sum^N_{k=2} B_k 
        \leq \ & \frac{B_1 }{N - N\left(1 - \mu \eta_z\right) \left(1 + \frac{3\eta_\omega C^2}{\mu}\right)} + \frac{\eta_z dG_z^2 \left(1 - \mu \eta_z\right) }{2} + \frac{1}{N} \sum^N_{k=2} J_k \,.
    \end{align*} 

    By the update rules of Algorithm \ref{alg:sgda} and Lemma \ref{lem:ada_gradient}, we have 
    \[\frac{1}{N} \sum^N_{k=2} J_k = \mathcal{O} \left(\frac{A\max \{\hat{G}_\omega^2, \hat{G}_z^2\} \ln \left(d N G^2 + \frac{1}{M} \sum^M_{m=1} \xi \rho^{2m} \right)}{N}\right) + \mathcal{O}\eta_\omega^2 AdG^2\,.\qedhere\]
\end{proof}
\subsection{Proof of Theorem \ref{thm:multi_thm}}
Recall that in Theorem \ref{thm:multi_thm}, we claimed that under Assumption \ref{asmp:pl}, \ref{asmp:wlipschitz}, \ref{asmp:bounded}, \ref{asmp:samples}, \ref{asmp:bounded_reward}, \ref{asmp:comm} and with Algorithm \ref{alg:sgda}, $\min \left\{\frac{1}{\beta}, 2(1 - \lambda), \sqrt{\frac{1}{adG^2 N}}\right\}$, $\eta_z = \min \left\{\frac{1}{C}, \frac{\mu }{3C^2 d G_z^2 A N} \right\}$,
    \begin{align*}
        \frac{1}{N} \sum^N_{k=2}  \sum^A_{i=1}  \left\| \nabla_\omega \Phi(\omega_k^i)\right\|^2 
        \leq \mathcal{O} \left(\frac{A}{(1 - \lambda) N } \right)
        + \mathcal{O} \left(\frac{A\max \{\hat{G}_\omega^2, \hat{G}_z^2\} \ln \left(d N G^2 + \frac{1}{M} \sum^M_{m=1} \xi \rho^{2m} \right)}{N \left(1 - \lambda\right)}\right) \,. 
    \end{align*}

\begin{proof}
    By Lemma \ref{lem:multi_lem1} and summing over $N$ iterations, for $i \in [A]$, there exists a constant $E_1 > 0$ such that 
    \begin{align*}
        &\frac{1}{N}\sum^N_{k=2} \Phi(\omega_k^i) - \Phi(\bar{\omega}_k) \\
        \leq \ & \frac{6\eta_\omega C^2}{\mu N} \sum^N_{k=2}\left(\Phi (\bar{\omega}_k) -  L \left(\bar{\omega}_k, \tilde{z}_k \right) \right) + \frac{3\eta_\omega C^2}{N} \sum^N_{k=2}\left\| \bar{\omega}_k -\tilde{\omega}_k\right\|^2 
        + \frac{\beta}{2N} \sum^N_{k=2}\left\|  \bar{\omega}_k - \omega_k^i \right\|^2 \\
        \ &+  \frac{3\eta_\omega C^2 }{AN} \sum^N_{k=2} \sum^A_{i=1} \left\| \tilde{\omega}_k - \hat{\omega}_k^i \right\|^2 +  \frac{3\eta_\omega C^2 }{AN} \sum^N_{k=2} \sum^A_{i=1} \left\| \tilde{z}_k - \hat{z}_k^i \right\|^2  
        + \frac{3\eta_\omega }{AN} \sum^N_{k=2} \sum^A_{i=1}  \left\| \epsilon \left(\hat{\omega}_k^i, \hat{z}_k^i\right)\right\|^2 \\
        \ &+ \frac{6\eta_\omega }{A N\left(1 - \lambda\right)}\sum^N_{k=2} \sum^A_{i=1}  \left\| \left(I + \sqrt{H_k^{-1}}\right) \hat{g}_\omega \left(\hat{\omega}^i_k, \hat{z}^i_k\right) \right\|^2 
        - \frac{\eta_\omega}{2N} \sum^N_{k=2}\left\| \nabla \Phi(\bar{\omega}_k)\right\|^2\\
        \leq \ &  \frac{6\eta_\omega C^2}{\mu N} \sum^N_{k=2}\left(\Phi (\bar{\omega}_k) -  L \left(\bar{\omega}_k, \tilde{z}_k \right) \right) + \frac{E_1 \cdot A\max \{\hat{G}_\omega^2, \hat{G}_z^2\} \ln \left(d N G^2 + \frac{1}{M} \sum^M_{m=1} \xi \rho^{2m} \right)}{N} \,,
    \end{align*}
    where the last inequality is by Lemma \ref{lem:ada_gradient} and the update rules of Algorithm \ref{alg:sgda}. 

    By Lemma \ref{lem:multi_bk}, for some constant $C_m$, we have 
    \begin{align*}
        \frac{1}{N}\sum^N_{k=2} \Phi(\omega_k^i) - \Phi(\bar{\omega}_k)
        \leq \ & \frac{6\eta_\omega C^2 B_1 }{\mu N - \mu N\left(1 - \mu \eta_z\right) \left(1 + \frac{3\eta_\omega C^2}{\mu}\right)} + \frac{6\eta_\omega C^2 \eta_z dG_z^2 \left(1 - \mu \eta_z\right) }{2 \mu} \\
        \ &- \frac{\eta_\omega}{2N} \sum^N_{k=2}\left\| \nabla \Phi(\bar{\omega}_k)\right\|^2 + \mathcal{O} \left(\frac{A\max \{\hat{G}_\omega^2, \hat{G}_z^2\} \ln \left(d N G^2 + \frac{1}{M} \sum^M_{m=1} \xi \rho^{2m} \right)}{N}\right) \\
        \ &+ C_m \eta_\omega^2 AdG^2\,.
    \end{align*}
    Define $P_k = \sum^A_{i=1} \left( \left\| \nabla_\omega \Phi(\omega_k^i)\right\|^2 - \left\| \nabla_\omega \Phi(\bar{\omega}_k)\right\|^2\right) + \frac{1}{\eta_\omega \left(1 - \lambda\right)} \sum^A_{i=1} \left(  \Phi(\omega_k^i) - \Phi(\bar{\omega}_k)\right)$. 
    Then summing over $P_k +  \left\| \nabla_\omega \Phi(\bar{\omega}_k)\right\|^2$, we have
    \begin{align*}
        \frac{1}{N} \sum^N_{k=2}  \sum^A_{i=1}  \left\| \nabla_\omega \Phi(\omega_k^i)\right\|^2 
        \leq \ & \frac{6C^2 B_1 A}{\left(\mu N - \mu N\left(1 - \mu \eta_z\right) \left(1 + \frac{3\eta_\omega C^2}{\mu}\right)\right)\left(1 - \lambda\right)} 
        + \frac{3C^2 \eta_z dG_z^2 A\left(1 - \mu \eta_z\right) }{\mu \left(1 - \lambda\right)} \\
        \ &- \frac{\eta_\omega A}{2N \left(1 - \lambda\right)} \sum^N_{k=2}\left\| \nabla \Phi(\bar{\omega}_k)\right\|^2 - \frac{A}{N} \sum^N_{k=2}\left\| \nabla_\omega \Phi(\bar{\omega}_k)\right\|^2 \\
        \ &+ \mathcal{O} \left(\frac{A\max \{\hat{G}_\omega^2, \hat{G}_z^2\} \ln \left(d N G^2 + \frac{1}{M} \sum^M_{m=1} \xi \rho^{2m} \right)}{N \left(1 - \lambda\right)}\right) + C_m \eta_\omega^2 AdG^2  \,.
    \end{align*}
    Since $0 < \lambda <1$, when $\eta_\omega = \min \left\{\frac{1}{\beta}, 2(1 - \lambda), \sqrt{\frac{1}{adG^2 N}}\right\}$, we have $\frac{\eta_\omega A}{2N \left(1 - \lambda\right)} \sum^N_{k=2}\left\| \nabla \Phi(\bar{\omega}_k)\right\|^2 \geq\frac{A}{N} \sum^N_{k=2}\left\| \nabla_\omega \Phi(\bar{\omega}_k)\right\|^2 $. 
    When $\eta_z = \min \left\{\frac{1}{C}, \frac{\mu }{3C^2 d G_z^2 A N} \right\}$, we have$\frac{3C^2 \eta_z dG_z^2 A\left(1 - \mu \eta_z\right) }{\mu \left(1 - \lambda\right)} \leq \frac{1}{N \left(1 - \lambda\right)}$.

    Combining the terms, we have the final result as
    \begin{align*}
        \frac{1}{N} \sum^N_{k=2}  \sum^A_{i=1}  \left\| \nabla_\omega \Phi(\omega_k^i)\right\|^2 
        \leq \mathcal{O} \left(\frac{A}{(1 - \lambda) N } \right)
        + \mathcal{O} \left(\frac{A\max \{\hat{G}_\omega^2, \hat{G}_z^2\} \ln \left(d N G^2 + \frac{1}{M} \sum^M_{m=1} \xi \rho^{2m} \right)}{N \left(1 - \lambda\right)}\right) \,. &\qedhere
    \end{align*}
\end{proof}

\section{More experiment details}
We study the performance of ASDGA with other commonly used adaptive gradient-based optimizers such as Adam and RMSProp. These adaptive gradient-based optimizers are commonly used in reinforcement learning algorithms and are known to be easier to tune compared to naive gradient descent. As is shown in Figure \ref{fig:overall}, our optimizer performs comparably well for most tasks and outperforms them some tasks such as Inverted Double Pendulum-v2. This highlights that our method is not only theoretically efficient but is also practically effective.  
\begin{figure*}[h]\centering 
\begin{subfigure}[t]{0.25\textwidth}
     \centering
     \includegraphics[width=\textwidth]{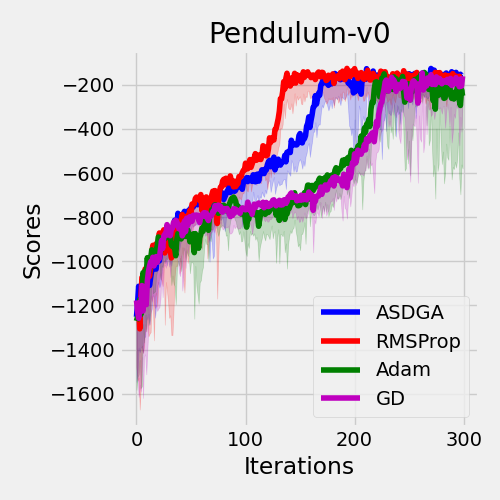}
     \caption{}
     \label{fig:pendulum}
\end{subfigure}
\begin{subfigure}[t]{0.25\textwidth}
     \centering
     \includegraphics[width=\textwidth]{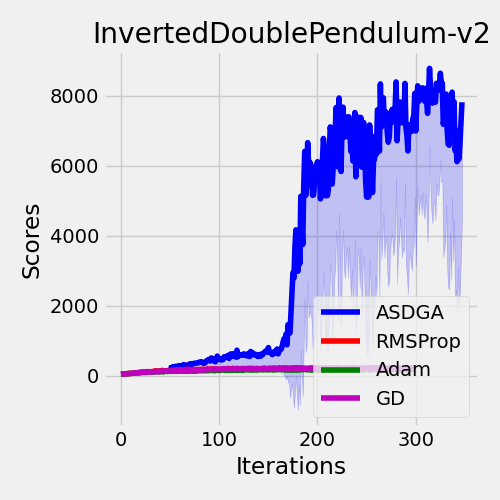}
     \caption{}
     \label{fig:double_pendulum}
\end{subfigure}
\begin{subfigure}[t]{0.25\textwidth}
     \centering
     \includegraphics[width=\textwidth]{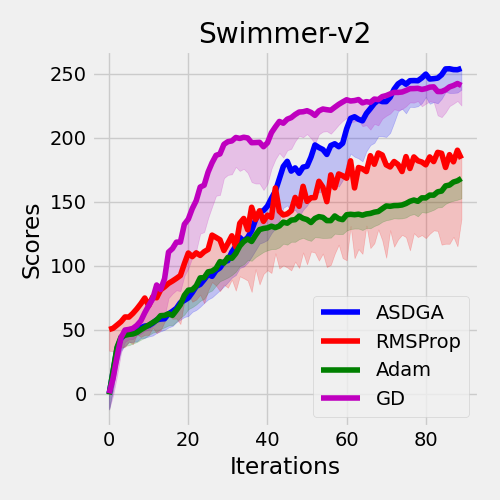}
     \caption{}
     \label{fig:swimmer}
\end{subfigure}
\begin{subfigure}[t]{0.25\textwidth}
     \centering
     \includegraphics[width=\textwidth]{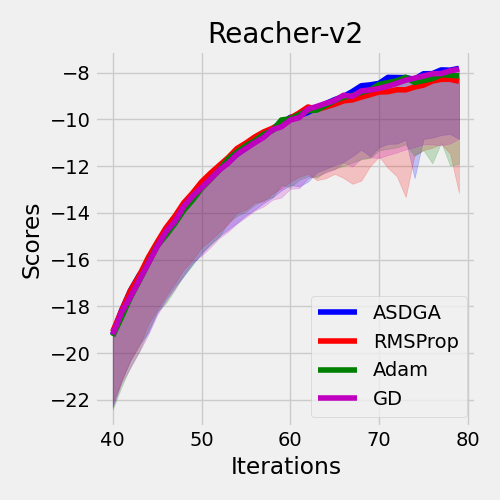}
     \caption{}
     \label{fig:reacher}
\end{subfigure}
\begin{subfigure}[t]{0.25\textwidth}
     \centering
     \includegraphics[width=\textwidth]{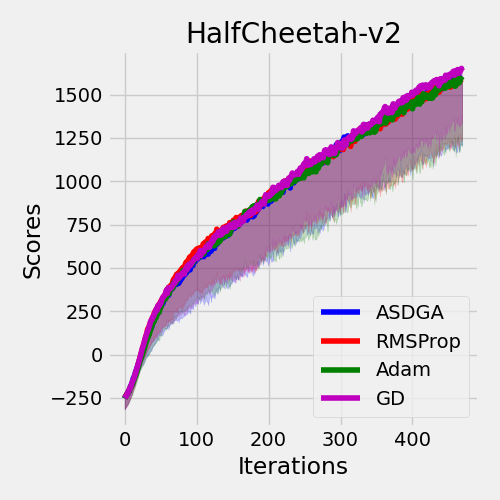}
     \caption{}
     \label{fig:halfcheetah}
\end{subfigure}
\begin{subfigure}[t]{0.25\textwidth}
     \centering
     \includegraphics[width=\textwidth]{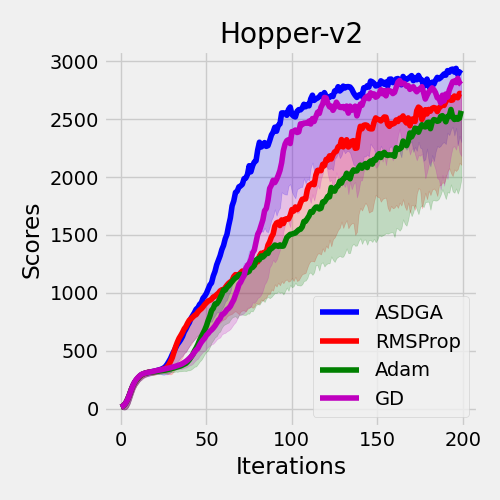}
     \caption{}
     \label{fig:hopper}
\end{subfigure}
\caption{Learning curves for different optimizers on six OpenAI Gym continuous control tasks.}\label{fig:overall}
\end{figure*}

\end{document}